\documentclass{article}

\usepackage[final]{neurips_2025}

\usepackage[utf8]{inputenc} 
\usepackage[T1]{fontenc}    
\usepackage{hyperref}       
\usepackage{url}            
\usepackage{booktabs}       
\usepackage{amsfonts}       
\usepackage{nicefrac}       
\usepackage{microtype}      
\usepackage{xcolor}         
\usepackage{graphicx}
\usepackage{subcaption}
\usepackage{multirow}
\usepackage{graphicx}
\usepackage{enumitem}
\usepackage{todonotes}
\usepackage{wrapfig}

\makeatletter
\newcommand{\ie}{i.e.\@ifnextchar.{\!\@gobble}{}}
\newcommand{\eg}{e.g.\@ifnextchar.{\!\@gobble}{}}
\newcommand{\etc}{etc\@ifnextchar.{}{.\@}}
\makeatother

\renewcommand{\citet}{\citep}

\usepackage{amsmath}
\usepackage{amssymb}
\usepackage{amsthm}

\theoremstyle{plain}
\newtheorem{theorem}{Theorem}[section]
\newtheorem*{theorem*}{Theorem}

\newtheorem{lemma}[theorem]{Lemma}
\newtheorem*{lemma*}{Lemma}

\theoremstyle{definition}
\newtheorem{definition}[theorem]{Definition}

\theoremstyle{remark}

\usepackage{placeins} 
\usepackage{makecell}
\usepackage[export]{adjustbox}   

\title{Pool Me Wisely: On the Effect of Pooling in Transformer-Based Models}

\author{%
  Sofiane Ennadir$^{*}$ \\
  King AI Labs,  Microsoft Gaming \\
  \texttt{sofiane.ennadir@king.com} \\
  \And
  Levente Z\'{o}lyomi$^{*\dagger}$ \\
  NXAI GmbH \\
  \texttt{levente.zolyomi@nx-ai.com} \\
  \And
  Oleg Smirnov \\
  King AI Labs, Microsoft Gaming \\
  \texttt{oleg.smirnov@microsoft.com} \\
  \And
  Tianze Wang$^{\dagger}$ \\
  Kreditz AB \\ 
  \texttt{tianze.wang@kreditz.com} \\
  \And
  John Pertoft \\
  King AI Labs, Microsoft Gaming \\
  \texttt{john.pertoft@king.com} \\
  \And
  Filip Cornell$^{\dagger}$ \\
  Amazon \\
  \texttt{filipco@amazon.com} \\
  \And
  Lele Cao \\
  King AI Labs, Microsoft Gaming \\
  \texttt{lelecao@microsoft.com} \\
}

\begin{document}

\maketitle
\def\thefootnote{*}\footnotetext{\text{Equal contribution.}}
\def\thefootnote{$\dagger$}\footnotetext{\text{Work conducted while at King AI Labs.}}

\begin{abstract}

Transformer models have become the dominant backbone for sequence modeling, leveraging self-attention to produce contextualized token representations. These are typically aggregated into fixed-size vectors via pooling operations for downstream tasks. While much of the literature has focused on attention mechanisms, the role of pooling remains underexplored despite its critical impact on model behavior. In this paper, we introduce a theoretical framework that rigorously characterizes the expressivity of Transformer-based models equipped with widely used pooling methods by deriving closed-form bounds on their representational capacity and the ability to distinguish similar inputs. Our analysis extends to different variations of attention formulations, demonstrating that these bounds hold across diverse architectural variants. We empirically evaluate pooling strategies across tasks requiring both \textit{global} and \textit{local} contextual understanding, spanning three major modalities: computer vision, natural language processing, and time-series analysis. Results reveal consistent trends in how pooling choices affect accuracy, sensitivity, and optimization behavior. Our findings unify theoretical and empirical perspectives, providing practical guidance for selecting or designing pooling mechanisms suited to specific tasks. This work positions pooling as a key architectural component in Transformer models and lays the foundation for more principled model design beyond attention alone. 
\end{abstract}

\section{Introduction}\label{sec:introduction}
The profound impact of the Transformer~\citep{vaswani2017attention} architectures across different modalities and in cross-modal modeling cannot be underestimated. These models have become the building stones of large-scale systems capable of successfully addressing a multitude of downstream tasks in computer vision~\citep{awais2025foundation}, natural language processing (NLP)~\citep{touvron2023llama, jiang2023mistral7b}, and time series analysis~\citep{liang2024foundation, goswami2024moment}. Since these models typically require substantial training data to perform effectively, pre-trained large-scale models trained with self-supervised objectives such as autoregressive, autoencoding, contrastive, or hybrid formulations have become the standard. These models are first trained to capture rich, contextualized representations and are later fine-tuned by attaching a classification or regression head specific to the downstream task, thereby serving as feature extractors.
\begin{wrapfigure}[18]{r}{0.38\textwidth} 
  \centering
  \includegraphics[width=\linewidth]{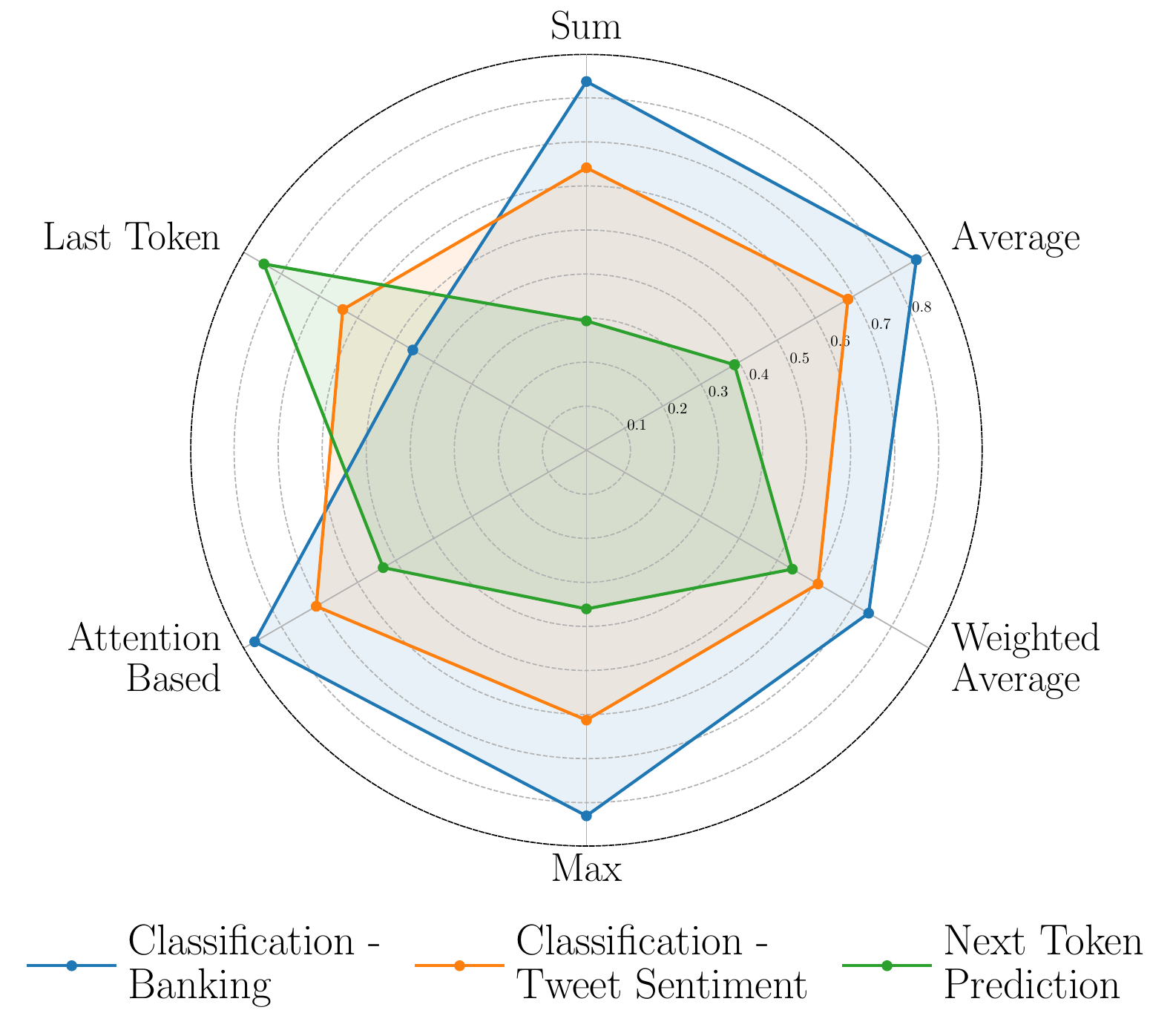}
  \caption{Performance of different pooling strategies using a GPT-2 pre-trained model.}
  \label{fig:fig_intro_radar_chart}
\end{wrapfigure}

Transformer-based models produce a sequence of token-level embeddings, which are typically aggregated into a single vector that captures task-relevant information. This operation, commonly referred to as \textit{pooling}, has been widely studied in domains such as multi-modal learning~\citep{fukui2016multimodal} and Graph Neural Networks~\citep{liu_review_gnn_pooling}, where it is recognized as a key architectural component. The choice of pooling function directly affects the content of the final representation and thus the model’s performance on downstream tasks. Recent studies have increasingly focused on understanding Transformer models from both empirical and theoretical perspectives. However, much of this research has concentrated on the backbone encoder, which processes the input into contextual embeddings, often neglecting the final pooling step that aggregates these into a single representation for prediction. Although some empirical work~\citep{tang2024pooling} has explored the effects of various pooling strategies, a systematic theoretical analysis is still largely absent.

In this work, we close this gap by providing an end-to-end analysis of Transformer-based models that explicitly incorporates the pooling stage. We introduce a theoretical framework to quantify Transformer expressivity, measuring the model’s ability to distinguish dissimilar inputs while preserving similarity. Applying this framework, we analyze how common pooling functions influence expressivity by encoding different token-level properties. From these insights, we derive practical guidelines for choosing pooling methods based on a task’s need for local versus global context. We then validate our theory with experiments in vision, natural language, and time-series domains using standard models and datasets. The results confirm that no single pooling strategy dominates all tasks, underscoring the importance of task-tailored pooling design. To our knowledge, this is the first theoretical examination of pooling mechanisms, offering a unified treatment of standard strategies found in the literature.

\section{Related Work}\label{sec:related_work}
A growing body of work has provided theoretical insights into Transformer-based models and their attention mechanisms. Prior research has explored a broad range of topics, including training dynamics~\citep{tian2023scan}, inductive biases~\citep{lavie2024towards}, and in-context learning~\citep{von2023transformers}. Much of this literature focuses on the core Transformer backbone, proposing architectural improvements and refined training procedures. However, in practical scenarios, particularly those involving large pretrained models, the backbone is typically frozen, and a lightweight classification or regression head is appended after the final pooling layer. This common design choice limits the direct applicability of many of the previously proposed modifications.

Prior studies have demonstrated that different pooling strategies can significantly influence downstream performance. For example,~\citet{dosovitskiy2021an} showed that, in Vision Transformers (ViT), the choice between using a CLS token and average pooling leads to measurable differences in classification accuracy. Those findings were reaffirmed in a follow-up study~\citep{naseer2021intriguing}, which examined the influence of the classification token across different Transformer layers. In the language domain,~\citet{lee2025nv} conducted an empirical analysis of BERT-based~\citep{devlin2019bert} embedding models and decoder-only models from the GPT~\citet{radford2019language,brown2020language} family, along with their standard pooling layers, and proposed a new technique to mitigate information dilution and recency bias commonly observed in the respective pooling operations. In addition,~\citet{tang2024pooling} systematically evaluated combinations of attention mechanisms and pooling methods in large language models, and introduced an attention-based pooling approach that aggregates representations from multiple hidden layers. However, the performance of these novel pooling mechanisms has been shown to vary considerably across different types of tasks.

Multiple studies have further examined the impact of pooling strategies across a variety of settings~\citep{tsukagoshi2021defsent,xing2024comparative}. This line of research concludes that the optimal pooling method generally depends on the specific downstream task, as well as other factors such as model size. Although these works provide extensive empirical evidence, to the best of our knowledge, none have proposed a principled theoretical framework to explain the behavior of pooling operations in a broader context.

More recently, increased attention has been directed toward studying the theoretical properties of Transformer architectures through the lens of Lipschitz continuity, with the goal of understanding its behavior and dynamics. For instance,~\citet{kim2021lipschitz} propose an L2-based attention mechanism that replaces the standard dot-product, providing both theoretical and empirical evidence of its Lipschitz continuity. In a similar direction,~\citet{dasoulas2021lipschitz} introduce LipschitzNorm, a normalization technique designed to enforce Lipschitz continuity within the self-attention mechanism. Building on these efforts,~\citet{qi2023lipsformer} further redesign the full Transformer architecture to ensure that the model remains Lipschitz continuous throughout. However, these analyses typically do not extend to the final pooling operation, despite its widespread use in practical applications.

Our work extends both lines of inquiry by closing the gap between empirical findings and theoretical understanding, and contributing to a deeper comprehension of how pooling functions influence the performance of Transformer-based models.

\section{Preliminaries}\label{sec:preliminaries}
We start by reviewing the Transformer architecture and its key components, forming the basis for the concepts of our theoretical study. 

\textbf{Transformer Architecture.} Let $X \in \mathcal{X} \subseteq \mathbb{R}^{n \times d}$ denote a sequence of $n$ tokens, where each token $x_i \in \mathbb{R}^d$. The backbone of a transformer $h : \mathcal{X} \subseteq \mathbb{R}^{n \times d} \rightarrow \mathcal{Z} \subseteq \mathbb{R}^{n \times d}$, as introduced in \citep{vaswani2017attention}, is the \emph{self-attention} mechanism, which computes a weighted combination of all token representations. Specifically, given learnable query, key, and value parameter matrices $W^Q, \, W^K, \, W^V \in \mathbb{R}^{d \times (d/H)},$ the output of a single attention head $\mathrm{AH}$ for input $X$ is defined as
\begin{equation}\label{eq:scaled_dot_product_attention}
    \mathrm{AH}(X) = \operatorname{softmax}\!\left(\frac{(XW^Q)(XW^K)^\top}{\sqrt{d/H}}\right)(XW^V),
\end{equation}
where $H$ denotes the number of parallel attention heads and $d/H$ is the dimension per head. In practice, multiple attention heads $\mathrm{AH}_i$ are computed in parallel, then concatenated and projected using a learnable weight matrix $W^O \in \mathbb{R}^{d \times d}$, yielding the Multi-Head Attention (MHA) operation:
\begin{equation}\label{eq:multi_head_attention}
    \mathrm{MHA}(X) = \mathrm{concat}\bigl(\mathrm{AH}_1(X), \mathrm{AH}_2(X), \dots, \mathrm{AH}_H(X)\bigr) W^O.
\end{equation}

\textbf{Attention Block.}
In addition to MHA, each Transformer attention block (AB) incorporates a residual connection, layer normalization~\citep{lei2016layer} and a position-wise feed-forward network (FFN), and can be written in the following two steps:
\begin{align}\label{eq:transformer_block_1}
    X' &= \mathrm{LN}\bigl(X + \mathrm{MHA}(X)\bigr); \quad 
    \mathrm{AB}(X) = \mathrm{LN}\bigl(X' + \mathrm{FFN}(X')\bigr). 
\end{align}

with $\mathrm{LN}(\cdot)$ denoting the layer normalization operation. $\mathrm{FFN}(\cdot)$ is a feed-forward network, formulated as $\mathrm{FFN}(X') = \sigma\bigl(X' W_{FFN}\bigr),$ with $\sigma$ being a non-linear activation function. While different placements of normalization layers, commonly called Pre-LN and Post-LN, have been examined in prior work~\citep{li2025mix}, this study focuses on the original Post-LN setup. We note that our main findings are transferable to other configurations.

\textbf{Pooling.} Given the output of a Transformer backbone \(Z\in\mathcal{Z}\subseteq\mathbb{R}^{n\times d}\), a pooling function \(g:\mathcal{Z}\to\mathcal{Y}\subseteq\mathbb{R}^d\) produces a fixed-size embedding for downstream tasks. Common choices are Average pooling, which computes the mean over tokens; Sum pooling, which sums the token embeddings; Max pooling, which takes the elementwise maximum; and Last-token pooling, which selects a designated token (for example the final or CLS token).These operations can be formally defined as:
$$
g_{\text{Avg}}(Z) = \frac{1}{n} \sum_{i=1}^{n} Z[i, :] ; \quad 
g_{\text{Sum}}(Z) = \sum_{i=1}^{n} Z[i, :] ; \quad
g_{\text{Max}}(Z) =  \max_i Z[i, :] ;  \quad
g_{\text{Last}}(Z) = Z[n, :].
$$

\textbf{Problem Setup.} Let \(f\colon\mathcal{X}\subseteq\mathbb{R}^{n\times d}\to\mathcal{Y}\subseteq\mathbb{R}^d\) be a Transformer-based model incorporating a final pooling layer. For our theoretical analysis, we model \(f\) as a single MHA block with \(H\) heads, followed by a one-layer FFN and the layer-normalization variant of~\citet{qi2023lipsformer}, which is provably stable under small input perturbations. We assume all activation functions are 1-Lipschitz (e.g.\ ReLU, LeakyReLU, TanH)~\citep{virmaux18}, and that the input space is bounded, \(\mathcal{X}\subset[0,B]^{n\times d}\) This bound is realistic: in vision and time-series applications \(B=1\) after normalization, and in NLP the embedding process usually results in bounded input representations due to the initialization of the embedding matrix.

\section{On the Expressivity of Transformer-Based Models}
\label{sec:expressivity}
In this section, we introduce the notion of expressivity for Transformer-based models (TBMs). Building upon this definition, we develop a theoretical analysis of several attention mechanisms and commonly used pooling strategies.

\subsection{Expressivity of TBMs}\label{sec:expressivity_of_TBM}

Inspired by work in graph representation learning, we define \textit{expressivity} as the capacity of a model to distinguish between similar and dissimilar inputs~\citep{xupowerful, morris2019weisfeiler, morris2020weisfeiler}. Specifically, by being able to distinguish between such cases, a model is able to produce meaningful representation that could be used by a classification or regression head to produce the final downstream task. For instance, in natural language processing, two semantically similar sentences should yield representations that are closer in the output embedding space than those produced by two semantically disparate sentences. In this perspective, defining an accurate measure of semantic similarity that is applicable across diverse domains such as NLP and computer vision, is crucial and fundamental to evaluating expressivity. Let $\mathcal{X}$ and $\mathcal{Y}$ denote the input and output spaces, respectively, we consider both spaces to be measurable and equipped with a measures $\mid\cdot\mid_{\mathcal{X}}$ and $\mid\cdot\mid_{\mathcal{Y}}$. With a well-designed embedding function, semantically similar elements from the input space are mapped to proximate points in the output space. Consequently, the distances in $\mathcal{Y}$ are expected to accurately reflect the semantic relationships present in $\mathcal{X}$.

Let $f \colon \mathcal{X} \subseteq \mathbb{R}^{n \times d} \rightarrow \mathcal{Y} \subseteq \mathbb{R}^d$ be a TBM as defined in Section~\ref{sec:preliminaries}. For a given input $X \in \mathcal{X}$, we define its \textit{neighborhood} with respect to the input distance metric and a threshold $\epsilon$ by:
\begin{align*}
\mathcal{B}(X, \epsilon) = \{ \tilde{X} \in \mathcal{X} \,:\, |X - \tilde{X}|_{\mathcal{X}} \leq \epsilon \}.
\end{align*}
Since the desired behavior is for inputs in close proximity to yield similar output representations, we consider the following measure: 
\begin{align}\label{eq:experessivity}
\mathcal{E}_{\epsilon}[f] = \mathbb{P}_{X \sim \mathcal{D}_{\mathcal{X}}} \Bigl[ \tilde{X} \in \mathcal{B}(X, \epsilon) : d_{\mathcal{Y}}(f(\tilde{X}), f(X)) > \sigma \Bigr],
\end{align}
where~$\mathcal{D}_{\mathcal{X}}$ represents the underlying distribution over the input space $\mathcal{X}$, and~$d_{\mathcal{Y}}$ is a distance metric on~$\mathcal{Y}$. The quantity~$\mathcal{E}_{\epsilon}[f]$ encapsulates the probability that two inputs, which are similar (within the same neighborhood) in the input space~$\mathcal{X}$, are mapped to outputs that differ by more than a threshold~$\sigma$. Intuitively, a small input distance should yield a correspondingly small output distance, while larger differences in the input should result in more pronounced variations in the output. A model that appropriately distinguishes these nuances is said to exhibit stronger expressive power, which is as motivated previously a vital attribute for achieving robust downstream performance. Definition~\ref{def:expressivity} reflects such expressivity in the context of TBMs.

\begin{definition}\label{def:expressivity}
Let $f \colon \mathcal{X} \subseteq \mathbb{R}^{n \times d} \rightarrow \mathcal{Y} \subseteq \mathbb{R}^d$ be a TBM. The model $f$ is said to be $(\epsilon, \sigma, \gamma)$-expressive if $\mathcal{E}_{\epsilon}[f] \leq \gamma$.
\end{definition}

Definition~\ref{def:expressivity} depends on several hyperparameters. The threshold $\epsilon$ specifies when two inputs are considered semantically similar and is inherently application-specific. For instance, a minor perturbation in an image may be negligible, while the same in a financial time series could be meaningful. The parameter~$\sigma$ defines the acceptable variation in the output space for representations to be considered similar. By setting $\epsilon$ based on domain knowledge, the interaction between $\epsilon$ and $\sigma$ allows us to capture the model's expressive capacity in a way that reflects the semantic structure of the data. This formulation highlights the model's adaptability and expressivity in application-specific contexts.

\subsection{Expressivity of Pooling Strategies}
\label{sec:analysis_transformers}
Based on Definition~\ref{def:expressivity}, and given fixed values of $\epsilon$ and $\sigma$, our objective is to quantify the corresponding expressivity parameter $\gamma$ for different pooling strategies. This analysis allows us to assess how the choice of pooling influences the model's ability to distinguish semantically meaningful variations in the input. Throughout the remainder of this paper, $\lVert \cdot \rVert$ denotes the operator norm.

\begin{theorem}\label{theo:bound_transformer}
Let $f \colon \mathcal{X} \subseteq \mathbb{R}^{n \times d} \rightarrow \mathcal{Y} \subseteq \mathbb{R}^d$ be a TBM following the framework introduced in Section~\ref{sec:preliminaries}. In respect to Definition~\ref{def:expressivity}, we have:  
\begin{itemize}[leftmargin=*]
    \item If $f$ employs Average pooling, then $f$ is $(\epsilon,\sigma,\gamma)$-expressive with 
    $
    \gamma = \frac{\epsilon}{\sigma \sqrt{n}} \left(\frac{d}{d-1}\right)^2 C_1 C_2
    $
    \item If $f$ employs Sum pooling, then $f$ is $(\epsilon,\sigma,\gamma)$-expressive with 
    $
    \gamma = \frac{\sqrt{n}\, \epsilon}{\sigma} \left(\frac{d}{d-1}\right)^2 C_1 C_2
    $
    \item If $f$ employs Last-token pooling, then $f$ is $(\epsilon,\sigma,\gamma)$-expressive with 
    $
    \gamma = \frac{\epsilon}{\sigma} \left(\frac{d}{d-1}\right)^2 C_1 C_2
    $
    \item If $f$ employs Max pooling, then $f$ is $(\epsilon,\sigma,\gamma)$-expressive with 
    $
    \gamma = \frac{\epsilon\sqrt{\min(n,d)}}{\sigma} \left(\frac{d}{d-1}\right)^2 C_1 C_2,
    $
\end{itemize}
{\small
$$ \text{with: }
    C_1 = 1 +\lVert W_O \lVert \sqrt{H}\, \max_h \Biggl[ \Bigl\lVert W^{V,h} \Bigr\rVert \left(4\,\frac{n}{\sqrt{d/H}}\,B^2\,\Bigl\lVert W^{Q,h} \Bigr\rVert\,\Bigl\lVert W^{K,h} \Bigr\rVert + 1\right) \Biggr], \quad
    C_2 = 1 + \Bigl\lVert W_{FFN} \Bigr\rVert.
$$
}
\end{theorem}

Theorem~\ref{theo:bound_transformer} shows that expressivity bounds across pooling strategies depend on shared architectural parameters, including the number of attention heads $H$, embedding dimension $d$, and sequence length $n$, captured through constants $C_1$ and $C_2$. These constants reflect the norms of key model components, such as attention weights, projection layers, and feed-forward networks. Each pooling function introduces distinct scaling effects, shaping how these elements combine to influence the model's ability to separate similar from dissimilar inputs.

For Average pooling, the bound scales with $1/\sqrt{n}$, smoothing the output by evenly distributing token contributions. This favors tasks where \textit{global} structure matters more than individual token details. In contrast, Sum pooling scales with $\sqrt{n}$, amplifying token-level variation. This is useful for tasks where \textit{localized} information is essential. Using a single token (e.g., the last or CLS token) leads to a scaling of $1$, preserving variations without change. This suits scenarios where a specific token encodes the most relevant context, such as in sentiment analysis. Max pooling introduces a bound that scales as $\sqrt{\min(n, d)}$. When $d$ is large relative to $n$, it behaves similarly to Sum pooling, capturing fine-grained differences. When $d$ is smaller, it emphasizes broader context. This flexibility enables Max pooling to shift between local and global focus based on model size and sequence length.

Overall, the theoretical results emphasize that pooling is a key factor in how Transformer models aggregate local token information into a global representation. Theorem~\ref{theo:bound_transformer} formalizes how this choice affects model expressivity across different settings.

\textbf{On the generalization to multi-layer TBMs.} We note that the current theoretical analysis focuses on a single-layer Transformer-based model; nonetheless, the results naturally extend to the multi-layer case. Specifically, a Transformer model with $L$ layers, denoted as $f^{(L)}$, can be expressed as a composition of $L$ single-layer functions: 
$f^{(L)}(x) = f^{(L-1)} \circ f^{(L-2)} \circ \dots \circ f^{(1)}(x).$ Under this formulation, and following standard results from Lipschitz continuity, the overall expressivity bound bound $\gamma$ for each pooling becomes a multiplicative composition of the bounds for each individual layer. As a result, our theoretical study remains applicable to deeper architectures, as confirmed by experiments involving exclusively multi-layer models.

\subsection{Expressivity of Alternative Attention Mechanisms}
Recent studies have proposed alternative formulations of the scaled dot-product self-attention mechanism to improve model behavior and facilitate theoretical analysis. For example, L2 Multi-Head Attention (L2-MHA)~\cite{kim2021lipschitz} employs an L2-kernel attention function, while LipsFormer~\cite{qi2023lipsformer} replaces the dot-product with a scaled cosine similarity. The theoretical results from the previous section are general and extend to these variants. In the following, we consider the same problem setup, with the only change being the use of an alternative attention mechanism in place of the standard formulation.

\begin{lemma}\label{lemma:l2_attention}
Let $f \colon \mathcal{X} \subseteq \mathbb{R}^{n \times d} \rightarrow \mathcal{Y} \subseteq \mathbb{R}^d$ be a L2-MHA-based TBM~\cite{kim2021lipschitz}. In respect to Definition~\ref{def:expressivity}, the following holds:
\begin{itemize}[leftmargin=*]
    \item If $f$ employs Average pooling, then $f$ is $(\epsilon,\sigma,\gamma)$-expressive with 
    $
    \gamma = \frac{\epsilon}{\sigma \sqrt{n}} \left(\frac{d}{d-1}\right)^2 C_1 C_2
    $
    \item If $f$ employs Sum pooling, then $f$ is $(\epsilon,\sigma,\gamma)$-expressive with 
    $
    \gamma = \frac{\sqrt{n}\,\epsilon}{\sigma} \left(\frac{d}{d-1}\right)^2 C_1 C_2
    $
    \item If $f$ employs Last-token pooling, then $f$ is $(\epsilon,\sigma,\gamma)$-expressive with 
    $
    \gamma = \frac{\epsilon}{\sigma} \left(\frac{d}{d-1}\right)^2 C_1 C_2
    $
    \item If $f$ employs Max pooling, then $f$ is $(\epsilon,\sigma,\gamma)$-expressive with 
    $
    \gamma = \frac{\epsilon \sqrt{\min(n,d)}}{\sigma} \left(\frac{d}{d-1}\right)^2 C_1 C_2,
    $
\end{itemize}
{\small
$$ \text{with~}
    C_1 = 1 + \frac{\sqrt{n}}{\sqrt{d/H}} \left( 4W_O \left(\frac{n}{e}\right) + 1 \right) \left( \sqrt{\sum_h \lVert W^{Q,h} \rVert^2 \lVert W^{V,h} \rVert^2} \right) \lVert W^O \rVert, \quad
    C_2 = 1 + \lVert W_{FFN} \rVert.
$$
}
\end{lemma}

Lemma~\ref{lemma:l2_attention} analyzes an L2-based attention mechanism~\cite{kim2021lipschitz} in which the query and key matrices are tied. This constraint influences the constant $C_1$ in the expressivity bound, reflecting the interaction among shared parameters, sequence length $n$, number of heads $H$, and embedding dimension $d$. Compared to the standard dot-product formulation, this structure alters how the L2-kernel shapes the bound, resulting in slightly different expressivity dependencies.

The pooling-related terms in Lemma~\ref{lemma:l2_attention} are consistent with those derived under standard self-attention, and the same trade-offs between local and global context remain applicable. Similar behavior is observed in Swin~\cite{liu2021swin} and LipsFormer~\cite{qi2023lipsformer}, which employs scaled cosine similarity and normalizes the key, query, and value matrices to maintain Lipschitz-continuity. Lemma~\ref{lemma:lipsformer} provides the corresponding bound for a single-layer model with $H$ attention heads and window size $w$.

In both cases, the bounds reveal comparable structure, reinforcing that pooling remains a critical component in controlling the balance between local preservation and global aggregation in TBMs.

\begin{lemma}\label{lemma:lipsformer}
Let $f \colon \mathcal{X}\rightarrow\mathcal{Y}$ to be a function based on the LipsFormer~\citep{qi2023lipsformer} framework, with corresponding hyper-parameters $\nabla, \nu, \tau > 0$ and window size $w$. In respect to Definition~\ref{def:expressivity}, we have:
\begin{itemize}[leftmargin=*]
    \item If $f$ is based on Average pooling, then $f$ is $(\epsilon, \sigma, \gamma)$-expressive with $\gamma= \frac{\epsilon}{\sigma \times \sqrt{n}} \times \big(\frac{d}{d-1}\big)^2 C_1 C_2$ 
    \item If $f$ is based on Sum pooling, then $f$ is $(\epsilon, \sigma, \gamma)$-expressive with $\gamma= \frac{\sqrt{n} \times \epsilon}{\sigma} \times \big(\frac{d}{d-1}\big)^2 C_1 C_2$ 
    \item If $f$ is based on Last-token pooling, then $f$ is $(\epsilon, \sigma, \gamma)$-expressive with $\gamma= \frac{\epsilon}{\sigma} \times \big(\frac{d}{d-1}\big)^2 C_1 C_2$ 
    \item If $f$ is based on Max pooling, then $f$ is $(\epsilon, \sigma, \gamma)$-expressive with $\gamma = \frac{\epsilon \sqrt{\min(n,d)}}{\sigma} \times \big(\frac{d}{d-1}\big)^2 C_1 C_2, $ 
\end{itemize}
where
\begin{align*}
    & C_1 = 1 +
       \lVert W_O\lVert\sqrt{H}\max_{h}
         \Bigl\{2w(w-1)\nu\tau\nabla^{-\frac12}\lVert W^K_h\lVert
           + 2(w-1)\nu\tau\nabla^{-\tfrac12}\lVert W^Q_h \lVert + 2w\nu\nabla^{-\tfrac12} \lVert W^V_h\lVert \Bigr\}, \\
    & C_2 = \big(1 + \lVert W_{FFN} \lVert\big)
\end{align*}
\end{lemma}

\section{Experimental Validation}
\label{sec:experiments}
Our theoretical analysis shows that the pooling strategy influences the classifier’s expressivity bound via a leading multiplicative factor, which can be contractive ($1/\sqrt n$), non‐expansive ($1$), or expansive ($\sqrt n$ or $\sqrt{\min(n,d)}$). Contractive methods like Average pooling enhance stability by smoothing small variations, whereas expansive methods such as Last‐token and Sum pooling increase expressivity, but can be more sensitive to minor perturbations. Therefore, we posit that pooling should be selected based on task requirements: tasks emphasizing global context (\eg, image inpainting or text classification) benefit from contractive pooling, while those relying on local detail (\eg, next-token prediction) may perform better with expansive alternatives. In this section, we validate these theoretical insights through empirical evaluation, to demonstrate the applicability of our findings and provide practical guidance for choosing pooling strategies in different domains.

\textbf{Experimental Setup.} We evaluate how pooling choice affects downstream performance across three domains where TBMs have shown strong results: (a) computer vision, (b) natural language processing, and (c) time series analysis. For each modality, we select a diverse set of established benchmarks with tasks requiring \textit{global} and \textit{local} contexts. Across all settings, we examine commonly used pooling methods: (i) Last-token pooling (or CLS/EOS, depending on the task), (ii) Average (Avg) pooling, (iii) Sum pooling, and (iv) Max pooling. We also include two learnable strategies: (v) Attention (Attn) pooling, which uses a learnable latent dictionary attended by the model output~\citep{lee2025nv, tang2024pooling}, and (vi) Weighted Average (W-Avg) pooling, which learns scalar weights over token positions. Further details on training and evaluation protocols are provided in Appendix~\ref{app:experimental_details}.

\subsection{Expressivity Analysis}
We begin by empirically analyzing the expressivity of the pooling strategies under study, in accordance with the theoretical bounds introduced earlier. Using the framework in Section~\ref{sec:expressivity_of_TBM}, we define local neighborhoods by injecting Gaussian noise into input samples, scaled to a chosen $\epsilon$. We then compute the average distance between the resulting pooled outputs, yielding an empirical estimate of $\gamma$.


\begin{figure}[htbp]
  \centering

  \begin{minipage}[t]{0.44\textwidth}
    \centering
    \includegraphics[width=\textwidth,trim={0.2cm 0.2cm 0.2cm 0.2cm},clip]{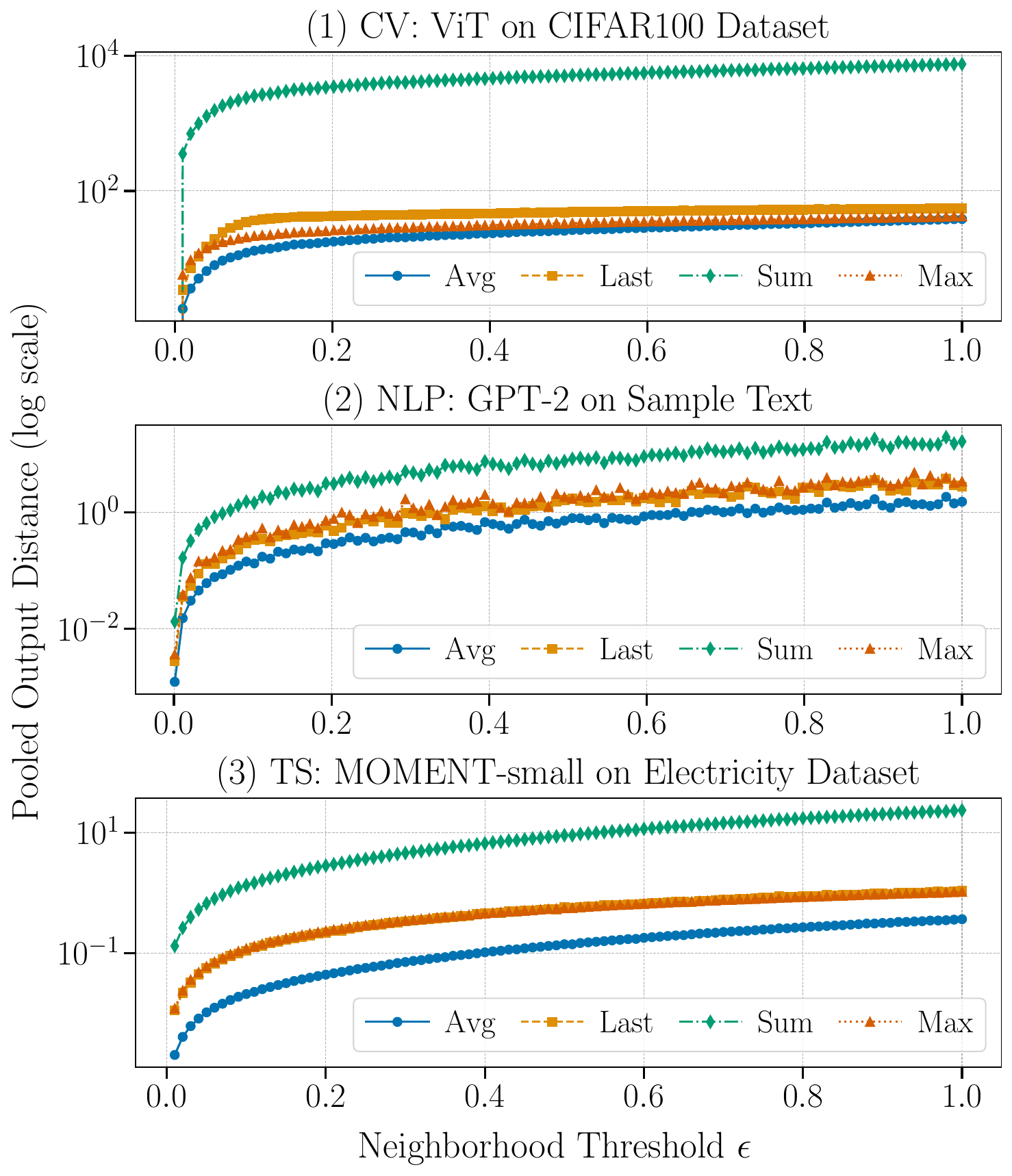}
    \label{fig:perturbation}
  \end{minipage}
  \hfill
  \begin{minipage}[t]{0.54\textwidth}
    \vspace{-19em}
    \centering
    \includegraphics[width=0.9\textwidth,trim={0.2cm 0.2cm 0.2cm 0.2cm},clip]{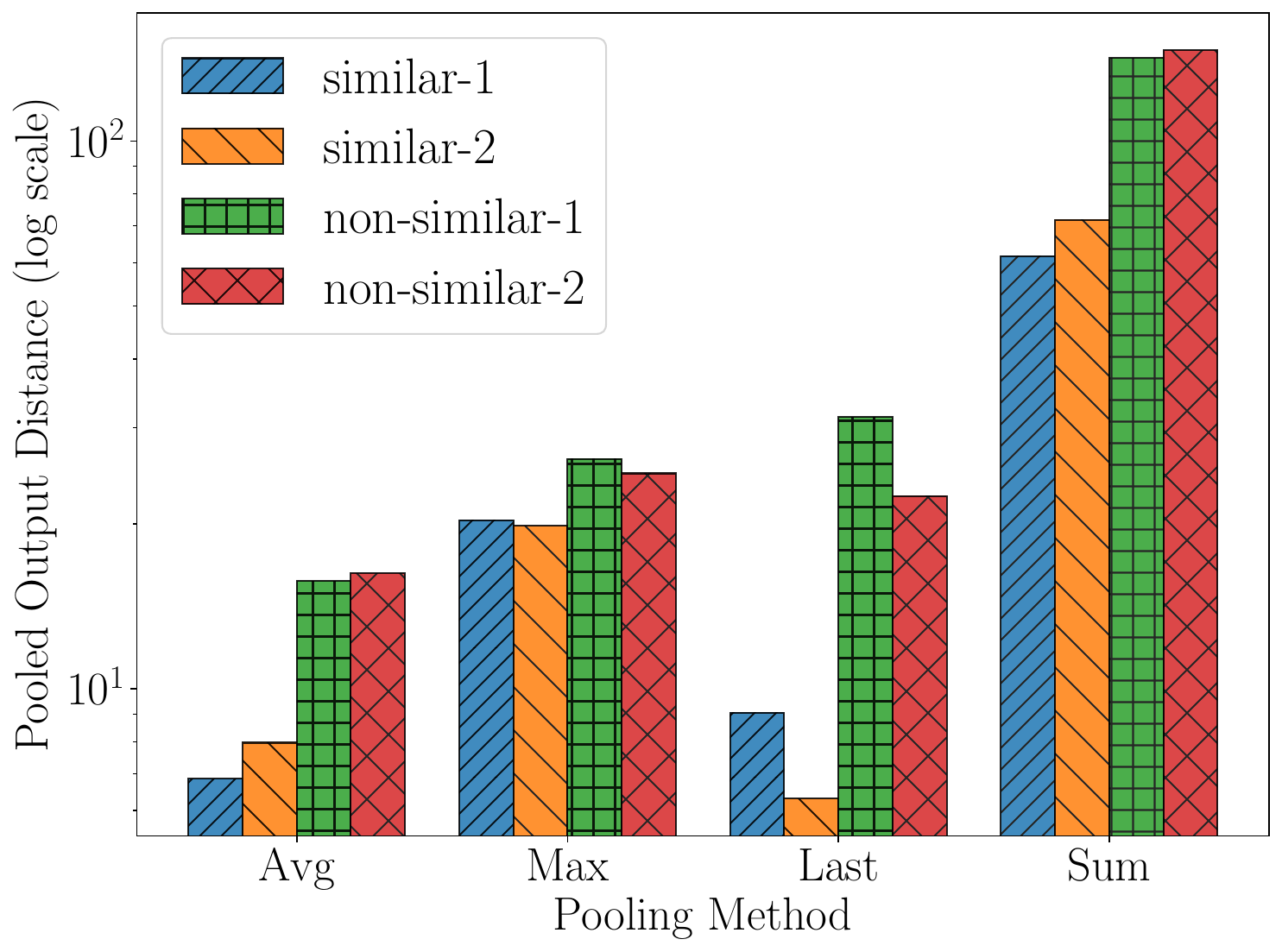}
    
    \vspace{0.2em}
    {\raggedright\scriptsize
    \texttt{original}: The \textit{stock} market saw significant \textit{fluctuations} last week.\\
    \texttt{similar-1}: The \textit{stock} market encountered notable \textit{swings} last week.\\
    \texttt{similar-2}: The \textit{equity} market experienced marked \textit{volatility} last week.\\
    \texttt{non-similar-1}: The \textit{stock} market encountered notable \textit{castle} last week. \\
    \texttt{non-similar-2}: The stock market encountered notable \textit{car} last week. \\
    }
    \label{fig:sentence_examples}
  \end{minipage}

  \caption{Empirical analysis of the expressivity power across modalities and pooling strategies. Left: Mean pooled‐output distance $\gamma$ versus perturbation $\epsilon$ across modalities highlighting the behavior of various methods. Right: pooled‐output distances for similar and dissimilar inputs, exemplifying expressivity of different strategies.}
  \label{fig:power_expressivity_analysis}
\end{figure}

Figure~\ref{fig:power_expressivity_analysis} (left part) presents results across different $\epsilon$ values and modalities, confirming the theoretical contrast between contractive and expansive pooling. Sum pooling shows high sensitivity to even small perturbations; as $\epsilon$ increases, its $\gamma$ grows rapidly. In contrast, Average pooling remains stable. To further illustrate this, Figure~\ref{fig:power_expressivity_analysis} (right part) shows how replacing a word with either a synonym or a semantically different term in NLP settings affects the pooled output across different strategies. This supports our hypothesis that expansive pooling better captures subtle variations, as required in tasks like sentiment analysis.

\subsection{Effect on the Downstream Performance}
\label{sec:experiment_effect_on_downstream_performance}
In line with our theoretical analysis, we empirically evaluate tasks with varying dependence on local versus global context to assess how different pooling strategies perform. This allows us to validate the extent to which each method's empirical behavior aligns with its expected theoretical properties.

\textbf{Computer Vision.} Results for image-based tasks are shown in Table~\ref{tab:vit_results}. As predicted by the analysis, Average pooling outperforms other fixed pooling methods in inpainting, segmentation, and in the MiniPlaces classification dataset, which benefits from modeling global structure. In contrast, Last-token (CLS) pooling generally yields the best results on classification tasks, particularly those involving large-scale or fine-grained datasets where local information is more critical. Max and Sum pooling consistently perform worse across all tasks. These trends also hold for alternative attention mechanisms, such as LipsFormer~\citep{qi2023lipsformer}, which employs scaled cosine similarity attention (see Lemma~\ref{lemma:lipsformer}). In this Swin-based model, which does not include a CLS token, Average pooling again achieves the best performance among fixed strategies, further emphasizing its strength in capturing localized context when a dedicated classification token is absent.

Among learnable pooling methods, Weighted Average pooling performs competitively across tasks, likely due to its capacity to adaptively weight tokens based on task-specific context. Attention-based pooling often underperforms, except in high-resource settings such as ImageNet-100, where sufficient supervision allows it to learn effective attention patterns. Its reduced reliability in low-resource or fine-grained tasks may be attributed to the additional complexity introduced by its parameterization.

\renewcommand{\arraystretch}{1.28}
\setlength{\tabcolsep}{5pt}
\begin{table}[]
\footnotesize
\caption{Mean and standard deviation of test metrics for computer vision tasks. Best performance per dataset and model is indicated in \textbf{bold}. Best performance among non-learnable pooling methods is \underline{underlined}.} %
\label{tab:vit_results}
\resizebox{\columnwidth}{!}{%
\begin{tabular}{llcccccccccc}
\hline
\multirow{2}{*}{\textbf{Model}}                                & \textbf{}        & \multicolumn{5}{c}{\textbf{Classification} (Accuracy)}                                                  & \multicolumn{3}{c}{\textbf{Inpainting} (MSE)}                           & \multicolumn{2}{c}{\textbf{Segmentation} (Accuracy)} \\ 
\cmidrule(lr){3-7} \cmidrule(lr){8-10} \cmidrule(lr){11-12}
                                                               & \textbf{Pooling} & CIFAR-10         & CIFAR-100        & ImageNet-100     & CUB-200-2011     & MiniPlaces       & CelebA               & OxfordFlower-102     & Oxford-IIIT Pet      & PascalVOC-Cls       & PascalVOC-Det       \\ \hline

\multirow{6}{*}{\rotatebox{90}{ViT-small}}  & Last (CLS)        & $\underline{90.35 \pm 0.03}$ & $\underline{76.42 \pm 0.08}$ & $\underline{87.84 \pm 0.03}$ & $\underline{\mathbf{77.45 \pm 0.13}}$ & $54.51 \pm 0.38$ & $0.246 \pm 0.001$    & $0.275 \pm 0.001$    & $0.268 \pm 0.005$    & $70.68 \pm 0.35$    & $29.11 \pm 0.31$    \\
                                                               & Avg          & $90.14 \pm 0.12$ & $76.26 \pm 0.33$ & $86.85 \pm 0.19$ & $71.48 \pm 0.18$ & $\underline{56.94 \pm 0.66}$ & $\underline{0.239 \pm 0.001}$    & $\underline{\mathbf{0.264 \pm 0.004}}$    & $\underline{\mathbf{0.260 \pm 0.008}}$    & $\underline{\mathbf{71.88 \pm 0.08}}$    & $\underline{\mathbf{35.10 \pm 0.25}}$    \\
                                                               & Sum              & $89.45 \pm 0.32$ & $75.72 \pm 0.29$ & $82.42 \pm 0.12$ & $68.37 \pm 0.41$ & $54.78 \pm 0.71$ & $0.285 \pm 0.021$    & $0.454 \pm 0.012$    & $0.282 \pm 0.007$    & $68.92 \pm 0.97$    & $25.21 \pm 0.07$    \\
                                                               & Max              & $84.96 \pm 0.25$ & $69.42 \pm 0.33$ & $83.21 \pm 0.23$ & $56.81 \pm 1.57$ & $51.30 \pm 0.40$ & $0.267 \pm 0.003$    & $0.343 \pm 0.008$    & $0.275 \pm 0.003$    & $69.43 \pm 1.83$    & $21.57 \pm 0.87$    \\
                                                               \cmidrule(lr){3-12}
                                                               & W-Avg& $\mathbf{90.88 \pm 0.11}$ & $\mathbf{78.41 \pm 0.02}$ & $\mathbf{87.29 \pm 0.12}$ & $72.97 \pm 0.31$ & $\mathbf{56.95 \pm 0.11}$ & $0.247 \pm 0.001$    & $0.286 \pm 0.003$    & $0.267 \pm 0.007$    & $71.83 \pm 0.02$    & $33.63 \pm 1.65$    \\
                                                               & Attn  & $89.81 \pm 0.22$ & $74.61 \pm 0.66$ & $87.84 \pm 0.06$ & $64.78 \pm 1.29$ & $53.08 \pm 1.17$ & $\mathbf{0.192 \pm 0.002}$    & $0.289 \pm 0.007$    & $0.273 \pm 0.011$    & $69.85 \pm 0.84$    & $24.22 \pm 0.39$    \\ \cline{2-12} 
\multirow{6}{*}{\rotatebox{90}{LipsFormer}}
                                           & Last         & $87.70 \pm 0.03$ & $66.53 \pm 0.08$ & $86.18 \pm 0.17$ & $48.77 \pm 0.38$ & $48.84 \pm 1.03$ & $0.225 \pm 0.002$ & $0.342 \pm 0.003$ & $0.260 \pm 0.003$ & $70.04 \pm 0.01$    & $34.22 \pm 0.21$    \\
                                           & Avg          & $\underline{\mathbf{93.26 \pm 0.03}}$ & $\underline{\mathbf{78.05 \pm 0.01}}$ & $\underline{89.74 \pm 0.08}$ & $\underline{\mathbf{72.23 \pm 0.29}}$ & $\underline{65.04 \pm 0.12}$ & $\underline{0.216 \pm 0.001 }$ & $\underline{0.306 \pm 0.005}$ & $\underline{\mathbf{0.237 \pm 0.003}}$ & $\underline{76.51 \pm 0.26}$    & $\underline{34.59 \pm 0.34}$    \\
                                           & Sum          & $90.82 \pm 0.65$ & $72.70 \pm 0.13$ & $87.44 \pm 0.28$ & $65.96 \pm 0.13$ & $57.38 \pm 0.08$ & $0.268 \pm 0.004$ & $0.364 \pm 0.088$ & $0.308 \pm 0.080$ & $73.62 \pm 0.05$    & $24.27 \pm 0.28$    \\
                                           & Max          & $90.75 \pm 0.35$ & $70.73 \pm 0.15$ & $87.68 \pm 0.32$ & $59.46 \pm 0.19$ & $56.53 \pm 0.18$ & $0.234 \pm 0.001$ & $0.369 \pm 0.043$ & $0.247 \pm 0.002$ & $71.43 \pm 1.62$    & $22.64 \pm 3.84$    \\
                                           \cmidrule(lr){3-12}
                                           & W-Avg        & $93.28 \pm 0.09$ & $78.00 \pm 0.12$ & $89.48 \pm 0.12$ & $72.23 \pm 0.07$ & $\mathbf{65.21 \pm 0.08}$ & $0.223 \pm 0.001$ & $0.325 \pm 0.002$ & $0.251 \pm 0.002  $ & $76.28 \pm 0.15$    & $34.39 \pm 0.40$    \\
                                           & Attn         & $92.36 \pm 0.09$ & $76.10 \pm 0.03$ & $\mathbf{89.92 \pm 0.16}$ & $67.01 \pm 0.06$ & $63.76 \pm 0.05$ & $\mathbf{0.148 \pm 0.003}$ & $\mathbf{0.287 \pm 0.004}$ & $0.279 \pm 0.023$ & $\mathbf{77.37 \pm 0.64}$    & $\mathbf{36.07 \pm 0.61}$    \\
\hline
\end{tabular}%
}
\end{table}

\textbf{NLP.} The impact of pooling across downstream NLP tasks and models is summarized in Table~\ref{tab:transposed-nlp-results}. The results support our theoretical analysis: no single pooling method is optimal across all tasks, and the best strategy depends on the task's contextual requirements. For tasks requiring global context, such as classification and semantic similarity, global pooling methods like Average or Sum significantly outperform Last-token pooling. Conversely, Last-token pooling yields superior performance in next-token prediction tasks.

\begin{table}[h]
\centering
\tiny
\caption{Mean and standard deviation of test metrics for NLP tasks. Best performance per dataset and model is indicated in \textbf{bold}. Best performance among non-learnable pooling methods is \underline{underlined}. (-) indicates non-applicable, as the model uses bidirectional attention mechanism.}
\label{tab:transposed-nlp-results}
\renewcommand{\arraystretch}{0.9}
\resizebox{\columnwidth}{!}{%
\begin{tabular}{lclcccccc}
\toprule
 & \shortstack{\textbf{Dataset}}   & \textbf{Pooling}           & BERT      & L2-GPT-2  & GPT-2   & Qwen 2.5   & Mistral-7B  & Llama3-8B   \\
\midrule
  \multirow{12}{*}{\rotatebox[origin=c]{90}{\textbf{Similarity}}}& \multirow{6}{*}{\rotatebox[origin=c]{90}{\shortstack{STSB\\(Spearman)}}}    & Last     & $0.587 \pm 0.009$ & $0.375 \pm 0.006$    & $0.602 \pm 0.005$  & $0.286 \pm 0.005$    & $0.514 \pm 0.001$ & $0.017 \pm 0.085$   \\
    & & Avg           & $0.713 \pm 0.008$ & $0.659 \pm 0.004$    & $\underline{0.671 \pm 0.004}$  & $\underline{0.620 \pm 0.005}$    & $\underline{0.635 \pm 0.005}$ & $0.624 \pm 0.004$    \\
     & & Sum               & $\underline{0.714 \pm 0.009}$ & $\underline{0.660 \pm 0.004}$    & $0.670 \pm 0.003$  & $0.619 \pm 0.005$    & $0.634 \pm 0.007$ & $\underline{0.626 \pm 0.005}$   \\
    & & Max       & $0.695 \pm 0.013$ & $0.648 \pm 0.002$   & $0.653 \pm 0.002$ & $0.560 \pm 0.011$     & $0.449 \pm 0.017$ & $0.487 \pm 0.003$   \\
     \cmidrule(lr){4-9}
    & & W-Avg  & $\mathbf{0.727 \pm 0.002}$ & $0.562 \pm 0.002$   & $0.568 \pm 0.003$  & $\mathbf{0.671 \pm 0.002}$    & $\mathbf{0.653 \pm 0.004}$ & $\mathbf{0.673 \pm 0.001}$   \\
     & & Attn & $0.703 \pm 0.013$ & $\mathbf{0.678 \pm 0.016}$    & $\mathbf{0.677 \pm 0.010}$  & $0.616 \pm 0.014$    & $0.452 \pm 0.088$ & $0.496 \pm 0.037$    \\
     \cmidrule(lr){3-9}
   & \multirow{6}{*}{\rotatebox[origin=c]{90}{\shortstack{Hellaswag\\(F1)}}}   & Last     & $0.307 \pm 0.001$ & $0.231 \pm 0.025$    & $0.264 \pm 0.024$  & $0.344 \pm 0.001$    & $\underline{0.770 \pm 0.002}$ & $0.678 \pm 0.002$   \\
  & & Avg           & $0.315 \pm 0.000$ & $\underline{\mathbf{0.297 \pm 0.001}}$    & $\underline{\mathbf{0.305 \pm 0.000}}$  & $\underline{0.432 \pm 0.000}$    & $0.769 \pm 0.000$ & $\underline{0.734 \pm 0.000}$   \\
   & & Sum               & $\underline{0.316 \pm 0.001}$ & $\underline{\mathbf{0.297 \pm 0.001}}$    & $\underline{\mathbf{0.305 \pm 0.005}}$  & $0.431 \pm 0.001$    & $0.769 \pm 0.000$ & $\underline{0.734 \pm 0.001}$   \\
   & & Max       & $0.298 \pm 0.003$ & $0.291 \pm 0.002$   & $0.293 \pm 0.002$ & $0.364 \pm 0.003$   & $0.709 \pm 0.001$ & $0.682 \pm 0.005$   \\
     \cmidrule(lr){4-9}
   & & W-Avg  & $\mathbf{0.318 \pm 0.001}$ & $0.284 \pm 0.001$   & $0.278 \pm 0.001$ & $\mathbf{0.452 \pm 0.000}$   & $\mathbf{0.801 \pm 0.000}$ & $\mathbf{0.763 \pm 0.000}$  \\
   & & Attn & $0.295 \pm 0.004$ & $0.264 \pm 0.012$    & $0.260 \pm 0.038$  & $0.410 \pm 0.009$  & $0.737 \pm 0.018$ & $0.459 \pm 0.271$\\
   \midrule
    \multirow{12}{*}{\rotatebox[origin=c]{90}{\textbf{Classification}}} & \multirow{6}{*}{\rotatebox[origin=c]{90}{\shortstack{Banking\\(Accuracy)}}}      & Last    & $77.175 \pm 0.112$    & $17.014 \pm 0.323$       & $45.528 \pm 0.244$     & $23.243 \pm 0.378$       & $74.847 \pm 1.584$     & $45.107 \pm 0.403$      \\
   & & Avg           & $\underline{85.142 \pm 0.002}$     & $\underline{86.882 \pm 0.076}$       & $\underline{86.497 \pm 0.103}$     & $\underline{83.486 \pm 0.402}$       & $\underline{88.183 \pm 0.402}$     & $\underline{87.558 \pm 0.323}$     \\
      & & Sum               & $83.863 \pm 0.806$     & $84.130 \pm 0.807$       & $83.724 \pm 1.047$     & $79.164 \pm 0.776$       & $86.442 \pm 0.703$     & $82.984 \pm 1.107$     \\
      & & Max       & $80.785 \pm 0.008$    & $83.091 \pm 0.315$       & $83.023 \pm 0.226$     & $74.007 \pm 1.106$      & $74.890 \pm 2.476$     & $75.324 \pm 1.602$     \\
     \cmidrule(lr){4-9}
      & & W-Avg  & $84.987 \pm 0.153$    & $67.253 \pm 0.275$       & $73.989 \pm 0.223$     & $\mathbf{85.513 \pm 0.221}$       & $\mathbf{89.271 \pm 0.426}$     & $\mathbf{88.928 \pm 0.296}$     \\
      & & Attn & $\mathbf{86.558 \pm 0.559}$    & $\mathbf{87.340 \pm 0.302}$       & $\mathbf{86.968 \pm 0.604}$     & $83.792 \pm 1.813$       & $73.352 \pm 1.731$     & $51.143 \pm 34.316$     \\
     \cmidrule(lr){3-9}
    & \multirow{6}{*}{\rotatebox[origin=c]{90}{\shortstack{Tweet\\(Accuracy)}}} & Last    & $67.621 \pm 0.083$    & $48.383 \pm 0.204$       & $63.899 \pm 0.143$     & $51.738 \pm 0.713$       & $56.693 \pm 0.492$     & $60.446 \pm 0.538$     \\
 & & Avg           & $\underline{69.348 \pm 0.128}$     & $\underline{67.593 \pm 0.103}$       & $\underline{68.573 \pm 0.208}$     & $\underline{68.961 \pm 0.330}$       & $\underline{67.231 \pm 0.218}$     & $\underline{67.328 \pm 0.223}$     \\
   &   & Sum               & $65.381 \pm 0.273$    & $63.149 \pm 0.195$       & $64.122 \pm 0.317$     & $59.625 \pm 2.702$       & $64.151 \pm 1.623$     & $64.231 \pm 2.029$     \\
    &  & Max       & $67.070 \pm 0.177$    & $61.392 \pm 0.254$      & $61.262 \pm 0.178$    & $62.971 \pm 0.433$      & $64.897 \pm 2.402$     & $64.804 \pm 1.089$      \\
     \cmidrule(lr){4-9}
     & & W-Avg  & $\mathbf{69.560 \pm 0.121}$     & $62.458 \pm 0.123$       & $60.718 \pm 0.224$     & $66.031 \pm 0.403$       & $\mathbf{67.293 \pm 0.228}$     & $\mathbf{67.476 \pm 0.185}$     \\
     & & Attn & $69.455 \pm 0.879$    & $\mathbf{69.131 \pm 0.529}$       & $\mathbf{70.844 \pm 0.821}$     & $\mathbf{70.627 \pm 0.663}$       & $46.584 \pm 6.903$     & $55.890 \pm 13.232$    \\
      \midrule
      \multirow{6}{*}{\rotatebox[origin=c]{90}{\textbf{Next Token}}} & \multirow{6}{*}{\rotatebox[origin=c]{90}{\shortstack{Tiny Stories \\(Top-10 Acc)}}}& Last    & ---             & $\underline{\mathbf{82.170 \pm 0.225}}$       & $\underline{\mathbf{84.569 \pm 0.001}}$     & $\underline{86.634 \pm 0.428}$       & $\underline{\mathbf{89.948 \pm 0.092}}$     & $\underline{\mathbf{90.608 \pm 0.227}}$      \\
  & & Avg           & ---             & $37.718 \pm 0.229$       & $38.826 \pm 0.435$     & $40.654 \pm 0.327$       & $61.047 \pm 0.228$     & $56.012 \pm 0.337$     \\
      & & Sum               & ---             & $24.852 \pm 0.428$       & $29.362 \pm 0.893$     & $28.911 \pm 0.630$       & $50.481 \pm 0.625$     & $42.731 \pm 1.318$     \\
      & & Max       & ---             & $35.450 \pm 0.332$      & $36.022 \pm 0.257$    & $37.229 \pm 0.253$      & $9.900 \pm 0.108$       & $32.199 \pm 0.252$     \\
     \cmidrule(lr){5-9}
      & & W-Avg  & ---             & $61.310 \pm 0.018$      & $53.998 \pm 0.348$     & $\mathbf{86.859 \pm 0.212}$       & $89.160 \pm 0.118$     & $87.389 \pm 0.019$    \\
      & & Attn & ---             & $50.364 \pm 0.287$       & $53.299 \pm 0.338$     & $55.970 \pm 0.302$       & $14.210 \pm 5.239$     & $11.138 \pm 8.959$       \\
\bottomrule
\end{tabular}
} %
\end{table}

These trends hold across models, although the performance gap between local and global pooling narrows for larger architectures (\eg, Mistral-7B~\cite{jiang2023mistral7b} and Llama~\citep{touvron2023llama}). In such models, Weighted Average pooling matches or exceeds the best-performing non-learnable methods, due to its ability to adaptively approximate effective pooling strategies. In smaller models, particularly the GPT-2~\cite{radford2019language} family, Attention pooling performs well on global-context tasks, often outperforming fixed global methods like Average or Sum. However, this advantage does not consistently generalize to larger models or all task types. Additional results on larger models are provided in Table \ref{tab:nlp-results-transposed} (Appendix \ref{app:additional_results_nlp}). We additionally analyze the effect

\textbf{Time Series.} As shown in Table~\ref{tab:ts_results}, for time series classification tasks, Last-token and Max pooling generally yield the worst results, as they focus on local features and fail to capture the broader context required for accurate classification. In contrast, Attention-based pooling consistently achieves the highest performance, due to its ability to assign task-specific weights to different input segments during joint training. Sum pooling outperforms Average pooling in several cases, that can be attributed to the larger norm of summed representations, which results in stronger gradients and faster learning under fixed training hyperparameters.

Forecasting shows that Last-token pooling yields the best results overall. This is consistent with its ability to retain fine-grained temporal details, which are critical for predicting future values based on recent history. In imputation tasks, results indicate that Attention-based pooling again performs best, while Sum pooling performs the worst. This supports the hypothesis that given sufficient data and training time, Attention pooling can adaptively focus on the most relevant parts of the sequence, whereas Sum pooling may amplify irrelevant noise and obscure important local patterns required for accurate imputation. Additional results for all datasets are provided in Appendix~\ref{app:additional_results}.

\begin{table}
\caption{Mean and standard deviation test metrics in time series tasks. Best performance per dataset and model in \textbf{bold}. Best performance among non-learnable pooling methods is \underline{underlined}.}
\label{tab:ts_results}
\renewcommand{\arraystretch}{1.4}  
\resizebox{\columnwidth}{!}{%
    \begin{tabular}{rlcccccccccc}
    \hline
      & \textbf{} & \multicolumn{4}{c}{\textbf{Classification} (Accuracy)} & \multicolumn{3}{c}{\textbf{Forecasting} (MSE)} & \multicolumn{3}{c}{\textbf{Imputation} (MSE)} \\
     \cmidrule(lr){3-6} \cmidrule(lr){7-9} \cmidrule(lr){10-12}
    \textbf{Model} & \textbf{Pooling} & \makecell{ECG200} & \makecell{Electric\\Devices} & \makecell{FordA} & \makecell{SmallKitchen\\Appliances} & \makecell{ETTh1} & \makecell{Electricity} & \makecell{Traffic} & \makecell{ETTh1} & \makecell{Electricity} & \makecell{Traffic} \\
    \cmidrule{1-12}
    \multirow{6}{*}{\rotatebox{90}{MOMENT-small}} & Last & $72.29 \pm 0.59$ & $60.45 \pm 0.48$ & $76.39 \pm 0.15$ & $64.06 \pm 0.75$ & \underline{$\mathbf{0.082 \pm 0.000}$} & \underline{$\mathbf{0.400 \pm 0.001}$} & \underline{$0.273 \pm 0.001$} & $0.081 \pm 0.002$ & $0.753 \pm 0.010$ & $1.709 \pm 0.016$ \\
     & Avg & $65.19 \pm 0.00$ & $61.40 \pm 0.54$ & $88.12 \pm 0.14$ & $62.40 \pm 1.01$ & $0.105 \pm 0.000$ & $0.790 \pm 0.000$ & $1.724 \pm 0.001$ & \underline{$0.080 \pm 0.002$} & $0.774 \pm 0.013$ & $1.583 \pm 0.016$ \\
     & Sum & \underline{$\mathbf{80.35 \pm 1.82}$} & $60.62 \pm 2.85$ & \underline{$92.93 \pm 0.43$} & \underline{$67.13 \pm 2.95$} & $0.103 \pm 0.001$ & $0.826 \pm 0.014$ & $1.422 \pm 0.050$ & $0.106 \pm 0.008$ & $1.072 \pm 0.143$ & $2.130 \pm 0.238$ \\
     & Max & $65.19 \pm 0.00$ & \underline{$61.78 \pm 1.24$} & $90.42 \pm 0.43$ & $63.94 \pm 1.76$ & $0.106 \pm 0.001$ & $0.800 \pm 0.001$ & $1.759 \pm 0.002$ & $0.082 \pm 0.002$ & \underline{$0.720 \pm 0.013$} & \underline{$1.211 \pm 0.051$} \\
     \cmidrule(lr){3-12}
     & W-Avg & $65.19 \pm 0.00$ & $62.54 \pm 0.67$ & $87.62 \pm 0.64$ & $62.40 \pm 1.01$ & $0.105 \pm 0.000$ & $0.539 \pm 0.008$ & $0.954 \pm 0.014$ & $0.080 \pm 0.002$ & $0.774 \pm 0.013$ & $1.582 \pm 0.016$ \\
     & Attn & $78.84 \pm 3.04$ & $\mathbf{62.95 \pm 2.04}$ & $\mathbf{93.60 \pm 0.37}$ & $\mathbf{67.63 \pm 2.95}$ & $0.106 \pm 0.001$ & $0.475 \pm 0.059$ & $\mathbf{0.258 \pm 0.002}$ & $\mathbf{0.076 \pm 0.003}$ & $\mathbf{0.379 \pm 0.028}$ & $\mathbf{0.265 \pm 0.015}$ \\
    \cmidrule{2-12} 
    \multirow{6}{*}{\rotatebox{90}{MOMENT-base}} & Last & $71.01 \pm 1.94$ & $63.05 \pm 0.62$ & $83.46 \pm 0.40$ & $63.25 \pm 0.54$ & \underline{$\mathbf{0.081 \pm 0.000}$} & \underline{$\mathbf{0.397 \pm 0.000}$} & \underline{$\mathbf{0.265 \pm 0.000}$} & $0.082 \pm 0.001$ & $0.760 \pm 0.011$ & $1.668 \pm 0.020$ \\
     & Avg & $65.19 \pm 0.00$ & $63.71 \pm 0.40$ & $89.75 \pm 0.37$ & $63.05 \pm 1.60$ & $0.105 \pm 0.000$ & $0.785 \pm 0.001$ & $1.719 \pm 0.001$ & \underline{$0.082 \pm 0.001$} & $0.769 \pm 0.016$ & $1.424 \pm 0.020$ \\
     & Sum & \underline{$\mathbf{83.30 \pm 1.51}$} & \underline{$64.30 \pm 0.98$} & \underline{$92.51 \pm 0.13$} & \underline{$66.55 \pm 1.48$} & $0.101 \pm 0.002$ & $0.805 \pm 0.009$ & $1.029 \pm 0.023$ & $0.103 \pm 0.002$ & $1.006 \pm 0.161$ & $1.301 \pm 0.089$ \\
     & Max & $65.78 \pm 1.32$ & $62.31 \pm 1.14$ & $90.23 \pm 0.80$ & $64.77 \pm 2.12$ & $0.106 \pm 0.000$ & $0.796 \pm 0.001$ & $1.747 \pm 0.003$ & $0.082 \pm 0.002$ & \underline{$0.707 \pm 0.014$} & \underline{$1.000 \pm 0.113$} \\
     \cmidrule(lr){3-12}
     & W-Avg & $65.19 \pm 0.00$ & $\mathbf{64.48 \pm 0.69}$ & $89.93 \pm 0.36$ & $63.20 \pm 1.74$ & $0.105 \pm 0.000$ & $0.511 \pm 0.006$ & $0.857 \pm 0.014$ & $0.082 \pm 0.001$ & $0.769 \pm 0.015$ & $1.425 \pm 0.020$ \\
     & Attn & $82.57 \pm 1.19$ & $64.15 \pm 1.37$ & $\mathbf{92.74 \pm 0.35}$ & $\mathbf{66.76 \pm 1.39}$ & $0.096 \pm 0.009$ & $0.507 \pm 0.148$ & $0.306 \pm 0.037$ & $\mathbf{0.072 \pm 0.003}$ & $\mathbf{0.370 \pm 0.013}$ & $\mathbf{0.273 \pm 0.004}$ \\
    \cmidrule{2-12}
    \multirow{6}{*}{\rotatebox{90}{MOMENT-large}} & Last & $72.67 \pm 0.95$ & $61.10 \pm 0.53$ & $79.61 \pm 0.31$ & \underline{$65.45 \pm 1.71$} & \underline{$\mathbf{0.080 \pm 0.000}$} & \underline{$\mathbf{0.379 \pm 0.000}$} & \underline{$\mathbf{0.272 \pm 0.001}$} & $0.082 \pm 0.001$ & $0.752 \pm 0.014$ & $1.699 \pm 0.017$ \\
     & Avg & $65.19 \pm 0.00$ & $61.72 \pm 0.93$ & $85.98 \pm 0.53$ & $60.42 \pm 0.91$ & $0.105 \pm 0.000$ & $0.778 \pm 0.000$ & $1.711 \pm 0.001$ & \underline{$0.081 \pm 0.002$} & $0.753 \pm 0.018$ & $1.508 \pm 0.011$ \\
     & Sum & \underline{$75.97 \pm 2.82$} & \underline{$\mathbf{63.45 \pm 0.76}$} & \underline{$92.85 \pm 0.42$} & $64.92 \pm 7.21$ & $0.101 \pm 0.001$ & $0.688 \pm 0.007$ & $0.782 \pm 0.003$ & $0.095 \pm 0.008$ & $0.887 \pm 0.075$ & \underline{$1.150 \pm 0.065$} \\
     & Max & $65.19 \pm 0.00$ & $60.08 \pm 0.54$ & $87.14 \pm 0.39$ & $62.27 \pm 0.64$ & $0.104 \pm 0.001$ & $0.785 \pm 0.001$ & $1.743 \pm 0.000$ & $0.081 \pm 0.001$ & \underline{$0.705 \pm 0.008$} & $1.158 \pm 0.025$ \\
     \cmidrule(lr){3-12}
     & W-Avg & $65.19 \pm 0.00$ & $61.48 \pm 0.90$ & $86.07 \pm 0.36$ & $60.54 \pm 0.53$ & $0.105 \pm 0.001$ & $0.534 \pm 0.003$ & $0.951 \pm 0.005$ & $0.081 \pm 0.002$ & $0.753 \pm 0.018$ & $1.503 \pm 0.011$ \\
     & Attn & $\mathbf{78.73 \pm 3.09}$ & $62.83 \pm 1.38$ & $\mathbf{93.02 \pm 0.26}$ & $\mathbf{65.64 \pm 3.54}$ & $0.097 \pm 0.004$ & $0.684 \pm 0.003$ & $0.423 \pm 0.028$ & $\mathbf{0.072 \pm 0.003}$ & $\mathbf{0.295 \pm 0.006}$ & $\mathbf{0.231 \pm 0.009}$ \\
    \bottomrule
    \end{tabular}
}
\end{table}

\subsection{Positional Weighting in Weighted Average Pooling}
\label{subsec:weight-avg-conv-results}
We investigate the role of learnable weights in the Weighted Average pooling layer through an ablation-style comparison, while keeping all other model components and training settings constant.

\begin{figure}[ht]
    \centering
    \begin{minipage}{0.45\textwidth}
        \centering
        \includegraphics[width=\linewidth]{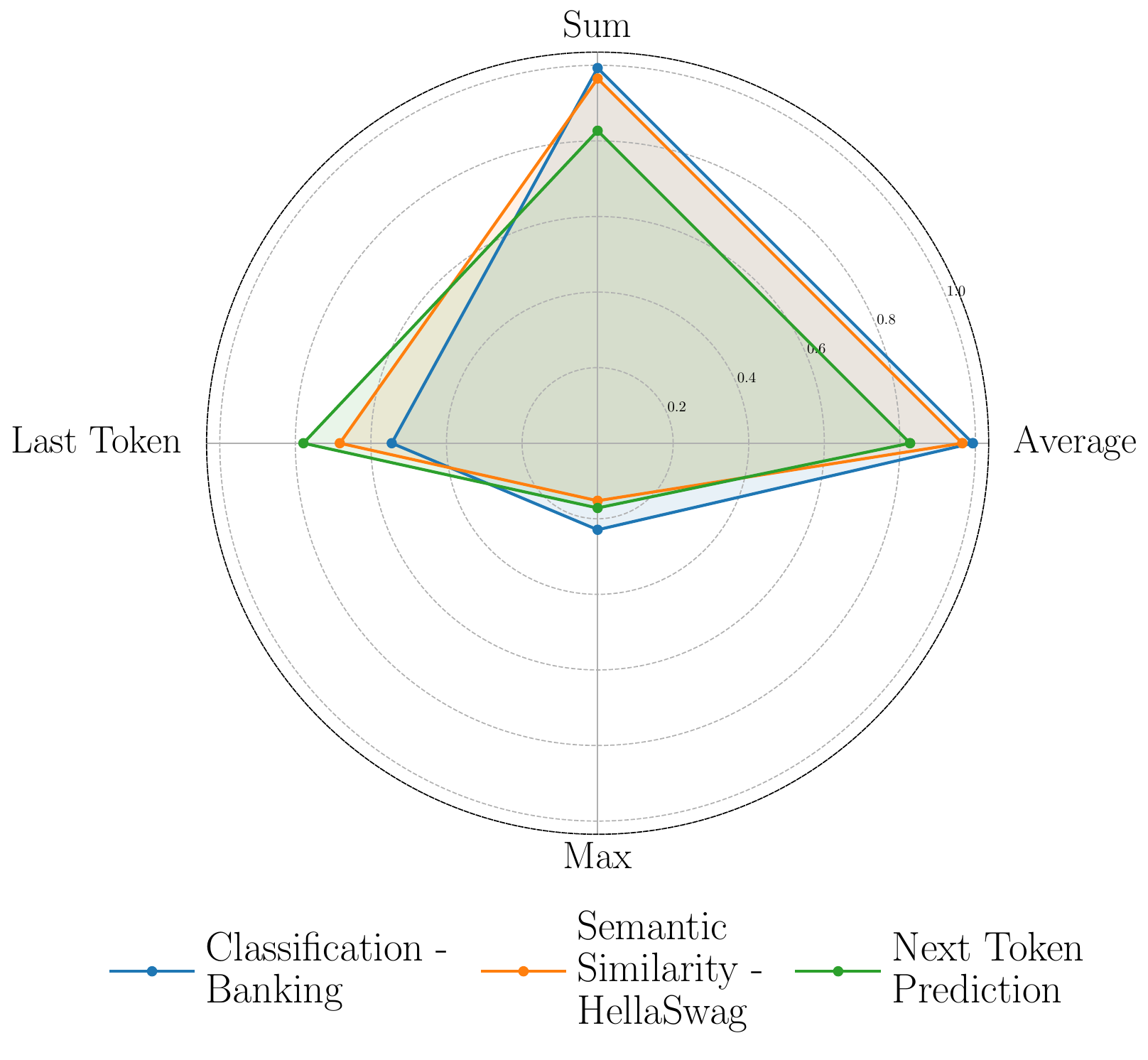}
    \end{minipage}
    \hfill
    \begin{minipage}{0.54\textwidth}
        \centering
        \includegraphics[width=\linewidth]{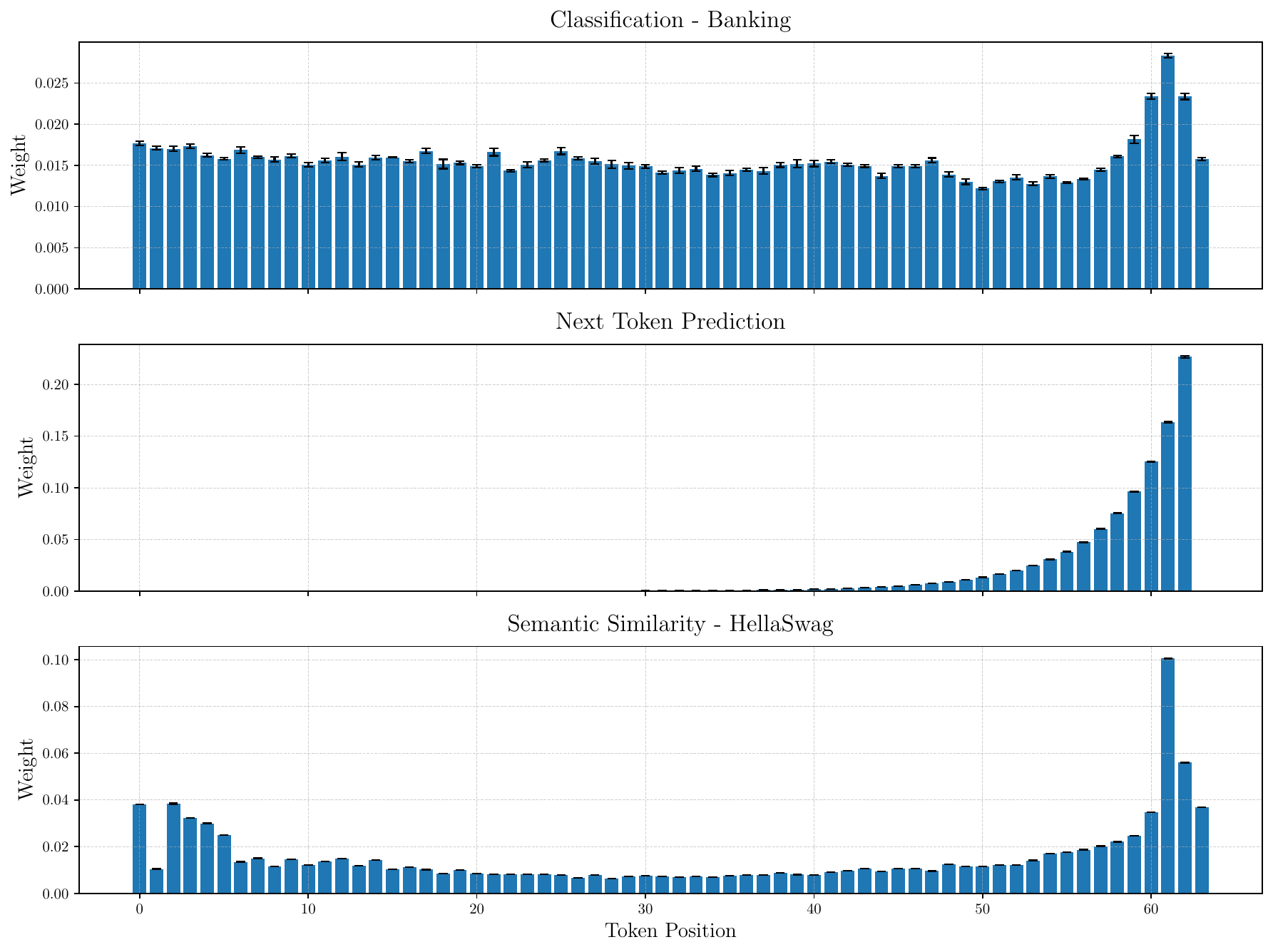}
    \end{minipage}
    \caption{\textbf{Left:} Cosine similarity between W-Avg pooling and other pooling methods, showing task-dependent alignment. \textbf{Right:} The distribution of the learned weights in the W-avg pooling, illustrating the adaptability of the pooling mechanism.}
    \label{fig:w-avg-analysis}
\end{figure}

Figure~\ref{fig:w-avg-analysis} depicts results for the Mistral~\citep{jiang2023mistral7b} backbone on representative tasks. In each case, the trainable variant converges to weight distributions that closely resemble the expected optimal pooling strategy for each task: near-uniform weights on text classification, a strong focus on the final token for next-token prediction, and intermediate patterns for tasks that require both local detail and global context.  Consistent with our theoretical analysis, it shows that the benefits of Weighted Average pooling arise from its ability to mimic the best-performing fixed pooling strategy which depends on the task's contextual demands. Full results for all models and datasets are provided in Appendix~\ref{app:additional_results}.


\section{Conclusion}
\label{sec:conclusion}
We presented an end-to-end study of pooling in Transformer models by introducing a formal expressivity framework and deriving closed-form bounds for standard pooling mechanisms. Our theory extends to alternative attention variants and shows that pooling is the key factor balancing local detail and global aggregation. Extensive experiments in vision, language, and time-series tasks validate these bounds: contractive pooling excels on global-context tasks, expansive pooling captures fine-grained distinctions, and learnable methods converge to the best-balanced strategy when given sufficient training data and time. No single pooling method dominates universally, highlighting the need for task-specific pooling design. In data-scarce regimes, our guidelines enable principled selection of pooling methods based on task demands and inductive biases. These contributions bridge theory and practice, enhance understanding of Transformer expressivity, and inform the design of adaptive pooling schemes for diverse downstream applications.

\textbf{Limitations and future work.} While our theoretical and empirical findings provide a solid foundation for understanding pooling in Transformer architectures, several limitations remain. Our evaluation, though broad, is limited to frozen backbones, leaving the effect of jointly adapting pooling and the backbone under end-to-end training less explored. The expressivity bounds we establish are necessary but not sufficient for optimal performance; bridging the gap between theoretical capacity and practical effectiveness remains an open challenge. In future work, we aim to develop hybrid pooling methods that dynamically balance global smoothing with token-level sensitivity while adapting to the TBM's inherent smoothing behavior. In addition to the analysis provided in Appendix \ref{appendix:adv_robustness}, we also plan to investigate further how pooling impacts robustness under perturbations, derive scaling laws that govern pooling performance as datasets grow, and refine our theoretical framework with tighter bounds that identify when specific strategies are provably optimal.

\begin{ack}
    This work was partially supported by Wallenberg Autonomous Systems Program (WASP). 
    L.Z. gratefully acknowledges NXAI GmbH for supporting his participation in NeurIPS 2025.
\end{ack}

\bibliographystyle{plain}
\bibliography{references}

\newpage

\setcounter{page}{1}

\appendix

\vbox{%
\hsize\textwidth
\linewidth\hsize
\vskip 0.1in
\centering
{\LARGE\bf Supplementary Material: \par}
\vspace{2\baselineskip}
}

\section{Proof of Theorem \ref{theo:bound_transformer}}\label{app:proof_theo_expressivity}
\begin{theorem*}
Let $f \colon \mathcal{X} \subseteq \mathbb{R}^{n \times d} \rightarrow \mathcal{Y} \subseteq \mathbb{R}^d$ be a TBM following the framework introduced in Section~\ref{sec:preliminaries}. In respect to Definition~\ref{def:expressivity}, we have:  
\begin{itemize}[leftmargin=*]
    \item If $f$ employs Average pooling, then $f$ is $(\epsilon,\sigma,\gamma)$-expressive with 
    $
    \gamma = \frac{\epsilon}{\sigma \sqrt{n}} \left(\frac{d}{d-1}\right)^2 C_1 C_2
    $
    \item If $f$ employs Sum pooling, then $f$ is $(\epsilon,\sigma,\gamma)$-expressive with 
    $
    \gamma = \frac{\sqrt{n}\, \epsilon}{\sigma} \left(\frac{d}{d-1}\right)^2 C_1 C_2
    $
    \item If $f$ employs Last-token pooling, then $f$ is $(\epsilon,\sigma,\gamma)$-expressive with 
    $
    \gamma = \frac{\epsilon}{\sigma} \left(\frac{d}{d-1}\right)^2 C_1 C_2
    $
    \item If $f$ employs Max pooling, then $f$ is $(\epsilon,\sigma,\gamma)$-expressive with 
    $
    \gamma = \frac{\epsilon\sqrt{\min(n,d)}}{\sigma} \left(\frac{d}{d-1}\right)^2 C_1 C_2,
    $
\end{itemize}
where
{\small
$$
    C_1 = 1 +\lVert W_O \lVert \sqrt{H}\, \max_h \Biggl[ \Bigl\lVert W^{V,h} \Bigr\rVert \left(4\,\frac{n}{\sqrt{d/H}}\,B^2\,\Bigl\lVert W^{Q,h} \Bigr\rVert\,\Bigl\lVert W^{K,h} \Bigr\rVert + 1\right) \Biggr], \quad
    C_2 = 1 + \Bigl\lVert W_{FFN} \Bigr\rVert.
$$
}
\end{theorem*}

\begin{proof}
Let the input $X \in \mathcal{X}$ consist of $n$ tokens $x_i \in \mathbb{R}^{d}$. We consider a Transformer model $f$ using scaled dot-product attention as defined in Equation~\ref{eq:scaled_dot_product_attention}, and formulated as:
\begin{align*}
    \text{AH}(X) &= \mathrm{softmax}\left(\frac{(XW^Q)(XW^K)^\top}{\sqrt{D/H}}\right)(XW^V) \\
    &= P X W^V = h(X) W^V,
\end{align*}
where $W^Q, W^K, W^V$ are learnable projection matrices. The attention matrix $P$ is computed from the softmax of pairwise scores:
\begin{align*}
    f(X) = P X = \mathrm{softmax}(X A^\top X^\top) X,
\end{align*}
with
\begin{align*}
    A = \frac{W^K W^{Q^\top}}{\sqrt{d/H}} \in \mathbb{R}^{d \times d}.
\end{align*}
Each row of the output $f(X)$ can be expressed as:
\begin{align*}
    f(X) = 
    \begin{bmatrix}
        h_1(X)^\top \\
        \vdots \\
        h_n(X)^\top
    \end{bmatrix} \in \mathbb{R}^{n \times d}, \quad \text{with} \quad h_i(X) = \sum_{j=1}^{n} P_{ij} x_j,
\end{align*}
where $P_i^\top = \mathrm{softmax}(X A x_i)$.

To analyze the Jacobian of $h$, we derive its partial derivatives:
\begin{align*}
    J_{ij} = X^\top P^{(i)} E_{ji} X A^\top + \delta_{ij} \big(X^\top P^{(i)} X A\big) + P_{ij} I_d,
\end{align*}
where:
\begin{itemize}
    \item $P^{(i)} = \mathrm{diag}(P_{i:}) - P_{i:}^\top P_{i:}$ is the Jacobian of the softmax,
    \item $E_{ji}$ is an $n \times n$ matrix with a $1$ at position $(j,i)$ and zeros elsewhere.
\end{itemize}

From this, we have:
\begin{align}
\text{If } i \ne j:\quad & J_{ij} = X^\top P^{(i)} E_{ji} X A^\top + P_{ij} I_d, \\
\text{If } i = j:\quad & J_{ii} = X^\top P^{(i)} E_{ii} X A^\top + X^\top P^{(i)} X A + P_{ii} I_d. \label{eq:jacobian_diag}
\end{align}

Assuming the input space is bounded, i.e., $\|x_i\|_2 \le B$ for all $i$, we get:
\begin{align*}
    \|X\|_F^2 = \sum_i \|x_i\|_2^2 \le n B^2 \quad \Rightarrow \quad \|X\| \le \|X\|_F \le \sqrt{n} B.
\end{align*}

Since $P_{i:}$ is a probability distribution and $\sigma_{max}(diag(p)) \leq 1$, we have $\|P^{(i)}\| \le 2$.

\paragraph{Case 1: $i \ne j$}
\begin{align*}
    \|J_{ij}\| &\le \|X^\top P^{(i)} E_{ji} X A^\top\| + \|P_{ij} I_d\| \\
    &\le \|X\|^2 \|P^{(i)}\| \|A\| + 1 \le 2 n B^2 \|A\| + 1.
\end{align*}

\paragraph{Case 2: $i = j$}
\begin{align*}
    \|J_{ii}\| &\le \|X^\top P^{(i)} E_{ii} X A^\top\| + \|X^\top P^{(i)} X A\| + \|P_{ii} I_d\| \\
    &\le 2 n B^2 \|A\| + 2 n B^2 \|A\| + 1 = 4 n B^2 \|A\| + 1.
\end{align*}

Thus, the Jacobian of $h$ is bounded:
\begin{align*}
    \|J_{ij}\|_{op} \le 
    \begin{cases}
        2 n B^2 \|A\| + 1, & \text{if } i \ne j, \\
        4 n B^2 \|A\| + 1, & \text{if } i = j.
    \end{cases}
\end{align*}

Therefore, the function $h$ is bounded with constant:
\begin{align*}
    \mathcal{L}_h \le 4 n B^2 \|A\| + 1.
\end{align*}

\paragraph{Attention Head Bound.} Including the value projection:
\begin{align*}
    \mathcal{L}_{\text{head}} \le \|W^{V,h}\| \left[ 4 \frac{n}{\sqrt{d/H}} B^2 \|W^{Q,h}\| \|W^{K,h}\| + 1 \right].
\end{align*}

\paragraph{Multi-Head Attention Bound.} Since $f$ is represented by $H$ attention head, their concatenated output as explained in Equation \ref{eq:multi_head_attention} satisfies:
\begin{align*}
    \mathcal{L}_{\text{MH}} &\le \|W_O\| \sqrt{H} \max_h \mathcal{L}_{\text{head}} \\
    &\le \|W_O\| \sqrt{H} \max_h \left\{ \|W^{V,h}\| \left[ 4 \frac{n}{\sqrt{d/H}} B^2 \|W^{Q,h}\| \|W^{K,h}\| + 1 \right] \right\}.
\end{align*}

\paragraph{Full Transformer Block.} Incorporating FFN and layer norm (with $\gamma = \beta = 1$):
\begin{align*}
    \mathcal{L}_f &\le L_{LN}^2 (1 + \mathcal{L}_{\text{MH}})(1 + \|W_{\text{FFN}}\|) \\
    &\le \left(\frac{d}{d - 1}\right)^2 (1 + \mathcal{L}_{\text{MH}})(1 + \|W_{\text{FFN}}\|) \\
    &\le \left(\frac{d}{d - 1}\right)^2 C_1 C_2,
\end{align*}
where:
\begin{align*}
    C_1 &= 1 + \|W_O\| \sqrt{H} \max_h \left\{ \|W^{V,h}\| \left[ 4 \frac{n}{\sqrt{d/H}} B^2 \|W^{Q,h}\| \|W^{K,h}\| + 1 \right] \right\}, \\
    C_2 &= 1 + \|W_{\text{FFN}}\|.
\end{align*}

\subsection*{Impact of Pooling Strategies}

Given a final representation $z = g(f(X))$ using pooling function $g$, we evaluate its effect on the bound.

\paragraph{Average pooling.} We recall that this pooling method can be written as a linear layer with weights $W_{\text{Avg}}$, therefore:
\[
W_{\text{Avg}} = \frac{1}{n} \mathbf{1}_n \quad \Rightarrow \quad \|W_{\text{Avg}}\| = \frac{1}{\sqrt{n}}.
\]
\[
\|f(X) - f(\tilde{X})\| \le \frac{1}{\sqrt{n}} \left(\frac{d}{d - 1}\right)^2 C_1 C_2 \epsilon.
\]

\paragraph{Sum pooling.} Similarly
to Average pooling:
\[
W_{\text{Sum}} = \mathbf{1}_n \quad \Rightarrow \quad \|W_{\text{Sum}}\| = \sqrt{n}.
\]
\[
\|f(X) - f(\tilde{X})\| \le \sqrt{n} \left(\frac{d}{d - 1}\right)^2 C_1 C_2 \epsilon.
\]

\paragraph{Last-Token pooling.} Considering the last token as the output as the pooling operation:
\[
W_{\text{Last}} = [0, \dots, 0, 1]^\top \quad \Rightarrow \quad \|W_{\text{Last}}\| = 1.
\]
\[
\|f(X) - f(\tilde{X})\| \le \left(\frac{d}{d - 1}\right)^2 C_1 C_2 \epsilon.
\]
Note that the same treatment can be applied to CLS or any other chosen token.

\paragraph{Max pooling.} Using norm bounds:
\begin{align*}
    \lVert f(X) - f(\tilde{X})\lVert^2 
    & = \sum_{j=1}^{d} \mid (f(X))_j - (f(\tilde{X}))_j \mid^2 \\
    &\leq \sum_{j=1}^{d} \max_i \mid X_{i,j} - \tilde{X}_{i,j}\mid^2 \\
    &= \lVert X - \tilde{X}\lVert_F^2
\end{align*}

For spectral norm:
\begin{align*}
    \lVert f(X) - f(\tilde{X})\lVert 
    &\leq \lVert X - \tilde{X}\lVert_F \\
    &\leq \sqrt{\operatorname{rank}(X - \tilde{X})} \lVert X - \tilde{X}\lVert \\
    &\leq \sqrt{\min(n, d)} \lVert X - \tilde{X}\lVert
\end{align*}

From those two results, we have:
\begin{align*}
    \lVert f(X) - f(\tilde{X})\lVert & \leq \sqrt{\min(n, d)} \lVert X - \tilde{X}\lVert \\
    & \leq \sqrt{\min(n, d)} \big(\frac{d}{d-1}\big)^2 C_1 C_2 \times \epsilon
\end{align*}

Applying the Markov inequality concludes the proof.

\end{proof}

\section{Proof of Lemma \ref{lemma:l2_attention}}\label{app:proof_l2_expressivity}
\begin{lemma*}
Let $f \colon \mathcal{X} \subseteq \mathbb{R}^{n \times d} \rightarrow \mathcal{Y} \subseteq \mathbb{R}^d$ be a L2-MHA-based TBM~\cite{kim2021lipschitz}. In respect to Definition~\ref{def:expressivity}, the following holds:
\begin{itemize}[leftmargin=*]
    \item If $f$ employs Average pooling, then $f$ is $(\epsilon,\sigma,\gamma)$-expressive with 
    $
    \gamma = \frac{\epsilon}{\sigma \sqrt{n}} \left(\frac{d}{d-1}\right)^2 C_1 C_2
    $
    \item If $f$ employs Sum pooling, then $f$ is $(\epsilon,\sigma,\gamma)$-expressive with 
    $
    \gamma = \frac{\sqrt{n}\,\epsilon}{\sigma} \left(\frac{d}{d-1}\right)^2 C_1 C_2
    $
    \item If $f$ employs Last-token pooling, then $f$ is $(\epsilon,\sigma,\gamma)$-expressive with 
    $
    \gamma = \frac{\epsilon}{\sigma} \left(\frac{d}{d-1}\right)^2 C_1 C_2
    $
    \item If $f$ employs Max pooling, then $f$ is $(\epsilon,\sigma,\gamma)$-expressive with 
    $
    \gamma = \frac{\epsilon \sqrt{\min(n,d)}}{\sigma} \left(\frac{d}{d-1}\right)^2 C_1 C_2,
    $
\end{itemize}
where
{\small
$$
    C_1 = 1 + \frac{\sqrt{n}}{\sqrt{d/H}} \left( 4W_O \left(\frac{n}{e}\right) + 1 \right) \left( \sqrt{\sum_h \lVert W^{Q,h} \rVert^2 \lVert W^{V,h} \rVert^2} \right) \lVert W^O \rVert, \quad
    C_2 = 1 + \lVert W_{FFN} \rVert.
$$
}
\end{lemma*}

\begin{proof}
Let the input $X \in \mathcal{X}$ be composed of $n$ tokens $x_i \in \mathbb{R}^d$. In this proof, we consider the Transformer-based model $f$ built using the L2 Multi-Head Attention (L2-MHA) mechanism, where the attention weights are computed as:
\begin{align*}
    P_{ij} \propto \exp \left( - \frac{\| \mathbf{x}_i W^Q - \mathbf{x}_j W^K \|^2}{\sqrt{d/H}} \right),
\end{align*}
with $W^Q, W^K$ being learnable projections.

From Theorem 3.2 of \cite{kim2021lipschitz}, the L2-MHA operator is bounded by:
\begin{align*}
    \text{Lip}_2(F) \leq \frac{\sqrt{n}}{\sqrt{d/H}} \left( 4W_0 \left(\frac{n}{e} \right) + 1 \right) \left( \sqrt{\sum_h \|W^{Q,h}\|^2 \|W^{V,h}\|^2} \right) \|W^O\|,
\end{align*}
where $W_0(\cdot)$ denotes the Lambert W-function.

Following the Transformer architecture defined in Section~\ref{sec:preliminaries}, we account for the additional effects of LayerNorm (LN) and the Feed-Forward Network (FFN). As in previous derivations, we obtain:
\begin{align*}
    \mathcal{L}_{f} \leq \mathcal{L}_{\text{LN}}^2 (1 + \mathcal{L}_{\text{MHA}})(1 + \mathcal{L}_{\text{FFN}}),
\end{align*}
where we now substitute the bound for L2-MHA:
\begin{align*}
    \mathcal{L}_f &\leq \left(\frac{d}{d-1}\right)^2 \left(1 + \frac{\sqrt{n}}{\sqrt{d/H}} \left( 4W_0 \left(\frac{n}{e} \right) + 1 \right) \left( \sqrt{\sum_h \|W^{Q,h}\|^2 \|W^{V,h}\|^2} \right) \|W^O\| \right)(1 + \|W_{\text{FFN}}\|) \\
    &\leq \left(\frac{d}{d-1}\right)^2 C_1 C_2,
\end{align*}
with constants defined as:
\begin{align*}
    C_1 &= 1 + \frac{\sqrt{n}}{\sqrt{d/H}} \left( 4W_0 \left(\frac{n}{e} \right) + 1 \right) \left( \sqrt{\sum_h \|W^{Q,h}\|^2 \|W^{V,h}\|^2} \right) \|W^O\|, \\
    C_2 &= 1 + \|W_{\text{FFN}}\|.
\end{align*}

Following the same steps as in the Theorem~\ref{theo:bound_transformer} proof, we get the following:

\textbf{For Average pooling:}
\begin{align*}
    \mathcal{L}_{\text{Avg}} \leq \frac{1}{\sqrt{n}} \times \left(\frac{d}{d-1}\right)^2 C_1 C_2.
\end{align*}

\textbf{For Sum pooling:}
\begin{align*}
    \mathcal{L}_{\text{Sum}} \leq \sqrt{n} \times \left(\frac{d}{d-1}\right)^2 C_1 C_2.
\end{align*}

\textbf{For Last-token pooling:}
\begin{align*}
    \mathcal{L}_{\text{Last}} \leq \left(\frac{d}{d-1}\right)^2 C_1 C_2.
\end{align*}

\textbf{For Max pooling:}
\begin{align*}
    \mathcal{L}_{\text{Max}} \leq \sqrt{\min(n,d)} \left(\frac{d}{d-1}\right)^2 C_1 C_2.
\end{align*}

\end{proof}

\section{Proof of Lemma \ref{lemma:lipsformer}}\label{app:proof_lipsformer}
\begin{lemma*}
Let $f \colon \mathcal{X}\rightarrow\mathcal{Y}$ to be a function based on the LipsFormer~\citep{qi2023lipsformer} framework, with corresponding hyper-parameters $\nabla, \nu, \tau > 0$ and window size $w$. In respect to Definition~\ref{def:expressivity}, we have:
\begin{itemize}[leftmargin=*]
    \item If $f$ is based on Average pooling, then $f$ is $(\epsilon, \sigma, \gamma)$-expressive with $\gamma= \frac{\epsilon}{\sigma \times \sqrt{n}} \times \big(\frac{d}{d-1}\big)^2 C_1 C_2$ 
    \item If $f$ is based on Sum pooling, then $f$ is $(\epsilon, \sigma, \gamma)$-expressive with $\gamma= \frac{\sqrt{n} \times \epsilon}{\sigma} \times \big(\frac{d}{d-1}\big)^2 C_1 C_2$ 
    \item If $f$ is based on Last-token pooling, then $f$ is $(\epsilon, \sigma, \gamma)$-expressive with $\gamma= \frac{\epsilon}{\sigma} \times \big(\frac{d}{d-1}\big)^2 C_1 C_2$ 
    \item If $f$ is based on Max pooling, then $f$ is $(\epsilon, \sigma, \gamma)$-expressive with $\gamma = \frac{\epsilon \sqrt{\min(n,d)}}{\sigma} \times \big(\frac{d}{d-1}\big)^2 C_1 C_2, $ 
\end{itemize}
where
\begin{align*}
    & C_1 = 1 +
       \lVert W_O\lVert\sqrt{H}\max_{h}
         \Bigl\{2w(w-1)\nu\tau\nabla^{-\frac12}\lVert W^K_h\lVert
           + 2(w-1)\nu\tau\nabla^{-\tfrac12}\lVert W^Q_h \lVert + 2w\nu\nabla^{-\tfrac12} \lVert W^V_h\lVert \Bigr\}, \\
    & C_2 = \big(1 + \lVert W_{FFN} \lVert\big)
\end{align*}
\end{lemma*}

\begin{proof}
Before delving in the specific analysis of the Lipsformodel model, we start the proof by providing some preliminary elements about the Swin Transformer~\citep{liu2021swin} which is different from the original Transformer-based Model defined in Section~\ref{sec:preliminaries}. 

A Swin Transformer block's input is similar to the one from a TBM, specifically, the input $X \in \mathbb{R}^{n \times d}$, can be viewed as $n$ tokens (or patches) each of dimension $d$. Rather than applying global self-attention to all $n$ tokens, the model partitions $X$ into small “local windows” of size $w$, thereby reducing complexity. 
Between successive Swin stages, there is a “patch merging” step, which consists of a linear downsampling that reduces the number of tokens while increasing their dimension. 

Let $\mathcal{W}$ denote the total number of windows, $X_\ell \in \mathbb{R}^{w \times d}$ be the slice of input corresponding to window $\ell$ and $W^Q, W^K, W^V$ are the query/key/value projection matrices, within each block, a window-based self-attention is computed as follows:
\begin{align*}
    \mathrm{LocalAttn}(X) =\bigoplus_{\ell=1}^{\mathcal{W}}\operatorname{softmax}\Bigl(\tfrac{(X_\ell W^Q)(X_\ell W^K)^\top}{\sqrt{d/H}} \Bigr)(X_\ell W^V)
\end{align*}

In our analysis, we rather focus on the LipsFormer~\citep{qi2023lipsformer} model, which is an adaptation of the previous equation. Specifically, the model modifies the Swin Transformer architecture by replacing the standard dot‐product attention with a \textit{scaled cosine similarity attention} mechanism. Given an input $X \in \mathbb{R}^{w \times d}$ (\eg, the tokens in a window), 
define the following:
\begin{align*}
   q_i = \frac{X_i\,W^Q}{\lVert X_i\,W^Q\rVert_2 + \nabla},
   \quad
   k_j = \frac{X_j\,W^K}{\lVert X_j\,W^K\rVert_2 + \nabla},
   \quad
   v_j = \frac{X_j\,W^V}{\lVert X_j\,W^V\rVert_2 + \nabla},
\end{align*}
where $X_i \in \mathbb{R}^d$ is the $i$-th token row of $X$, and $\nabla>0$ is a small constant to avoid division by $0$. Then accordingly write the usual attention matrices as:
\begin{align*}
   Q = \bigl[q_1^\top,\dots,q_w^\top \bigr],
   \quad
   K = \bigl[\,k_1^\top,\dots;\;k_w^\top \bigr],
   \quad
   V = \bigl[\,v_1^\top,\dots,v_w^\top \bigr],
\end{align*}
which are all in $\mathbb{R}^{w\times d}$. The scaled cosine similarity attention (SCSA) can be then formulated as:

\begin{align*}
    \mathrm{SCSA}(X) = \nu \operatorname{softmax}\Bigl[\tau\bigl(QK^\top\bigr)\Bigr]V,
\end{align*}

where $\tau,\nu>0$ are scalars that scale the argument of the softmax and the final output. As can be seen in the formulation, the key difference from standard attention is that each query/key vector is row‐normalized to unit $\ell_2$ length (up to $\nabla$).

Similar to a TBM, $H$ attention heads are used with each one using separate projection matrices $W^Q_h, W^K_h, W^V_h$, forming $Q,K,V$ for each head, then concatenates the outputs and multiplies by $W_O$ :
\begin{align*}
  \mathrm{MultiHead\,SCSA}(X) = \bigl[\mathrm{SCSA}_1(X),\mathrm{SCSA}_2(X),\dots,\mathrm{SCSA}_H(X)\bigr].
\end{align*}

Let's now derive the upper-bounds of this model. Similar to the previous proofs, let's consider that our model $f$ is built using the scaled cosine similarity attention, with $H$ attention heads and one layer. From Appendix H.2 in the original paper \cite{qi2023lipsformer}, we have the following for a single head of attention:
\begin{align*}
   \mathcal{L}_{\mathrm{SCSA}} \leq
   2n(n-1)\nu\tau\nabla^{-\tfrac12}\lVert W^K\lVert + 2(n-1)\nu\tau\nabla^{-\tfrac12}\lVert W^Q\lVert + 
   2n\nu\nabla^{-\tfrac12}\,\lVert W^V\lVert,
\end{align*}

with $n$ being the number of tokens within a local window, and $\nu,\tau>0$ and $\nabla>0$ the chosen hyper-parameters of in SCSA.

When considering the multi-head attention framework, we get:

\begin{align*}
  \mathcal{L}_{\mathrm{MH-SCSA}} \leq
  \lVert W_O\lVert\sqrt{H}
  \max_{1\le h\le H}\Bigl[\mathcal{L}_{\mathrm{SCSA_h}}].
\end{align*}

Similar to previous proofs and since we consider the same the Feed-Forward and Layer Normalization aspect, we directly get the following result:

\begin{align*}
   \mathcal{L}_f \le \Bigl(\tfrac{d}{d-1}\Bigr)^2
   \Bigl[
       1 +
       & \lVert W_O\lVert\sqrt{H}\max_{h=1,\dots,H}
         \Bigl\{
           2w(w-1)\nu\tau\nabla^{-\frac12}\lVert W^K_h\lVert \\
           & + 2(w-1)\nu\tau\nabla^{-\tfrac12}\lVert W^Q_h \lVert + 2w\nu\nabla^{-\tfrac12} \lVert W^V_h\lVert
         \Bigr\} \Bigr]
   \Bigl[1 + \lVert W_{\mathrm{FFN}}\lVert \Bigr] 
\end{align*}

which could be written as:
\begin{align*}
    \mathcal{L}_f \leq \left(\frac{d}{d-1}\right)^2 C_1 C_2,
\end{align*}
\small
\begin{align*}
    \text{with }C_1 &= 1 +
       \lVert W_O\lVert\sqrt{H}\max_{h}
         \Bigl\{
           2w(w-1)\nu\tau\nabla^{-\frac12}\lVert W^K_h\lVert
           + 2(w-1)\nu\tau\nabla^{-\tfrac12}\lVert W^Q_h \lVert + 2w\nu\nabla^{-\tfrac12} \lVert W^V_h\lVert \Bigr\} \\
    C_2 &= (1 + \lVert W_{FFN}\lVert),
\end{align*}
\normalsize

For the pooling operation, similar analogy that was used in the case of dot-product attention can be used, and we find therefore the final results:

\textbf{For Average pooling:}
\begin{align*}
    \mathcal{L}_{\text{avg}} \leq \frac{1}{\sqrt{n}} \times \left(\frac{d}{d-1}\right)^2 C_1 C_2.
\end{align*}

\textbf{For Sum pooling:}
\begin{align*}
    \mathcal{L}_{\text{sum}} \leq \sqrt{n} \times \left(\frac{d}{d-1}\right)^2 C_1 C_2.
\end{align*}

\textbf{For Last-token pooling:}
\begin{align*}
    \mathcal{L}_{\text{last}} \leq \left(\frac{d}{d-1}\right)^2 C_1 C_2.
\end{align*}

\textbf{For Max pooling:}
\begin{align*}
    \mathcal{L}_{\text{last}} \leq \sqrt{\min(n,d)} \left(\frac{d}{d-1}\right)^2 C_1 C_2.
\end{align*}

By applying the Markov inequality we get the desired result.

\end{proof}

\section{Experimental Details}\label{app:experimental_details}
We start by noting that the necessary code to reproduce the results is provided in the supplementary materials and shall be made public upon publication. In what follows, we provide experimental details and hyper-parameters choices.

\subsection{Computer Vision.} All computer vision experiments used a frozen Transformer backbone (ViT-base~\citep{dosovitskiy2021an}, ViT-small, or LipsFormer~\citep{qi2023lipsformer} built on Swin Transformer~\citep{liu2021swin} architecture), with a randomly initialized heads fine‐tuned on each task. We optimized with Adam~\citep{kingma2014adam} at a learning rate of $1\times10^{-3}$. All the tasks were trained for $10$ epochs; all of which yielded stable convergence. Images were resized and either padded or center‐cropped to the model's input resolution. No data augmentation was applied during training.

We evaluated tasks that require both local and global context. For classification (CIFAR 10/100~\citep{krizhevsky2009learning}, ImageNet-100~\citep{russakovsky2015imagenet}, MiniPlaces~\citep{zhou2017places}, Caltech-UCSD Birds (CUB)~\citep{wah2011caltech}), the head's output dimension matched the number of classes and was trained with cross‐entropy loss. For inpainting (CelebA~\citep{liu2018large}, Oxford Flowers~\citep{nilsback2008automated}, Oxford-IIIT Pet~\citep{parkhi2012cats}), the head predicted pixel values in masked regions and was trained with mean squared error. For segmentation (Pascal-VOC~\citep{everingham2015pascal}), a linear per‐pixel classifier trained with cross‐entropy loss was used and report mean pixel accuracy. All the experiments were run using a single NVIDIA L4 GPU and took 25 GPU hours to obtain all results.

\subsection{NLP}
In all experiments involving LLMs, the Transformer backbone was kept frozen, and only a randomly initialized linear head and parametric pooling parameters were fine-tuned. We used the provided pretrained (non instruction-tuned) checkpoints for Llama3 \cite{grattafiori2024llama}, Mistral 7B \cite{jiang2023mistral7b}, Qwen 2.5 \cite{yang2024qwen2} and BERT \cite{devlin2019bert}, and we pre-trained GPT-2 and L2-GPT-2 (see details below). Optimization was performed using the Adam~\citep{kingma2014adam} optimizer with a learning rate of $1 \times 10^{-3}$. Ten epochs of fine-tuning consistently yielded stable convergence across tasks. Each experiment was repeated five times with fixed random seeds to improve robustness and ensure reproducibility.

All pooling methods, including those with learnable components, were trained using the same configuration. Experiments were conducted on an instance with 2$\times$NVIDIA L4 GPUs using PyTorch~\cite{NEURIPS2019_bdbca288} with the Distributed Data Parallel framework and a batch size of 32 per GPU. Running all experiments took 1832 GPU hours on L4 GPUs.

Each dataset's training split was used for fine-tuning. Hyperparameters were selected based on validation performance (where available), and final results were reported on the held-out test set. To maintain consistent input dimensions, all sequences were padded or truncated to a predefined maximum length. Tokenization was done using each model's default tokenizer, and [PAD] tokens were used for padding. 

For classification tasks, we used a linear head with output dimensionality matching the number of classes, trained by minimizing the cross-entropy loss. In semantic similarity tasks, the pooled embeddings were linearly projected without changing dimensionality. Cosine similarity between embedding pairs was used as the main metric. For STS Benchmark (STSB)~\citep{cer2017semeval}, similarity scores (rescaled from $[1, 5]$ to $[0, 1]$) were predicted by minimizing mean squared error. In the HellaSwag~\citep{zellers2019hellaswag} task, the goal was to match a given context to its correct ending. The context and four candidates were encoded with the same LLM and pooling method, projected linearly, and compared via cosine similarity. Cross-entropy loss was applied over the similarity scores, encouraging correct pairings. For next-token prediction, we used the TinyStories~\citep{eldan2023tinystories} corpus under a standard autoregressive setting. The training set comprised 4000 batches randomly sampled from the corpus. A randomly initialized language modeling head was trained to predict the next token based on preceding context. We used a held out test-set and randomly sampled tokens to predict.

\paragraph{GPT-2 Pretraining.}
To obtain a checkpoint for L2-GPT-2 (a GPT-2 model with L2-MHA), we followed the standard pretraining procedure described in \cite{radford2019language}, modifying the attention mechanism to use the L2 kernel with tied query and key matrices. This change slightly reduced the number of parameters (from 123M to 116M). The model was pretrained on the OpenWebText~\citep{Gokaslan2019OpenWeb} corpus for $60\,000$ iterations using a batch size of $12$, block size of $1024$, and $40$ gradient accumulation steps. Training was conducted on 8$\times$NVIDIA L4 GPUs and took about 960 GPU hours.

For a fair comparison, we also pretrained a baseline GPT-2 checkpoint using identical settings, differing only in the use of the standard dot-product attention mechanism.

\subsection{Time Series}
For time series analysis, we used MOMENT~\citep{goswami2024moment}, a family of Transformer-based foundation models for time series. We evaluate three pretrained checkpoints (AutonLab/MOMENT-1-\{small, base, large\}) trained on the Time Series Pile dataset~\citep{goswami2024moment}.

During training, we kept the model backbone frozen and fine-tuned only the linear head and pooling operator (Weighted Average and Attention-based) for classification, forecasting, and imputation tasks. Optimization was performed using Adam~\citep{kingma2014adam} with a learning rate of $1 \times 10^{-3}$ with a batch size of $64$ on a single NVIDIA L4 GPU.

For classification, we run optimization for $20$ epochs across six datasets, six pooling methods, three model sizes, and five random seeds, yielding $540$ experiment trials and a total of approximately $22$ GPU hours. For forecasting, we trained the prediction head for $10$ epochs using a forecasting horizon of $96$ future time steps across seven datasets, six pooling methods, three model sizes, and five random seeds, resulting in $630$ experiment trials and approximately $90$ GPU hours. For imputation, we trained the prediction head for $10$ epochs across seven datasets, six pooling methods, three model sizes, and five random seeds, adding another $630$ experiment trials and approximately $96$ GPU hours.

\section{Additional Results}\label{app:additional_results}
\textbf{Computer Vision.}\label{app:additional_results_cv} In addition to ViT-small model, we extend our analysis to evaluate whether the theoretical insights hold in larger architectures with more attention heads and blocks. Table~\ref{tab:vit_results_appendix} reports the results for the same pooling benchmarks using ViT-base as the backbone.

Consistent with previous findings, Weighted Average pooling maintains strong performance across tasks, reflecting its ability to adapt to context and produce stable, generalizable representations. Similar patterns emerge for Attention-based pooling, which performs best in the inpainting task but does not outperform Weighted Average pooling in other settings. This suggests that Attention-based pooling may require more data and computational resources to reach its full potential.

Among flat pooling strategies, CLS pooling continues to yield the best results for classification tasks. Notably, the performance gap between CLS and Average pooling narrows, indicating that larger models can offset suboptimal pooling through increased representational capacity.

\renewcommand{\arraystretch}{1.4}
\setlength{\tabcolsep}{4pt}
\begin{table}[h]
\footnotesize
\caption{Mean and standard deviation of test metrics for computer vision tasks. Best performance per dataset and model is indicated in \textbf{bold}. Best performance among non-learnable pooling methods is \underline{underlined}.}
\label{tab:vit_results_appendix}
\resizebox{\columnwidth}{!}{%
\begin{tabular}{llcccccccccc}
\hline
\multirow{2}{*}{\textbf{Model}}                                & \textbf{}        & \multicolumn{5}{c}{\textbf{Classification} (Accuracy)}                                                  & \multicolumn{3}{c}{\textbf{Inpainting} (MSE)}                           & \multicolumn{2}{c}{\textbf{Segmentation} (Accuracy)} \\ 
\cmidrule(lr){3-7} \cmidrule(lr){8-10} \cmidrule(lr){11-12}
                                                               & \textbf{Pooling} & CIFAR-10         & CIFAR-100        & ImageNet-100     & CUB-200-2011     & MiniPlaces       & CelebA               & OxfordFlower-102     & Oxford-IIIT Pet      & PascalVOC-Cls       & PascalVOC-Det       \\ \hline
\multirow{6}{*}{\rotatebox{90}{ViT-base}}   & Last (CLS)        & $\underline{92.34 \pm 0.05}$ & $79.57 \pm 0.13$ & $\underline{\mathbf{90.67 \pm 0.17}}$ & $\underline{\mathbf{78.87 \pm 0.53}}$ & $58.86 \pm 0.13$ & $0.240 \pm 0.002$    & $\underline{0.314 \pm 0.003}$    & $\underline{\mathbf{0.256 \pm 0.003}}$    & $\underline{\mathbf{72.49 \pm 0.68}}$    & $\underline{\mathbf{28.01 \pm 0.94}}$    \\
                                                               & Avg          & $92.25 \pm 0.34$ & $\underline{79.67 \pm 0.40}$ & $90.50 \pm 0.02$ & $73.36 \pm 0.50$ & $\underline{\mathbf{59.81 \pm 0.33}}$ & $\underline{0.237 \pm 0.001}$    & $0.319 \pm 0.005$    & $0.266 \pm 0.003$    & $72.19 \pm 0.73$    & $26.59 \pm 1.45$    \\
                                                               & Sum              & $91.75 \pm 0.59$ & $78.94 \pm 0.10$ & $86.98 \pm 0.04$ & $72.19 \pm 0.07$ & $59.07 \pm 0.44$ & $0.312 \pm 0.004$    & $0.678 \pm 0.091$    & $0.831 \pm 0.083$    & $71.73 \pm 0.91$    & $25.09 \pm 0.17$    \\
                                                               & Max              & $91.39 \pm 0.88$ & $74.80 \pm 0.64$ & $90.18 \pm 0.15$ & $59.51 \pm 0.89$ & $56.61 \pm 0.68$ & $0.255 \pm 0.001$    & $0.401 \pm 0.045$    & $0.281 \pm 0.009$    & $70.21 \pm 1.29$    & $22.38 \pm 0.67$    \\
                                                               \cmidrule(lr){3-12}
                                                               & W-Avg& $\mathbf{92.55 \pm 0.17}$ & $\mathbf{80.62 \pm 0.13}$ & $90.48 \pm 0.07$ & $74.81 \pm 0.25$ & $59.80 \pm 0.12$ & $0.236 \pm 0.001$    & $0.328 \pm 0.002$    & $0.270 \pm 0.002$    & $71.82 \pm 0.71$    & $26.62 \pm 0.20$    \\
                                                               & Attn  & $91.81 \pm 0.22$ & $76.84 \pm 0.66$ & $90.39 \pm 0.21$ & $68.62 \pm 1.89$ & $57.62 \pm 0.27$ & $\mathbf{0.162 \pm 0.003}$    & $\mathbf{0.303 \pm 0.004}$    & $0.323 \pm 0.048$    & $71.89 \pm 0.23$    & $25.14 \pm 1.04$    \\ \cline{2-12} 

\hline
\end{tabular}%
}
\end{table}

\textbf{NLP.}\label{app:additional_results_nlp}
Beyond the theoretical expressivity bounds shown in Figure~\ref{fig:power_expressivity_analysis}, we further examine how these bounds manifest empirically in NLP settings. To this end, we construct a set of sentence variants: a base sentence (the original), two semantically similar versions created by replacing adjectives, and two dissimilar versions using unrelated words. Figure~\ref{fig:additional_expressivity_analysis} shows the resulting changes in pooled representations across different pooling strategies.

\begin{figure}[htbp]
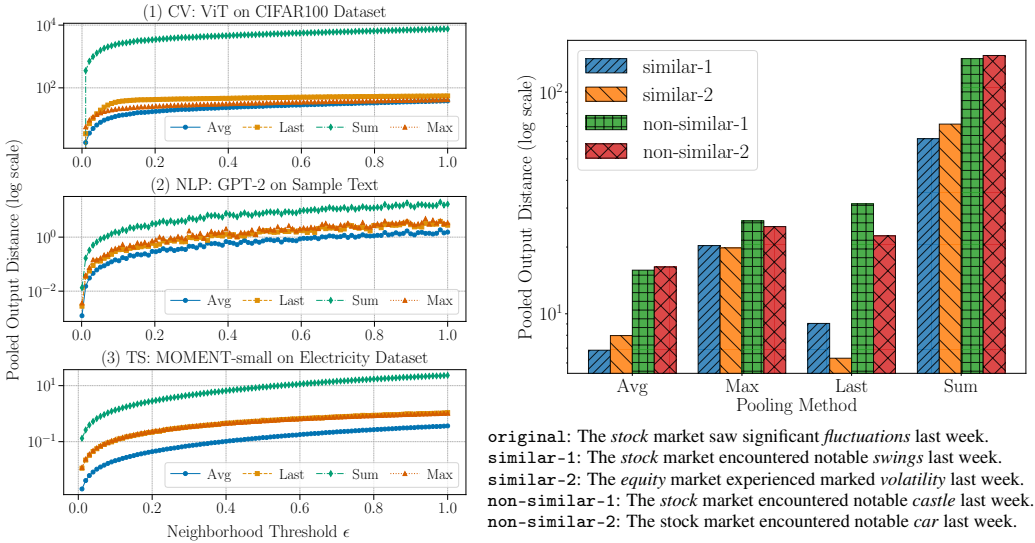

  \centering

  \begin{minipage}[t]{0.44\textwidth}
    \centering
    \includegraphics[width=\textwidth,trim={0.2cm 0.2cm 0.2cm 0.2cm},clip]{Figures/perturbation_combined.pdf}
    \label{fig:perturbation_appendix}
  \end{minipage}
  \hfill
  \begin{minipage}[t]{0.54\textwidth}
    \vspace{-19em}
    \centering
    \includegraphics[width=0.9\textwidth,trim={0.2cm 0.2cm 0.2cm 0.2cm},clip]{Figures/gpt-manual-perturbation-four-examples.pdf}
    
    \vspace{0.2em}
    {\raggedright\scriptsize
    \texttt{original}: The \textit{stock} market saw significant \textit{fluctuations} last week.\\
    \texttt{similar-1}: The \textit{stock} market encountered notable \textit{swings} last week.\\
    \texttt{similar-2}: The \textit{equity} market experienced marked \textit{volatility} last week.\\
    \texttt{non-similar-1}: The \textit{stock} market encountered notable \textit{castle} last week. \\
    \texttt{non-similar-2}: The stock market encountered notable \textit{car} last week. \\
    }
    \label{fig:sentence_examples_2}
  \end{minipage}

  \caption{Empirical analysis of the expressivity power across modalities and pooling strategies. Left: Mean pooled‐output distance $\gamma$ versus perturbation $\epsilon$ across modalities highlighting the behavior of various methods. Right: pooled‐output distances for similar and dissimilar inputs, exemplifying expressivity of different strategies.}
  \label{fig:additional_expressivity_analysis}
\end{figure}

  

\textbf{Time Series.}\label{app:additional_results_ts} In addition to the results presented in Table~\ref{tab:ts_results}, we report extended empirical evaluations of pooling operators on a broader range of time-series datasets. Tables~\ref{tab:moment_classification_results},~\ref{tab:moment_forecasting_results}, and~\ref{tab:moment_imputation_results} provide results for classification, forecasting, and imputation tasks, respectively. Overall, the findings on these additional datasets are consistent with the trends discussed in Section~\ref{sec:experiment_effect_on_downstream_performance}, further supporting our analysis.

\begin{table}[h]
\tiny
\centering
\caption{Mean and standard deviation of test accuracy for time series classification tasks. Best performance per dataset and model in \textbf{bold}. Best performance among non-learnable pooling methods is \underline{underlined}.}
\label{tab:moment_classification_results}
\begin{tabular}{rrcccccccccc}
\toprule
\makecell{Model} & \makecell{Pooling\\Operator} & \makecell{ECG200} 
  & \makecell{Electric\\Devices} 
  & \makecell{FordA} & \makecell{FordB} 
  & \makecell{SmallKitchen\\Appliances} 
  & \makecell{SwedishLeaf} \\
\midrule
\multirow{6}{*}{\rotatebox{90}{MOMENT-small}} & Last & $72.29 \pm 0.59$ & $60.45 \pm 0.48$ & $76.39 \pm 0.15$ & $62.07 \pm 0.44$ & $64.06 \pm 0.75$ & $69.25 \pm 0.74$ \\
 & Avg & $65.19 \pm 0.00$ & $61.40 \pm 0.54$ & $88.12 \pm 0.14$ & $67.05 \pm 0.13$ & $62.40 \pm 1.01$ & $55.53 \pm 9.55$ \\
 & Sum & \underline{$\mathbf{80.35 \pm 1.82}$} & $60.62 \pm 2.85$ & \underline{$92.93 \pm 0.43$} & \underline{$78.57 \pm 0.72$} & \underline{$67.13 \pm 2.95$} & \underline{$\mathbf{79.42 \pm 1.33}$} \\
 & Max & $65.19 \pm 0.00$ & \underline{$61.78 \pm 1.24$} & $90.42 \pm 0.43$ & $72.28 \pm 1.00$ & $63.94 \pm 1.76$ & $58.59 \pm 6.53$ \\
 \cmidrule(lr){3-8}
 & W-Avg & $65.19 \pm 0.00$ & $62.54 \pm 0.67$ & $87.62 \pm 0.64$ & $67.36 \pm 0.44$ & $62.40 \pm 1.01$ & $58.81 \pm 8.41$ \\
 & Attn & $78.84 \pm 3.04$ & $\mathbf{62.95 \pm 2.04}$ & $\mathbf{93.60 \pm 0.37}$ & $\mathbf{79.61 \pm 0.86}$ & $\mathbf{67.63 \pm 2.95}$ & $75.93 \pm 2.07$ \\
\cline{1-8}
\multirow{6}{*}{\rotatebox{90}{MOMENT-base}} & Last & $71.01 \pm 1.94$ & $63.05 \pm 0.62$ & $83.46 \pm 0.40$ & $64.60 \pm 0.37$ & $63.25 \pm 0.54$ & $72.03 \pm 1.09$ \\
 & Avg & $65.19 \pm 0.00$ & $63.71 \pm 0.40$ & $89.75 \pm 0.37$ & $70.63 \pm 0.54$ & $63.05 \pm 1.60$ & $62.92 \pm 5.93$ \\
 & Sum & \underline{$\mathbf{83.30 \pm 1.51}$} & \underline{$64.30 \pm 0.98$} & \underline{$92.51 \pm 0.13$} & \underline{$79.02 \pm 0.98$} & \underline{$66.55 \pm 1.48$} & \underline{$\mathbf{84.05 \pm 0.94}$} \\
 & Max & $65.78 \pm 1.32$ & $62.31 \pm 1.14$ & $90.23 \pm 0.80$ & $70.91 \pm 2.14$ & $64.77 \pm 2.12$ & $66.24 \pm 4.44$ \\
 \cmidrule(lr){3-8}
 & W-Avg & $65.19 \pm 0.00$ & $\mathbf{64.48 \pm 0.69}$ & $89.93 \pm 0.36$ & $70.98 \pm 0.62$ & $63.20 \pm 1.74$ & $64.87 \pm 5.16$ \\
 & Attn & $82.57 \pm 1.19$ & $64.15 \pm 1.37$ & $\mathbf{92.74 \pm 0.35}$ & $\mathbf{79.94 \pm 0.59}$ & $\mathbf{66.76 \pm 1.39}$ & $78.66 \pm 2.25$ \\
\cline{1-8}
\multirow{6}{*}{\rotatebox{90}{MOMENT-large}} & Last & $72.67 \pm 0.95$ & $61.10 \pm 0.53$ & $79.61 \pm 0.31$ & $63.63 \pm 0.64$ & \underline{$65.45 \pm 1.71$} & $76.75 \pm 1.47$ \\
 & Avg & $65.19 \pm 0.00$ & $61.72 \pm 0.93$ & $85.98 \pm 0.53$ & $67.56 \pm 1.18$ & $60.42 \pm 0.91$ & $59.39 \pm 6.89$ \\
 & Sum & \underline{$75.97 \pm 2.82$} & \underline{$\mathbf{63.45 \pm 0.76}$} & \underline{$92.85 \pm 0.42$} & \underline{$78.08 \pm 0.48$} & $64.92 \pm 7.21$ & \underline{$\mathbf{82.23 \pm 0.88}$} \\
 & Max & $65.19 \pm 0.00$ & $60.08 \pm 0.54$ & $87.14 \pm 0.39$ & $69.52 \pm 0.74$ & $62.27 \pm 0.64$ & $64.50 \pm 3.75$ \\
 \cmidrule(lr){3-8}
 & W-Avg & $65.19 \pm 0.00$ & $61.48 \pm 0.90$ & $86.07 \pm 0.36$ & $67.65 \pm 1.09$ & $60.54 \pm 0.53$ & $60.84 \pm 6.34$ \\
 & Attn & $\mathbf{78.73 \pm 3.09}$ & $62.83 \pm 1.38$ & $\mathbf{93.02 \pm 0.26}$ & $\mathbf{79.95 \pm 1.02}$ & $\mathbf{65.64 \pm 3.54}$ & $80.66 \pm 1.04$ \\
\cline{1-8}
\bottomrule
\end{tabular}
\end{table}

\begin{table}
\tiny
\caption{Mean and standard deviation of test MSE for time series forecasting tasks. Best performance per dataset and model in \textbf{bold}. Best performance among non-learnable pooling methods is \underline{underlined}.}
\label{tab:moment_forecasting_results}
\resizebox{\columnwidth}{!}{
    \begin{tabular}{rrcccccccccc}
    \toprule
    \makecell{Model} & \makecell{Pooling\\Operator} & \makecell{ETTh1} 
      & \makecell{ETTh2} 
      & \makecell{ETTm1} & \makecell{ETTm2} 
      & \makecell{electricity} &\makecell{traffic}
      & \makecell{weather} \\
    \midrule
    \multirow{6}{*}{\rotatebox{90}{MOMENT-small}} & Last & \underline{$\mathbf{0.082 \pm 0.000}$} & \underline{$\mathbf{0.193 \pm 0.000}$} & \underline{$\mathbf{0.040 \pm 0.000}$} & \underline{$0.107 \pm 0.000$} & \underline{$\mathbf{0.400 \pm 0.001}$} & \underline{$0.273 \pm 0.001$} & $0.002 \pm 0.000$ \\
     & Avg & $0.105 \pm 0.000$ & $0.305 \pm 0.001$ & $0.079 \pm 0.000$ & $0.234 \pm 0.000$ & $0.790 \pm 0.000$ & $1.724 \pm 0.001$ & \underline{$\mathbf{0.002 \pm 0.000}$} \\
     & Sum & $0.103 \pm 0.001$ & $0.300 \pm 0.005$ & $0.053 \pm 0.000$ & $0.141 \pm 0.000$ & $0.826 \pm 0.014$ & $1.422 \pm 0.050$ & $0.005 \pm 0.001$ \\
     & Max & $0.106 \pm 0.001$ & $0.311 \pm 0.002$ & $0.082 \pm 0.001$ & $0.230 \pm 0.001$ & $0.800 \pm 0.001$ & $1.759 \pm 0.002$ & $0.002 \pm 0.000$ \\
     \cmidrule(lr){3-9}
     & W-Avg & $0.105 \pm 0.000$ & $0.287 \pm 0.002$ & $0.058 \pm 0.005$ & $0.184 \pm 0.000$ & $0.539 \pm 0.008$ & $0.954 \pm 0.014$ & $0.002 \pm 0.000$ \\
     & Attn & $0.106 \pm 0.001$ & $0.313 \pm 0.003$ & $0.041 \pm 0.001$ & $\mathbf{0.097 \pm 0.000}$ & $0.475 \pm 0.059$ & $\mathbf{0.258 \pm 0.002}$ & $0.003 \pm 0.000$ \\
    \cline{1-9}
    \multirow{6}{*}{\rotatebox{90}{MOMENT-base}} & Last & \underline{$\mathbf{0.081 \pm 0.000}$} & \underline{$\mathbf{0.193 \pm 0.000}$} & \underline{$\mathbf{0.040 \pm 0.000}$} & \underline{$0.105 \pm 0.000$} & \underline{$\mathbf{0.397 \pm 0.000}$} & \underline{$\mathbf{0.265 \pm 0.000}$} & $0.002 \pm 0.000$ \\
     & Avg & $0.105 \pm 0.000$ & $0.304 \pm 0.002$ & $0.070 \pm 0.000$ & $0.226 \pm 0.000$ & $0.785 \pm 0.001$ & $1.719 \pm 0.001$ & \underline{$\mathbf{0.002 \pm 0.000}$} \\
     & Sum & $0.101 \pm 0.002$ & $0.283 \pm 0.002$ & $0.052 \pm 0.000$ & $0.136 \pm 0.001$ & $0.805 \pm 0.009$ & $1.029 \pm 0.023$ & $0.004 \pm 0.000$ \\
     & Max & $0.106 \pm 0.000$ & $0.309 \pm 0.002$ & $0.075 \pm 0.001$ & $0.225 \pm 0.000$ & $0.796 \pm 0.001$ & $1.747 \pm 0.003$ & $0.002 \pm 0.000$ \\
     \cmidrule(lr){3-9}
     & W-Avg & $0.105 \pm 0.000$ & $0.280 \pm 0.002$ & $0.058 \pm 0.000$ & $0.178 \pm 0.001$ & $0.511 \pm 0.006$ & $0.857 \pm 0.014$ & $0.002 \pm 0.000$ \\
     & Attn & $0.096 \pm 0.009$ & $0.291 \pm 0.012$ & $0.043 \pm 0.000$ & $\mathbf{0.097 \pm 0.001}$ & $0.507 \pm 0.148$ & $0.306 \pm 0.037$ & $0.003 \pm 0.000$ \\
    \cline{1-9}
    \multirow{6}{*}{\rotatebox{90}{MOMENT-large}} & Last & \underline{$\mathbf{0.080 \pm 0.000}$} & \underline{$\mathbf{0.195 \pm 0.000}$} & \underline{$0.039 \pm 0.000$} & \underline{$\mathbf{0.103 \pm 0.000}$} & \underline{$\mathbf{0.379 \pm 0.000}$} & \underline{$\mathbf{0.272 \pm 0.001}$} & $0.002 \pm 0.000$ \\
     & Avg & $0.105 \pm 0.000$ & $0.306 \pm 0.000$ & $0.073 \pm 0.000$ & $0.207 \pm 0.000$ & $0.778 \pm 0.000$ & $1.711 \pm 0.001$ & \underline{$\mathbf{0.002 \pm 0.000}$} \\
     & Sum & $0.101 \pm 0.001$ & $0.269 \pm 0.002$ & $0.049 \pm 0.000$ & $0.126 \pm 0.001$ & $0.688 \pm 0.007$ & $0.782 \pm 0.003$ & $0.003 \pm 0.000$ \\
     & Max & $0.104 \pm 0.001$ & $0.306 \pm 0.002$ & $0.073 \pm 0.000$ & $0.206 \pm 0.001$ & $0.785 \pm 0.001$ & $1.743 \pm 0.000$ & $0.002 \pm 0.000$ \\
     \cmidrule(lr){3-9}
     & W-Avg & $0.105 \pm 0.001$ & $0.283 \pm 0.000$ & $0.053 \pm 0.000$ & $0.170 \pm 0.003$ & $0.534 \pm 0.003$ & $0.951 \pm 0.005$ & $0.002 \pm 0.000$ \\
     & Attn & $0.097 \pm 0.004$ & $0.306 \pm 0.000$ & $\mathbf{0.039 \pm 0.000}$ & $0.106 \pm 0.003$ & $0.684 \pm 0.003$ & $0.423 \pm 0.028$ & $0.003 \pm 0.000$ \\
    \cline{1-9}
    \bottomrule
    \end{tabular}
}
\end{table}

\begin{table}
\tiny
\caption{Mean and standard deviation of test MSE for time series imputation tasks. Best performance per dataset and model in \textbf{bold}. Best performance among non-learnable pooling methods is \underline{underlined}.}
\label{tab:moment_imputation_results}
\resizebox{\columnwidth}{!}{
    \begin{tabular}{rrcccccccccc}
    \toprule
    \makecell{Model} & \makecell{Pooling\\Operator} & \makecell{ETTh1} 
    & \makecell{ETTh2} 
    & \makecell{ETTm1} & \makecell{ETTm2} 
    & \makecell{electricity} &\makecell{traffic}
    & \makecell{weather} \\
    \midrule
    \multirow{6}{*}{\rotatebox{90}{MOMENT-small}} & Last & $0.081 \pm 0.002$ & $0.241 \pm 0.002$ & $0.051 \pm 0.000$ & $0.181 \pm 0.001$ & $0.753 \pm 0.010$ & $1.709 \pm 0.016$ & $0.002 \pm 0.000$ \\
     & Avg & \underline{$0.080 \pm 0.002$} & $0.233 \pm 0.002$ & \underline{$\mathbf{0.050 \pm 0.000}$} & $0.178 \pm 0.001$ & $0.774 \pm 0.013$ & $1.583 \pm 0.016$ & \underline{$\mathbf{0.002 \pm 0.000}$} \\
     & Sum & $0.106 \pm 0.008$ & $0.309 \pm 0.041$ & $0.054 \pm 0.001$ & $0.186 \pm 0.010$ & $1.072 \pm 0.143$ & $2.130 \pm 0.238$ & $0.037 \pm 0.038$ \\
     & Max & $0.082 \pm 0.002$ & \underline{$0.229 \pm 0.004$} & $0.051 \pm 0.000$ & \underline{$0.174 \pm 0.003$} & \underline{$0.720 \pm 0.013$} & \underline{$1.211 \pm 0.051$} & $0.003 \pm 0.000$ \\
     \cmidrule(lr){3-9}
     & W-Avg & $0.080 \pm 0.002$ & $0.233 \pm 0.002$ & $0.050 \pm 0.000$ & $0.178 \pm 0.001$ & $0.774 \pm 0.013$ & $1.582 \pm 0.016$ & $0.002 \pm 0.000$ \\
     & Attn & $\mathbf{0.076 \pm 0.003}$ & $\mathbf{0.088 \pm 0.007}$ & $0.051 \pm 0.001$ & $\mathbf{0.152 \pm 0.001}$ & $\mathbf{0.379 \pm 0.028}$ & $\mathbf{0.265 \pm 0.015}$ & $0.003 \pm 0.001$ \\
    \cline{1-9}
    \multirow{6}{*}{\rotatebox{90}{MOMENT-base}} & Last & $0.082 \pm 0.001$ & $0.243 \pm 0.005$ & $0.051 \pm 0.000$ & $0.181 \pm 0.002$ & $0.760 \pm 0.011$ & $1.668 \pm 0.020$ & $0.002 \pm 0.000$ \\
     & Avg & \underline{$0.082 \pm 0.001$} & $0.233 \pm 0.005$ & \underline{$0.051 \pm 0.000$} & $0.178 \pm 0.002$ & $0.769 \pm 0.016$ & $1.424 \pm 0.020$ & \underline{$\mathbf{0.002 \pm 0.000}$} \\
     & Sum & $0.103 \pm 0.002$ & $0.255 \pm 0.022$ & $0.054 \pm 0.001$ & $0.186 \pm 0.006$ & $1.006 \pm 0.161$ & $1.301 \pm 0.089$ & $0.015 \pm 0.009$ \\
     & Max & $0.082 \pm 0.002$ & \underline{$0.219 \pm 0.004$} & $0.051 \pm 0.000$ & \underline{$0.173 \pm 0.002$} & \underline{$0.707 \pm 0.014$} & \underline{$1.000 \pm 0.113$} & $0.002 \pm 0.000$ \\
     \cmidrule(lr){3-9}
     & W-Avg & $0.082 \pm 0.001$ & $0.233 \pm 0.005$ & $0.051 \pm 0.000$ & $0.178 \pm 0.002$ & $0.769 \pm 0.015$ & $1.425 \pm 0.020$ & $0.002 \pm 0.000$ \\
     & Attn & $\mathbf{0.072 \pm 0.003}$ & $\mathbf{0.082 \pm 0.002}$ & $\mathbf{0.050 \pm 0.000}$ & $\mathbf{0.151 \pm 0.002}$ & $\mathbf{0.370 \pm 0.013}$ & $\mathbf{0.273 \pm 0.004}$ & $0.004 \pm 0.003$ \\
    \cline{1-9}
    \multirow{6}{*}{\rotatebox{90}{MOMENT-large}} & Last & $0.082 \pm 0.001$ & $0.242 \pm 0.005$ & $0.051 \pm 0.001$ & $0.181 \pm 0.001$ & $0.752 \pm 0.014$ & $1.699 \pm 0.017$ & $0.002 \pm 0.000$ \\
     & Avg & \underline{$0.081 \pm 0.002$} & $0.238 \pm 0.007$ & \underline{$0.050 \pm 0.000$} & $0.177 \pm 0.001$ & $0.753 \pm 0.018$ & $1.508 \pm 0.011$ & \underline{$\mathbf{0.002 \pm 0.000}$} \\
     & Sum & $0.095 \pm 0.008$ & $0.254 \pm 0.020$ & $0.053 \pm 0.001$ & $0.177 \pm 0.006$ & $0.887 \pm 0.075$ & \underline{$1.150 \pm 0.065$} & $0.014 \pm 0.002$ \\
     & Max & $0.081 \pm 0.001$ & \underline{$0.230 \pm 0.003$} & $0.050 \pm 0.001$ & \underline{$0.170 \pm 0.002$} & \underline{$0.705 \pm 0.008$} & $1.158 \pm 0.025$ & $0.003 \pm 0.000$ \\
     \cmidrule(lr){3-9}
     & W-Avg & $0.081 \pm 0.002$ & $0.238 \pm 0.007$ & $0.050 \pm 0.000$ & $0.176 \pm 0.001$ & $0.753 \pm 0.018$ & $1.503 \pm 0.011$ & $0.002 \pm 0.000$ \\
     & Attn & $\mathbf{0.072 \pm 0.003}$ & $\mathbf{0.091 \pm 0.006}$ & $\mathbf{0.050 \pm 0.000}$ & $\mathbf{0.147 \pm 0.003}$ & $\mathbf{0.295 \pm 0.006}$ & $\mathbf{0.231 \pm 0.009}$ & $0.004 \pm 0.001$ \\
    \cline{1-9}
    \bottomrule
    \end{tabular}
}
\end{table}

\textbf{Weighted Average Pooling.} \label{app:additional_results_weighted_average}
In Section~\ref{subsec:weight-avg-conv-results}, we presented an analysis of the learned weights in the Weighted Average pooling method using the Mistral-7B model. Figure~\ref{fig:all-models-weight-avg} extends this analysis to additional models and datasets. We observe that the learned weight distributions for a given dataset remain consistent across models, with smaller models (\eg, GPT-2 family) placing more emphasis on later tokens, while larger models exhibit more uniform weighting. The average cosine similarity between Weighted Average pooling and non-learnable pooling methods follows a similar trend: as model size increases, similarity to Max pooling decreases, suggesting reduced reliance on token-level extremes in larger architectures.
\begin{figure}[h]
    \centering

    \begin{subcaptionbox}{GPT-2\label{fig:radar-gpt2}}[0.45\textwidth]
        {\includegraphics[width=0.67\linewidth]{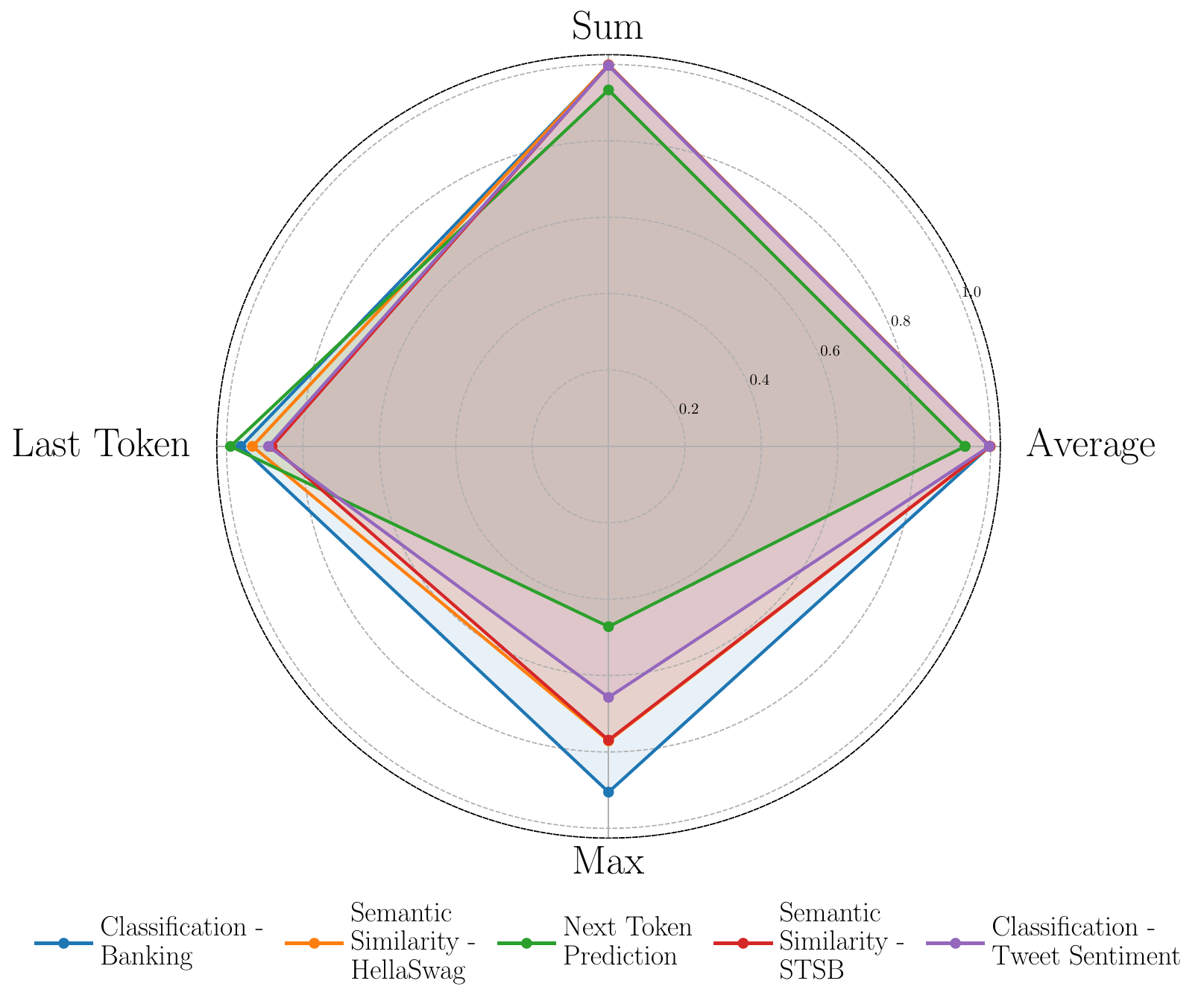}}
    \end{subcaptionbox}
    \hfill
    \begin{subcaptionbox}{GPT-2\label{fig:weights-gpt2}}[0.45\textwidth]
        {\includegraphics[width=0.67\linewidth]{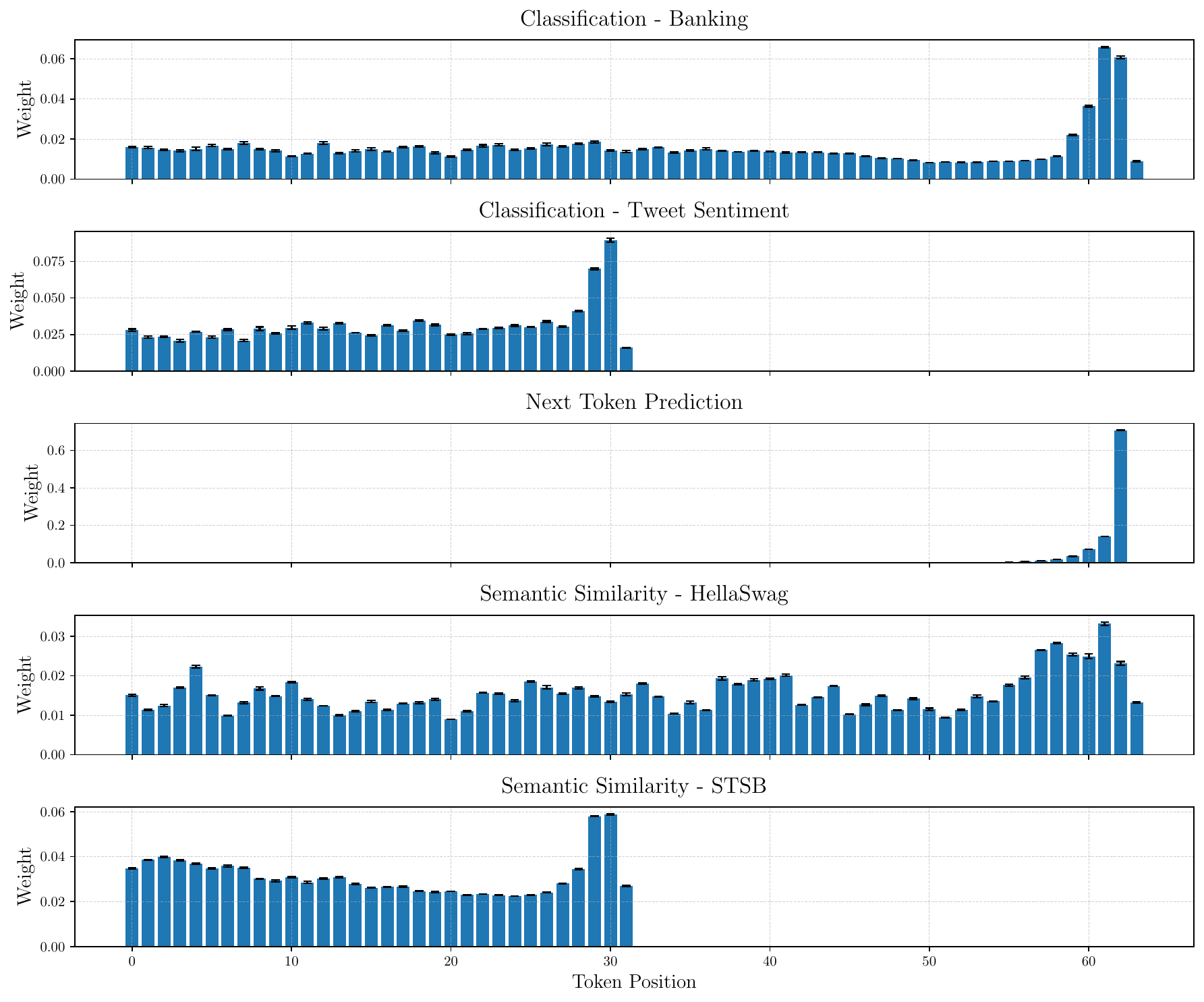}}
    \end{subcaptionbox}
    
    \vspace{0.2em}

    \begin{subcaptionbox}{L2-GPT-2\label{fig:radar-nanogpt}}[0.45\textwidth]
        {\includegraphics[width=0.67\linewidth]{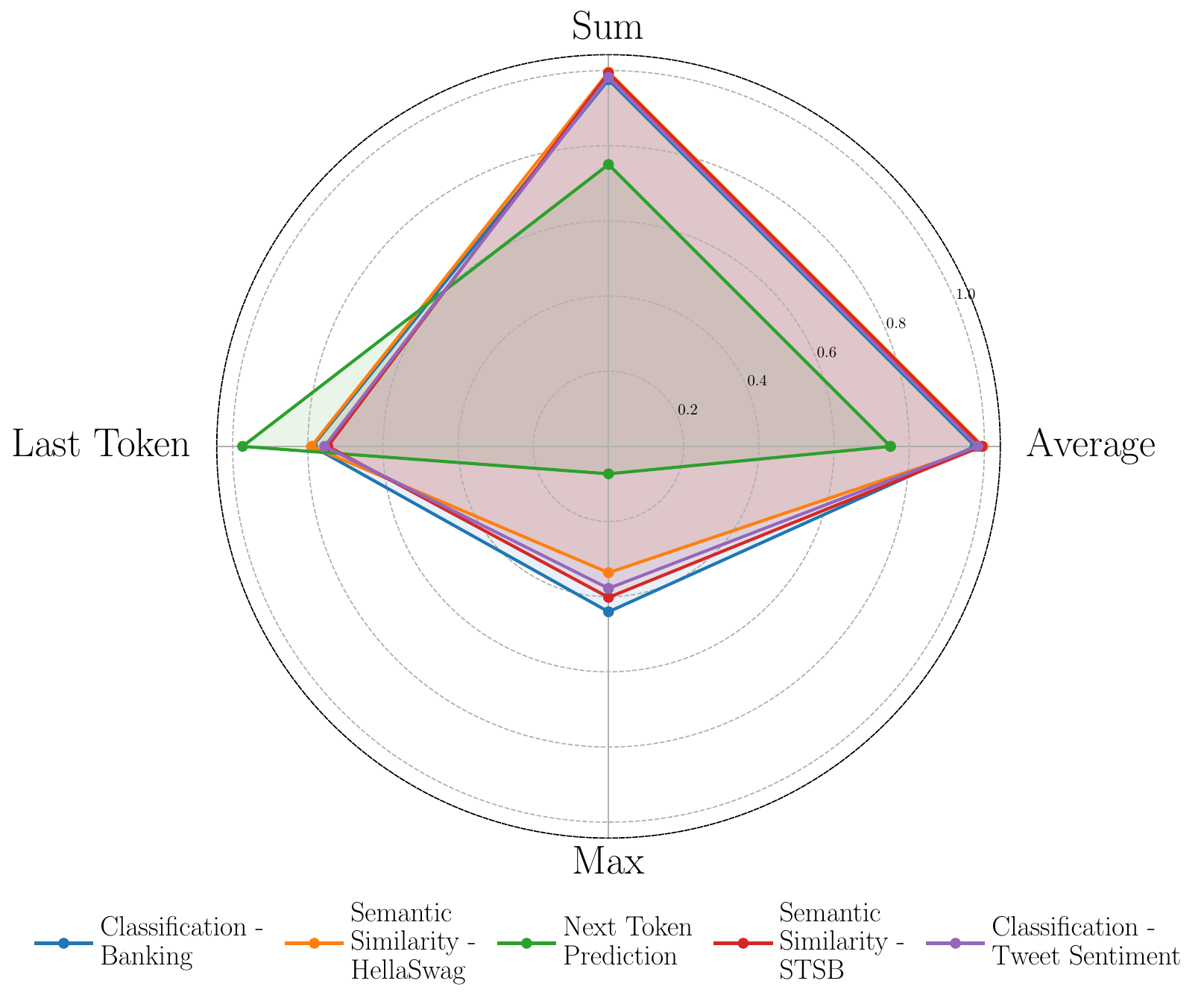}}
    \end{subcaptionbox}
    \hfill
    \begin{subcaptionbox}{L2-GPT-2\label{fig:weights-nanogpt}}[0.45\textwidth]
        {\includegraphics[width=0.67\linewidth]{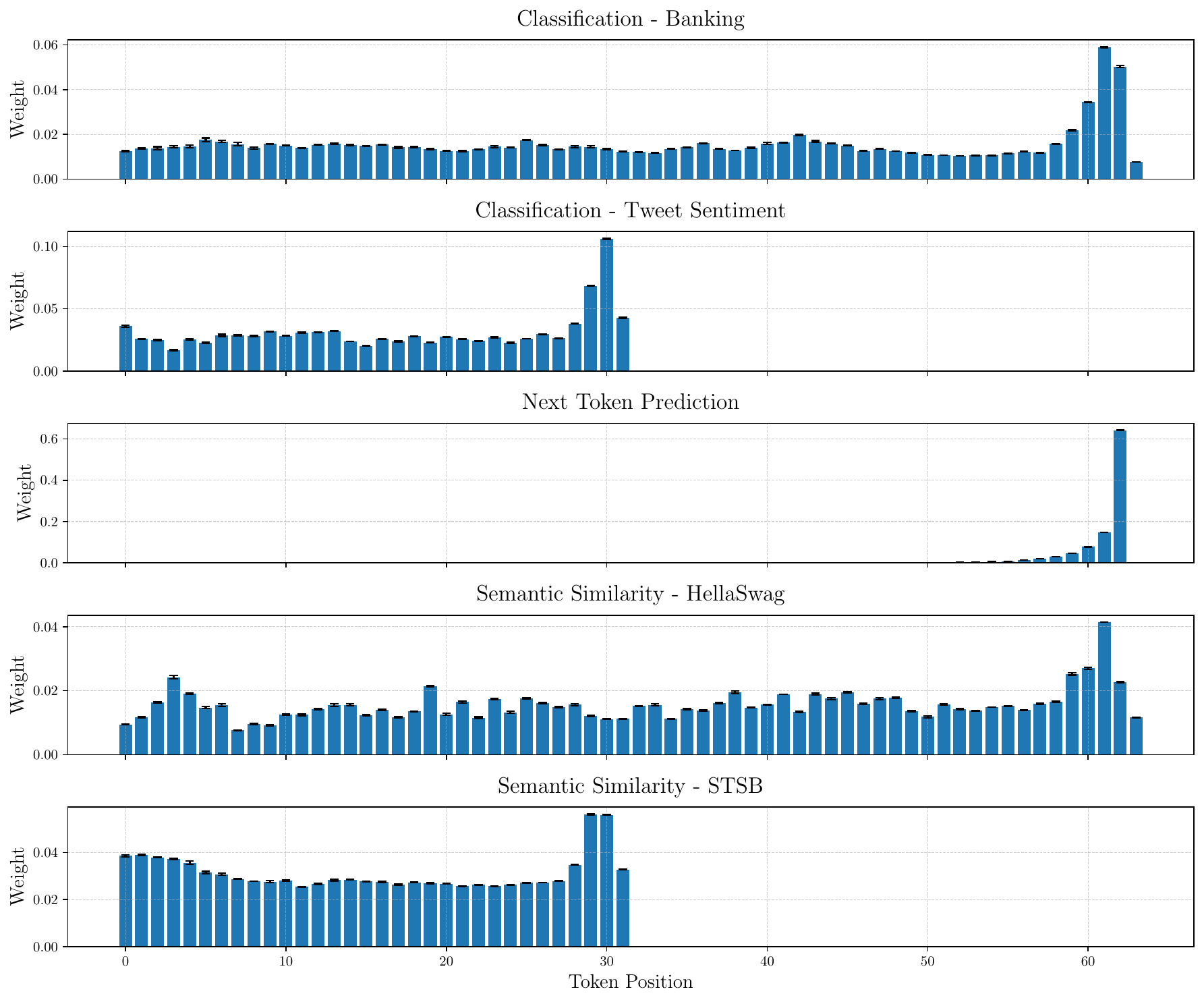}}
    \end{subcaptionbox}
    
    \vspace{0.2em}

    \begin{subcaptionbox}{Qwen2.5\label{fig:radar-qwen}}[0.45\textwidth]
        {\includegraphics[width=0.67\linewidth]{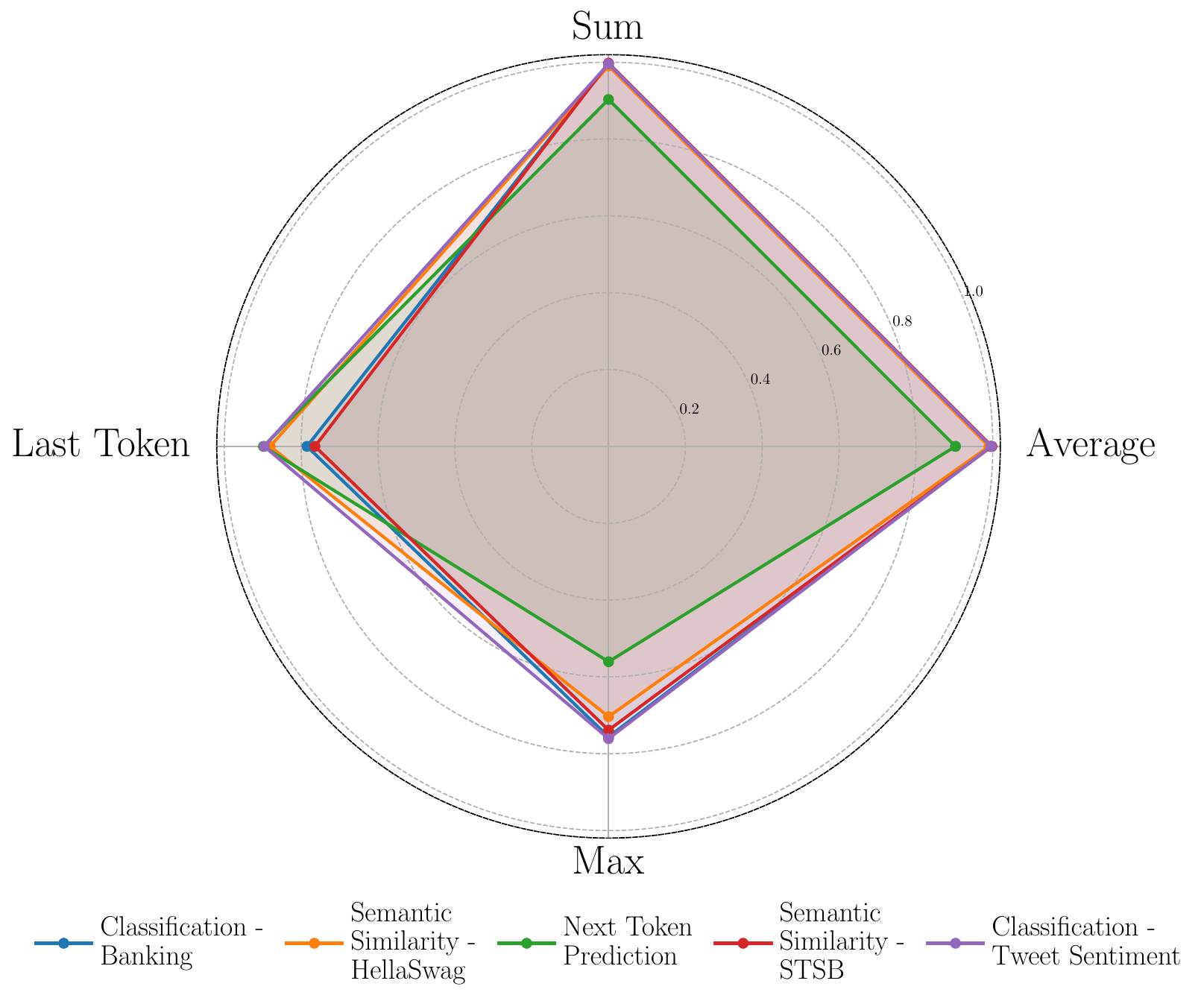}}
    \end{subcaptionbox}
    \hfill
    \begin{subcaptionbox}{Qwen2.5\label{fig:weights-qwen}}[0.45\textwidth]
        {\includegraphics[width=0.67\linewidth]{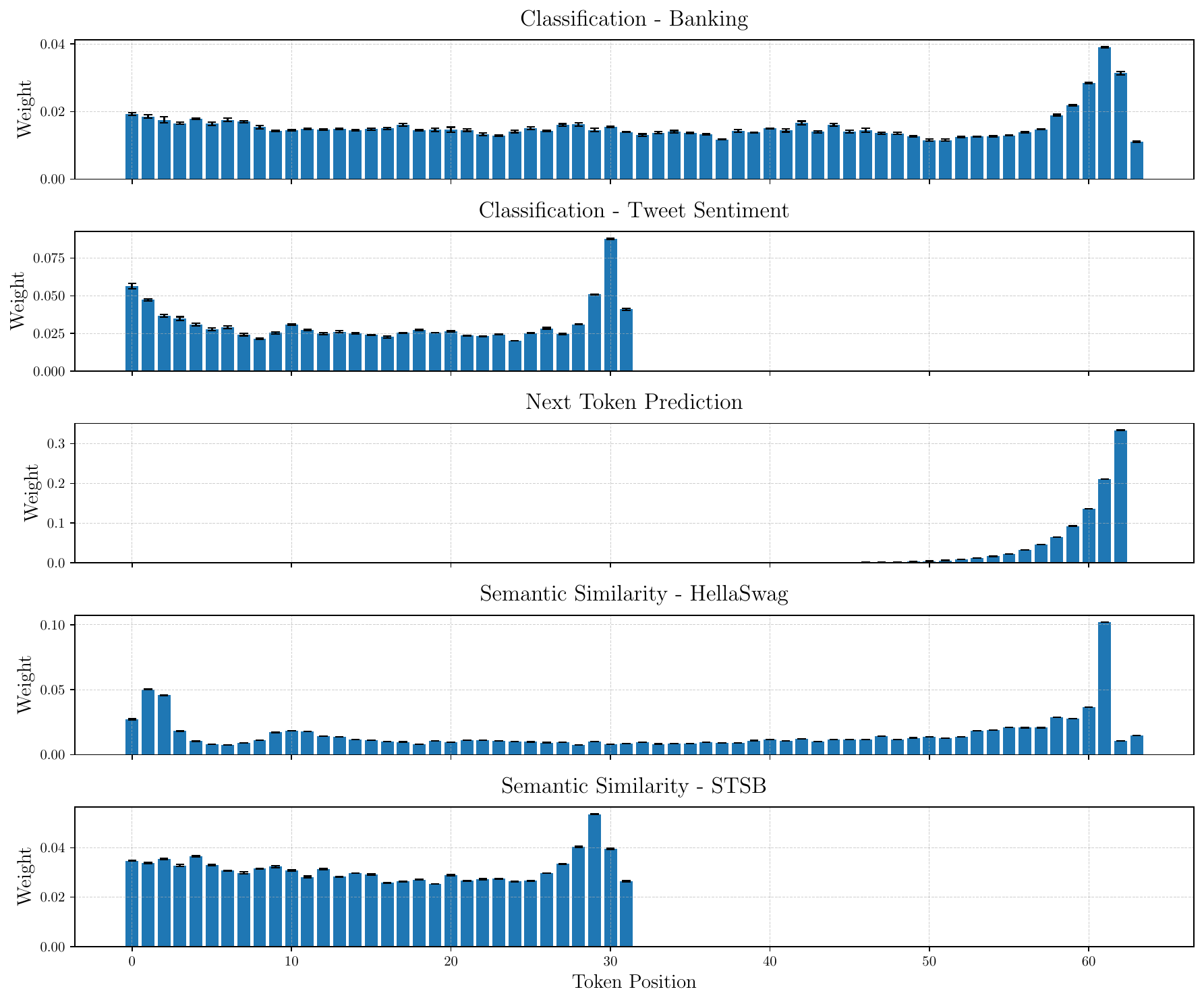}}
    \end{subcaptionbox}

    \vspace{0.2em}

    \begin{subcaptionbox}{Mistral-7B\label{fig:radar-mistral}}[0.45\textwidth]
        {\includegraphics[width=0.67\linewidth]{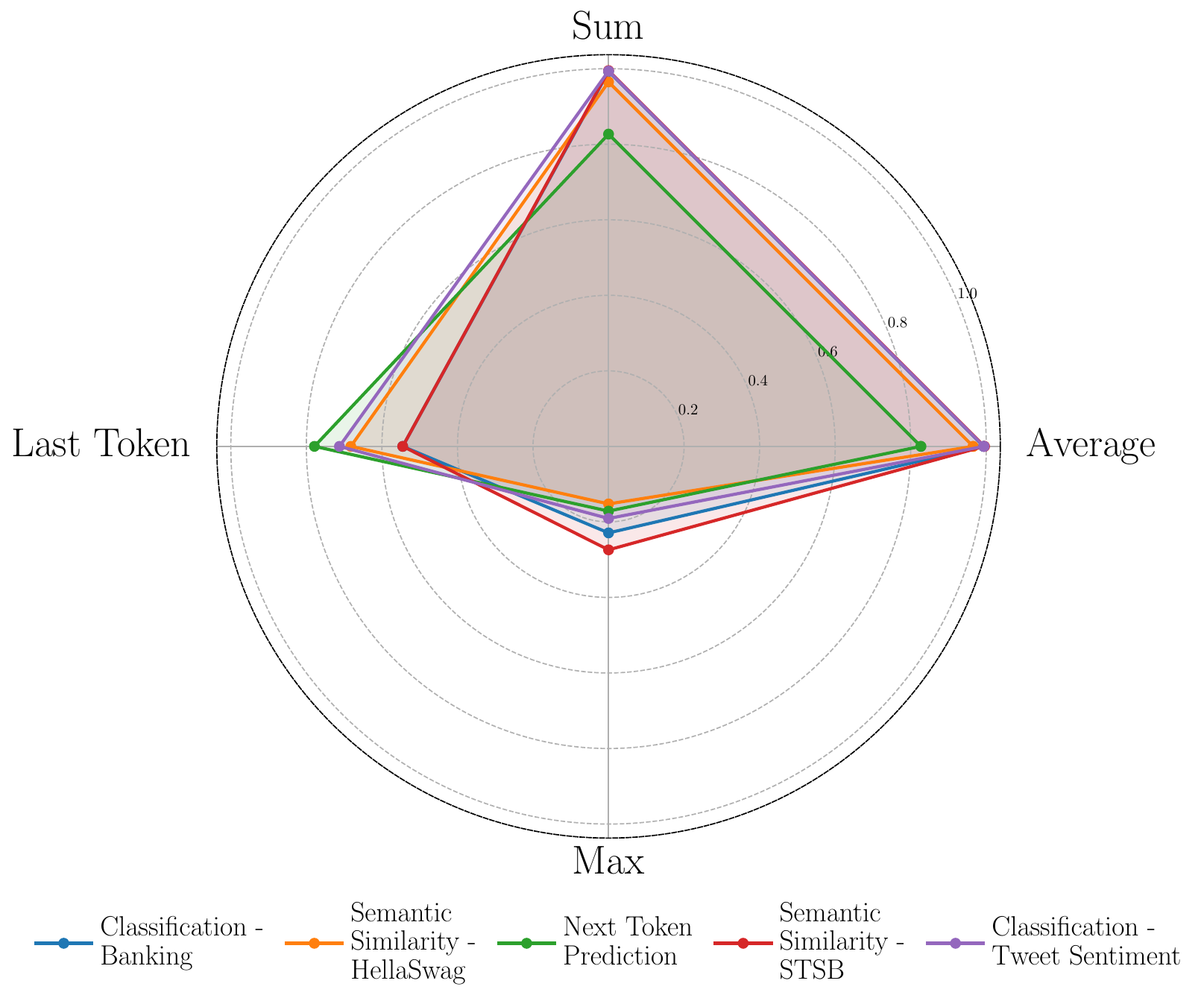}}
    \end{subcaptionbox}
    \hfill
    \begin{subcaptionbox}{Mistral-7B\label{fig:weights-mistral}}[0.45\textwidth]
        {\includegraphics[width=0.67\linewidth]{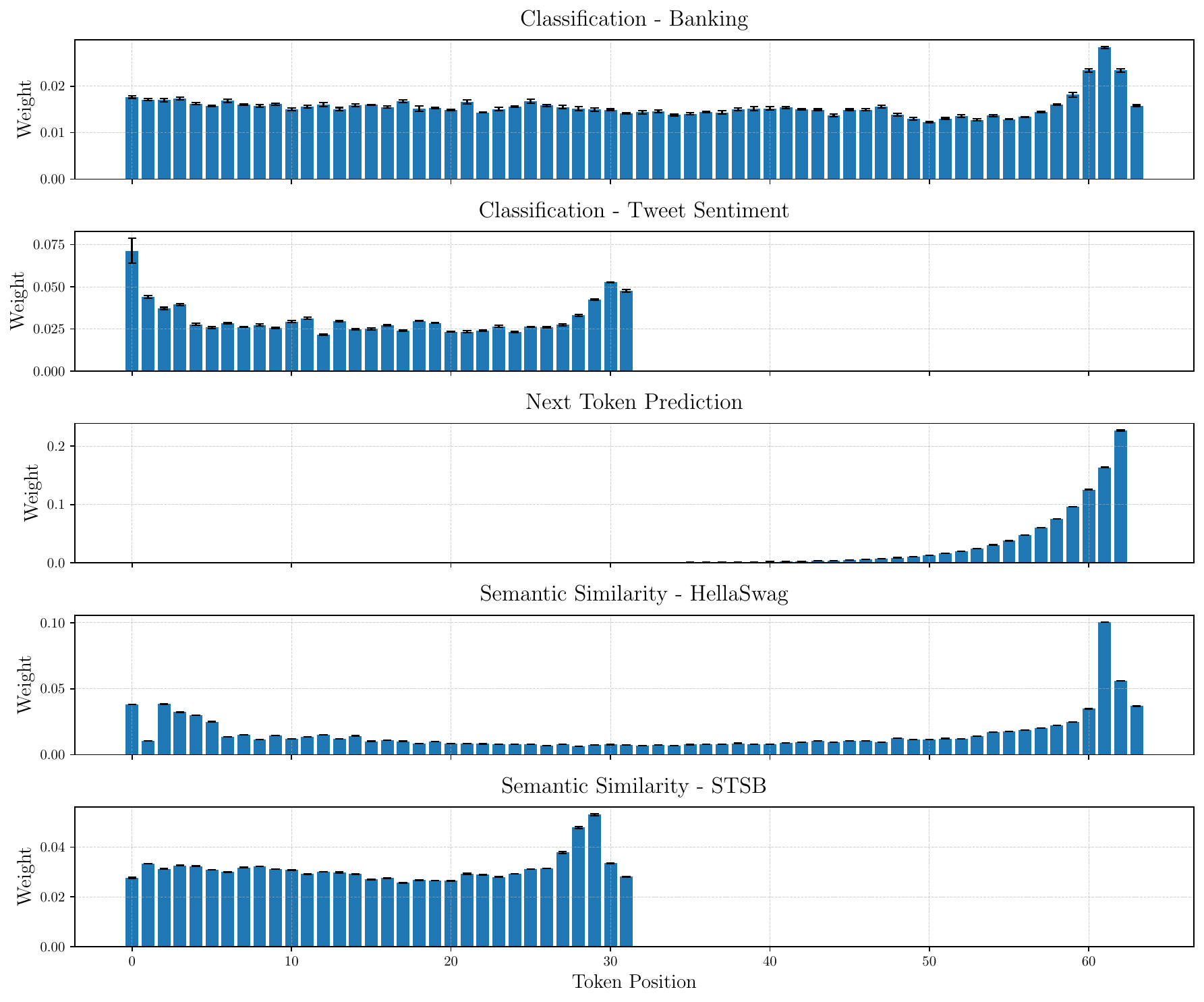}}
    \end{subcaptionbox}

    \vspace{0.2em}

    \begin{subcaptionbox}{Llama\label{fig:radar-LlaMa}}[0.45\textwidth]
        {\includegraphics[width=0.67\linewidth]{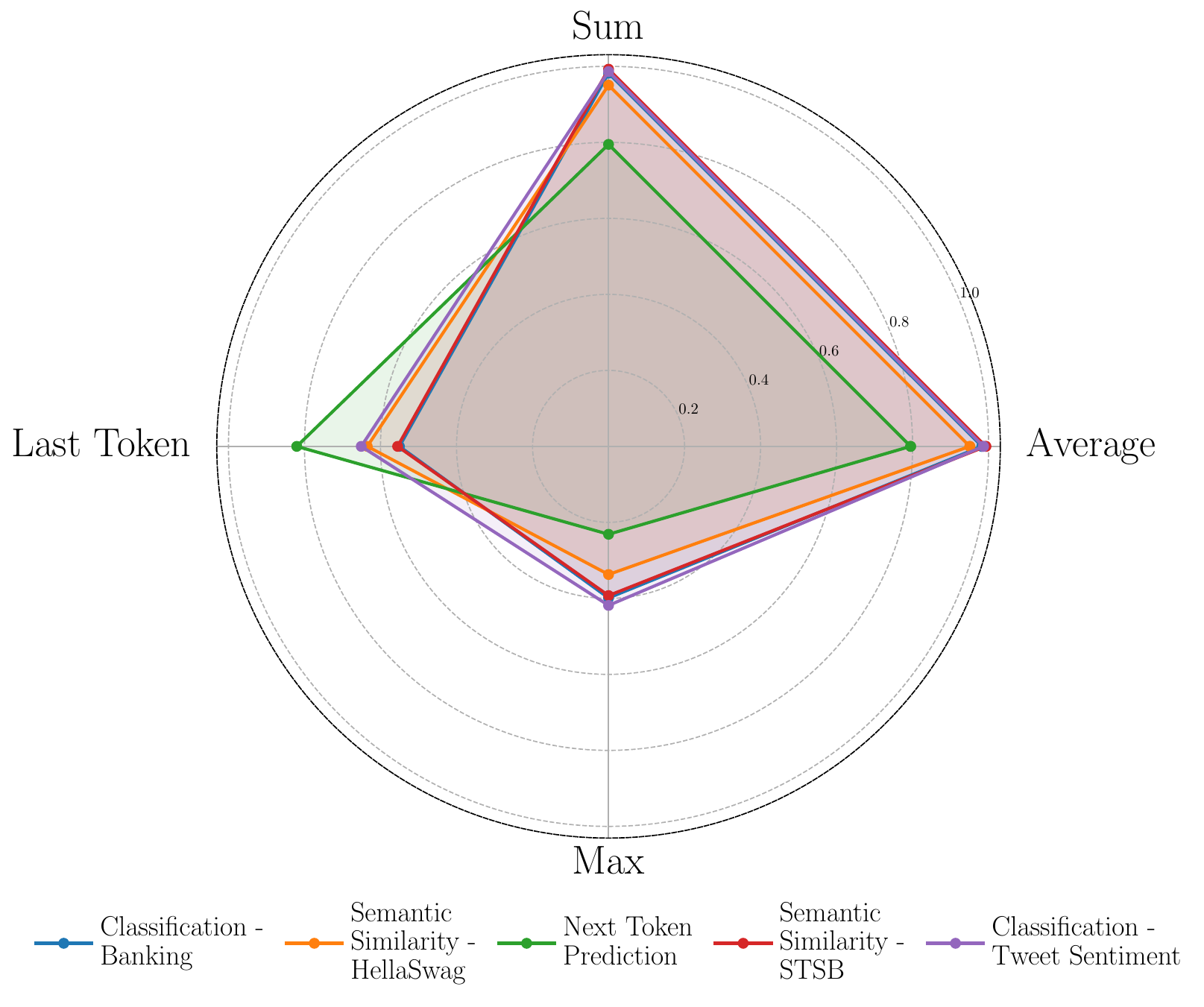}}
    \end{subcaptionbox}
    \hfill
    \begin{subcaptionbox}{Llama\label{fig:weights-LlaMa}}[0.45\textwidth]
        {\includegraphics[width=0.67\linewidth]{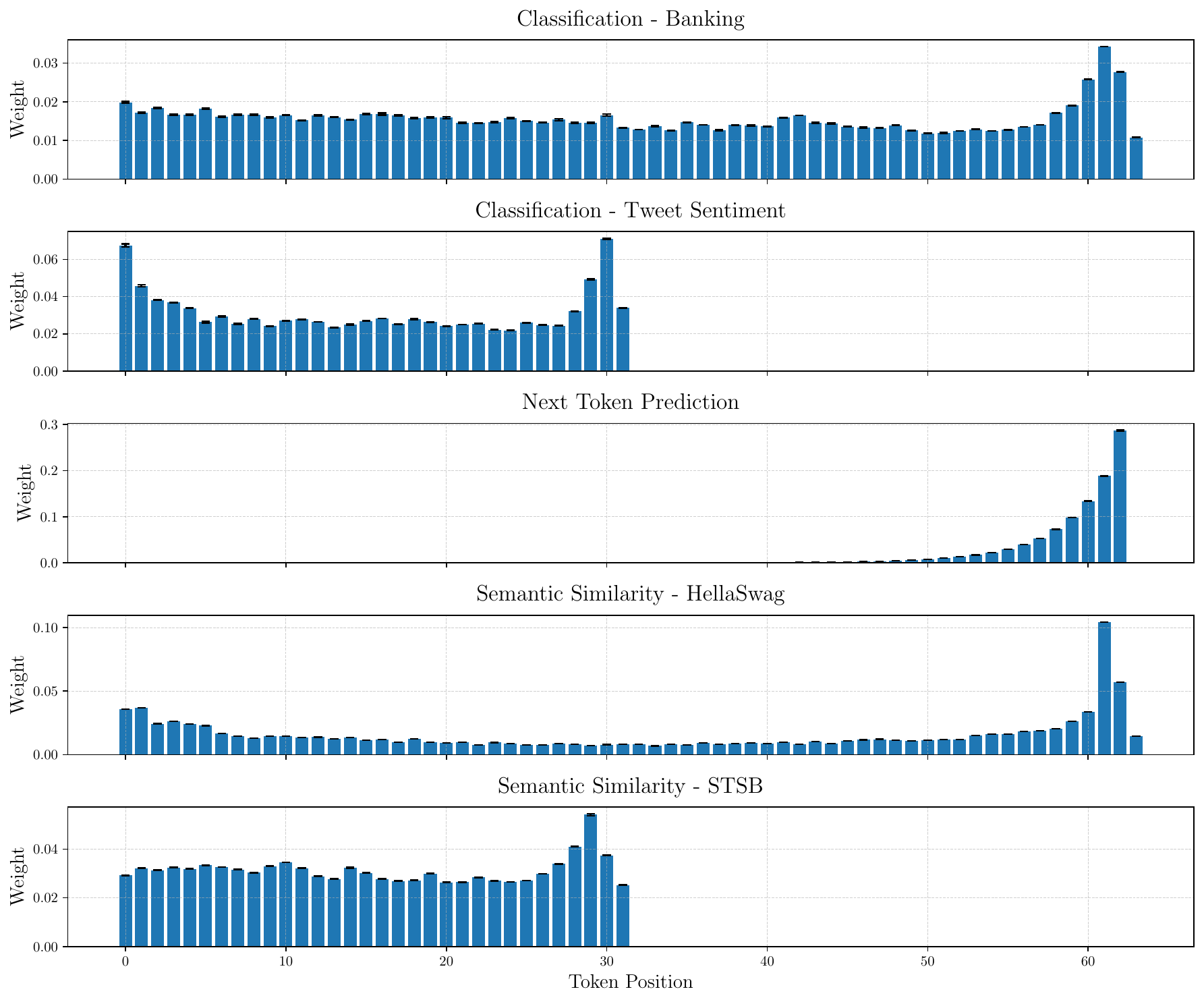}}
    \end{subcaptionbox}

    \caption{Left: Cosine similarity of weighted average pooling with other pooling methods. Right: Learned weight distributions.}
    \label{fig:all-models-weight-avg}
\end{figure}

\subsection{Additional NLP-related Results on Larger Models}\label{appendix:additional_results_larger_models_nlp}

Table \ref{tab:extended_nlp_results} below reports results across tasks for larger models under a consistent experimental setup to our previous experiments.
With larger models, trends across pooling strategies remain visible, but the absolute differences between pooling methods diminish, aligning with the theoretical interpretation.

\begin{table}[h]
\centering
\tiny
\caption{Mean and standard deviation of test metrics for NLP tasks. STSB and HellaSwag are grouped under Sentiment Analysis. Values are mean $\pm$ std.}
\label{tab:nlp-results-transposed}
\renewcommand{\arraystretch}{0.9}
\resizebox{\columnwidth}{!}{%
\begin{tabular}{lcccccc}
\toprule
\textbf{Model} & \textbf{Pooling} & \shortstack{Sentiment \\ (STSB)} & \shortstack{Sentiment \\ (HellaSwag)} & \shortstack{Banking \\ (Accuracy)} & \shortstack{Tweet \\ (Accuracy)} & \shortstack{Next Token \\ (Accuracy)} \\
\midrule
\multirow{6}{*}{Qwen2.5–14B}
& Last   & $0.288 \pm 0.003$ & $0.692 \pm 0.011$ & $34.497 \pm 0.178$ & $57.008 \pm 0.004$ & $\mathbf{\underline{91.039 \pm 0.002}}$ \\
& Avg    & $\underline{0.581 \pm 0.000}$ & $0.773 \pm 0.000$ & $\underline{85.390 \pm 0.001}$ & $\mathbf{\underline{69.930 \pm 0.008}}$ & $53.893 \pm 0.250$ \\
& Sum    & $0.579 \pm 0.002$ & $\underline{0.776 \pm 0.002}$ & $82.711 \pm 0.008$ & $61.218 \pm 1.194$ & $47.053 \pm 0.341$ \\
& Max    & $0.488 \pm 0.005$ & $0.727 \pm 0.001$ & $77.825 \pm 0.006$ & $55.390 \pm 0.728$ & $25.578 \pm 0.332$ \\
\cmidrule(lr){2-7}
& W-Avg  & $\mathbf{{0.589 \pm 0.001}}$ & $\mathbf{{0.798 \pm 0.002}}$ & $\mathbf{{86.867 \pm 0.002}}$ & $69.770 \pm 0.007$ & $90.653 \pm 0.020$ \\
& Attn   & $0.259 \pm 0.011$ & $0.712 \pm 0.013$ & $66.387 \pm 1.939$ & $56.410 \pm 3.001$ & $40.573 \pm 0.892$ \\
\midrule
\multirow{6}{*}{Qwen2.5–32B}
& Last   & $0.303 \pm 0.001$ & $0.723 \pm 0.009$ & $34.383 \pm 0.268$ & $55.886 \pm 0.006$ & $\mathbf{\underline{89.886 \pm 0.005}}$ \\
& Avg    & $0.603 \pm 0.000$ & $0.781 \pm 0.001$ & $\underline{87.760 \pm 0.002}$ & $\underline{68.328 \pm 0.012}$ & $48.536 \pm 0.334$ \\
& Sum    & $\underline{0.604 \pm 0.001}$ & $\underline{0.782 \pm 0.005}$ & $85.227 \pm 0.013$ & $63.665 \pm 0.865$ & $39.204 \pm 0.627$ \\
& Max    & $0.488 \pm 0.005$ & $0.729 \pm 0.008$ & $83.929 \pm 0.008$ & $65.268 \pm 0.596$ & $31.273 \pm 0.586$ \\
\cmidrule(lr){2-7}
& W-Avg  & $\mathbf{0.627 \pm 0.003}$ & $\mathbf{0.812 \pm 0.003}$ & $\mathbf{89.903 \pm 0.004}$ & $\mathbf{69.464 \pm 0.009}$ & $89.022 \pm 0.019$ \\
& Attn   & $0.355 \pm 0.016$ & $0.730 \pm 0.011$ & $68.929 \pm 1.054$ & $58.828 \pm 2.695$ & $29.750 \pm 1.028$ \\
\midrule
\multirow{6}{*}{Mistral3.1–24B}
& Last   & $0.503 \pm 0.002$ & $0.745 \pm 0.008$ & $75.487 \pm 0.087$ & $53.701 \pm 0.005$ & $\underline{\mathbf{88.972 \pm 0.007}}$ \\
& Avg    & $\underline{0.631 \pm 0.001}$ & $0.784 \pm 0.001$ & $87.403 \pm 0.006$ & $66.871 \pm 0.009$ & $51.723 \pm 0.812$ \\
& Sum    & $0.622 \pm 0.004$ & $0.783 \pm 0.004$ & $87.597 \pm 0.015$ & $62.791 \pm 0.976$ & $42.306 \pm 0.732$ \\
& Max    & $0.488 \pm 0.003$ & $0.733 \pm 0.007$ & $79.675 \pm 0.010$ & $61.655 \pm 0.473$ & $27.804 \pm 0.923$ \\
\cmidrule(lr){2-7}
& W-Avg  & $\mathbf{0.682 \pm 0.002}$ & $\mathbf{0.816 \pm 0.003}$ & $\mathbf{88.711 \pm 0.017}$ & $\mathbf{66.200 \pm 0.005}$ & $87.849 \pm 0.024$ \\
& Attn   & $0.392 \pm 0.043$ & $0.697 \pm 0.009$ & $72.922 \pm 1.012$ & $31.294 \pm 4.452$ & $19.911 \pm 2.023$ \\
\bottomrule
\end{tabular}
}
\end{table} 
\label{tab:extended_nlp_results}

\subsection{Additional Results on Sequence Length Changes for NLP}

To further examine pooling behavior, we conducted experiments with the Mistral-7B model on HellaSwag, varying the maximum input sequence length from 16 to 128 tokens. Inputs were truncated or padded as required, with padding tokens excluded from pooling operations to maintain consistency with our setup.

The results, summarized in Table \ref{tab:sequence-length-nlp}, compare pooling strategies across different sequence lengths. Naturally, shorter contexts lead to performance drops due to truncation of semantically important content. To isolate the contribution of pooling itself, comparisons should be made column-wise (i.e., at fixed sequence lengths).

We find that pooling sensitivity is most pronounced at shorter lengths, especially for Last-token and Attention-based pooling. For longer contexts (64 or 128 tokens), performance stabilizes across pooling methods, and the relative differences align more closely with theoretical expectations.

\begin{table}[h]
\centering
\tiny
\caption{Mean and standard deviation of metrics for Mistral-7B on HellaSwag across different input sequence lengths. Values are mean $\pm$ std. Best performance per column is in \textbf{bold}.}
\label{tab:sequence-length-nlp}
\renewcommand{\arraystretch}{0.9}
\resizebox{0.8\columnwidth}{!}{%
\begin{tabular}{lcccc}
\toprule
\textbf{Pooling} & \textbf{16} & \textbf{32} & \textbf{64} & \textbf{128} \\
\midrule
Last   & $0.454 \pm 0.003$ & $0.621 \pm 0.002$ & $0.770 \pm 0.002$ & $0.781 \pm 0.001$ \\
Avg    & $0.503 \pm 0.001$ & $0.702 \pm 0.000$ & $0.769 \pm 0.000$ & $0.771 \pm 0.001$ \\
Sum    & $0.504 \pm 0.001$ & $0.702 \pm 0.001$ & $0.769 \pm 0.000$ & $0.771 \pm 0.001$ \\
Max    & $0.435 \pm 0.003$ & $0.616 \pm 0.003$ & $0.709 \pm 0.001$ & $0.700 \pm 0.008$ \\
\cmidrule(lr){2-5}
W-Avg  & $\mathbf{0.523 \pm 0.002}$ & $\mathbf{0.724 \pm 0.000}$ & $\mathbf{0.801 \pm 0.000}$ & $\mathbf{0.802 \pm 0.003}$ \\
Attn   & $0.220 \pm 0.169$ & $0.278 \pm 0.251$ & $0.737 \pm 0.018$ & $0.764 \pm 0.025$ \\
\bottomrule
\end{tabular}
}
\end{table}

\subsection{Pooling and Adversarial Robustness}\label{appendix:adv_robustness}

Beyond shaping model expressivity for downstream tasks, the choice of pooling operation also impacts the model's adversarial robustness. Our theoretical framework provides insights in this direction by interpreting neighborhood changes, introduced in Section~\ref{sec:expressivity_of_TBM}, as adversarial rather than semantic perturbations intentionally crafted to mislead the model. The analysis suggests that certain pooling operations, such as \textit{Average}, may naturally smooth out adversarial noise, while others like \textit{Max} can either amplify or ignore the perturbation depending on whether the adversarial signal falls within the selected region.

To empirically validate these insights, we apply the Fast Gradient Sign Method (FGSM) to a pre-trained ViT model evaluated on CIFAR-10 and CIFAR-100. We use the same attack budget as the one usually used in the literature ($\epsilon=3/255$ and $\epsilon=2/255$) and we used the same number of epochs and the same initialization for all the poolings to ensure fairness of the comparison~\cite{ennadir2024if}. Table~\ref{tab:adv_attack_fgsm_vit} reports the attack success rates for various pooling strategies under FGSM. As expected, the success rates vary across pooling methods, confirming that pooling choices can meaningfully influence robustness. Therefore, in domains where robustness is critical, such as healthcare or finance, the pooling strategy should be selected not only for its expressivity but also for its impact on adversarial resilience.

\begin{table}[]
\tiny
\centering
\caption{Attack Success rate for different considered Pooling strategies using the FGSM adversarial attack on the CIFAR-10 and CIFAR-100. }
\label{tab:adv_attack_fgsm_vit}
\resizebox{0.8\columnwidth}{!}{%
\begin{tabular}{lccccccc}
\toprule
Dataset                    & Attack Budget      & CLS   & Avg   & Sum   & Max   & W-Avg & Attention-Based \\ \hline
\multirow{2}{*}{CIFAR-10}  & $\epsilon = 3/255$ & 18.56 & 18.94 & 10.92 & 20.32 & 16.93 & 13.9            \\
                           & $\epsilon = 8/255$ & 29.38 & 28.16 & 10.92 & 25.97 & 25.59 & 14.44           \\ \hline
\multirow{2}{*}{CIFAR-100} & $\epsilon = 3/255$ & 34.29 & 32.44 & 20.67 & 31.84 & 31.57 & 29.54           \\
                           & $\epsilon = 8/255$ & 43.99 & 41.42 & 20.68 & 38.07 & 40.49 & 32.94           \\
\bottomrule
\end{tabular}%
}
\end{table}

\clearpage
\newpage

\section*{NeurIPS Paper Checklist}

\begin{enumerate}

\item {\bf Claims}
    \item[] Question: Do the main claims made in the abstract and introduction accurately reflect the paper's contributions and scope?
    \item[] Answer: \answerYes{} 
    \item[] Justification: The main claims made in the abstract and the introduction are consistent with the provided theoretical results in Section~\ref{sec:expressivity} and additionally empirically validated in our experimental results in Section~\ref{sec:experiments}.
    \item[] Guidelines:
    \begin{itemize}
        \item The answer NA means that the abstract and introduction do not include the claims made in the paper.
        \item The abstract and/or introduction should clearly state the claims made, including the contributions made in the paper and important assumptions and limitations. A No or NA answer to this question will not be perceived well by the reviewers. 
        \item The claims made should match theoretical and experimental results, and reflect how much the results can be expected to generalize to other settings. 
        \item It is fine to include aspirational goals as motivation as long as it is clear that these goals are not attained by the paper. 
    \end{itemize}

\item {\bf Limitations}
    \item[] Question: Does the paper discuss the limitations of the work performed by the authors?
    \item[] Answer: \answerYes{} 
    \item[] Justification: The limitations of our work with respect to both the theoretical and empirical results are discussed in Section~\ref{sec:conclusion}.
    \item[] Guidelines:
    \begin{itemize}
        \item The answer NA means that the paper has no limitation while the answer No means that the paper has limitations, but those are not discussed in the paper. 
        \item The authors are encouraged to create a separate "Limitations" section in their paper.
        \item The paper should point out any strong assumptions and how robust the results are to violations of these assumptions (e.g., independence assumptions, noiseless settings, model well-specification, asymptotic approximations only holding locally). The authors should reflect on how these assumptions might be violated in practice and what the implications would be.
        \item The authors should reflect on the scope of the claims made, e.g., if the approach was only tested on a few datasets or with a few runs. In general, empirical results often depend on implicit assumptions, which should be articulated.
        \item The authors should reflect on the factors that influence the performance of the approach. For example, a facial recognition algorithm may perform poorly when image resolution is low or images are taken in low lighting. Or a speech-to-text system might not be used reliably to provide closed captions for online lectures because it fails to handle technical jargon.
        \item The authors should discuss the computational efficiency of the proposed algorithms and how they scale with dataset size.
        \item If applicable, the authors should discuss possible limitations of their approach to address problems of privacy and fairness.
        \item While the authors might fear that complete honesty about limitations might be used by reviewers as grounds for rejection, a worse outcome might be that reviewers discover limitations that aren't acknowledged in the paper. The authors should use their best judgment and recognize that individual actions in favor of transparency play an important role in developing norms that preserve the integrity of the community. Reviewers will be specifically instructed to not penalize honesty concerning limitations.
    \end{itemize}

\item {\bf Theory assumptions and proofs}
    \item[] Question: For each theoretical result, does the paper provide the full set of assumptions and a complete (and correct) proof?
    \item[] Answer: \answerYes{} 
    \item[] Justification: All the theoretical claims are explained in the corresponding proofs (Appendices~\ref{app:proof_theo_expressivity}, \ref{app:proof_l2_expressivity}~and~\ref{app:proof_lipsformer}). In addition, we clearly state our problem setup detailing the assumptions in Section~\ref{sec:preliminaries}, which are referenced in each proof and theorem. 
    \item[] Guidelines:
    \begin{itemize}
        \item The answer NA means that the paper does not include theoretical results. 
        \item All the theorems, formulas, and proofs in the paper should be numbered and cross-referenced.
        \item All assumptions should be clearly stated or referenced in the statement of any theorems.
        \item The proofs can either appear in the main paper or the supplemental material, but if they appear in the supplemental material, the authors are encouraged to provide a short proof sketch to provide intuition. 
        \item Inversely, any informal proof provided in the core of the paper should be complemented by formal proofs provided in appendix or supplemental material.
        \item Theorems and Lemmas that the proof relies upon should be properly referenced. 
    \end{itemize}

    \item {\bf Experimental result reproducibility}
    \item[] Question: Does the paper fully disclose all the information needed to reproduce the main experimental results of the paper to the extent that it affects the main claims and/or conclusions of the paper (regardless of whether the code and data are provided or not)?
    \item[] Answer: \answerYes{} 
    \item[] Justification: In addition to providing a detailed experimental setup in Appendix~\ref{app:experimental_details}, we provide the source code to reproduce our results in the supplementary materials section. 
    \item[] Guidelines:
    \begin{itemize}
        \item The answer NA means that the paper does not include experiments.
        \item If the paper includes experiments, a No answer to this question will not be perceived well by the reviewers: Making the paper reproducible is important, regardless of whether the code and data are provided or not.
        \item If the contribution is a dataset and/or model, the authors should describe the steps taken to make their results reproducible or verifiable. 
        \item Depending on the contribution, reproducibility can be accomplished in various ways. For example, if the contribution is a novel architecture, describing the architecture fully might suffice, or if the contribution is a specific model and empirical evaluation, it may be necessary to either make it possible for others to replicate the model with the same dataset, or provide access to the model. In general. releasing code and data is often one good way to accomplish this, but reproducibility can also be provided via detailed instructions for how to replicate the results, access to a hosted model (e.g., in the case of a large language model), releasing of a model checkpoint, or other means that are appropriate to the research performed.
        \item While NeurIPS does not require releasing code, the conference does require all submissions to provide some reasonable avenue for reproducibility, which may depend on the nature of the contribution. For example
        \begin{enumerate}
            \item If the contribution is primarily a new algorithm, the paper should make it clear how to reproduce that algorithm.
            \item If the contribution is primarily a new model architecture, the paper should describe the architecture clearly and fully.
            \item If the contribution is a new model (e.g., a large language model), then there should either be a way to access this model for reproducing the results or a way to reproduce the model (e.g., with an open-source dataset or instructions for how to construct the dataset).
            \item We recognize that reproducibility may be tricky in some cases, in which case authors are welcome to describe the particular way they provide for reproducibility. In the case of closed-source models, it may be that access to the model is limited in some way (e.g., to registered users), but it should be possible for other researchers to have some path to reproducing or verifying the results.
        \end{enumerate}
    \end{itemize}

\item {\bf Open access to data and code}
    \item[] Question: Does the paper provide open access to the data and code, with sufficient instructions to faithfully reproduce the main experimental results, as described in supplemental material?
    \item[] Answer: \answerYes{} 
    \item[] Justification: In all our experiments we used publicly available datasets that can be found on, e.g., HuggingFace or other open source platforms. We provide the source code in the supplementary material for reproducibility.
    \item[] Guidelines:
    \begin{itemize}
        \item The answer NA means that paper does not include experiments requiring code.
        \item Please see the NeurIPS code and data submission guidelines (\url{https://nips.cc/public/guides/CodeSubmissionPolicy}) for more details.
        \item While we encourage the release of code and data, we understand that this might not be possible, so “No” is an acceptable answer. Papers cannot be rejected simply for not including code, unless this is central to the contribution (e.g., for a new open-source benchmark).
        \item The instructions should contain the exact command and environment needed to run to reproduce the results. See the NeurIPS code and data submission guidelines (\url{https://nips.cc/public/guides/CodeSubmissionPolicy}) for more details.
        \item The authors should provide instructions on data access and preparation, including how to access the raw data, preprocessed data, intermediate data, and generated data, etc.
        \item The authors should provide scripts to reproduce all experimental results for the new proposed method and baselines. If only a subset of experiments are reproducible, they should state which ones are omitted from the script and why.
        \item At submission time, to preserve anonymity, the authors should release anonymized versions (if applicable).
        \item Providing as much information as possible in supplemental material (appended to the paper) is recommended, but including URLs to data and code is permitted.
    \end{itemize}

\item {\bf Experimental setting/details}
    \item[] Question: Does the paper specify all the training and test details (e.g., data splits, hyperparameters, how they were chosen, type of optimizer, etc.) necessary to understand the results?
    \item[] Answer: \answerYes{} 
    \item[] Justification: Appendix \ref{app:experimental_details} contains the experimental setup for all modalities, including training parameters and data splits. The code contains all hyperparameters for every model and experiment.
    \item[] Guidelines:
    \begin{itemize}
        \item The answer NA means that the paper does not include experiments.
        \item The experimental setting should be presented in the core of the paper to a level of detail that is necessary to appreciate the results and make sense of them.
        \item The full details can be provided either with the code, in appendix, or as supplemental material.
    \end{itemize}

\item {\bf Experiment statistical significance}
    \item[] Question: Does the paper report error bars suitably and correctly defined or other appropriate information about the statistical significance of the experiments?
    \item[] Answer: \answerYes{} 
    \item[] Justification: We report standard deviation of all results in the respective tables in Section \ref{sec:experiment_effect_on_downstream_performance} obtained via repeated experiments with 5 random seeds. Further details in Appendix \ref{app:experimental_details}.
    \item[] Guidelines:
    \begin{itemize}
        \item The answer NA means that the paper does not include experiments.
        \item The authors should answer "Yes" if the results are accompanied by error bars, confidence intervals, or statistical significance tests, at least for the experiments that support the main claims of the paper.
        \item The factors of variability that the error bars are capturing should be clearly stated (for example, train/test split, initialization, random drawing of some parameter, or overall run with given experimental conditions).
        \item The method for calculating the error bars should be explained (closed form formula, call to a library function, bootstrap, etc.)
        \item The assumptions made should be given (e.g., Normally distributed errors).
        \item It should be clear whether the error bar is the standard deviation or the standard error of the mean.
        \item It is OK to report 1-sigma error bars, but one should state it. The authors should preferably report a 2-sigma error bar than state that they have a 96\% CI, if the hypothesis of Normality of errors is not verified.
        \item For asymmetric distributions, the authors should be careful not to show in tables or figures symmetric error bars that would yield results that are out of range (e.g. negative error rates).
        \item If error bars are reported in tables or plots, The authors should explain in the text how they were calculated and reference the corresponding figures or tables in the text.
    \end{itemize}

\item {\bf Experiments compute resources}
    \item[] Question: For each experiment, does the paper provide sufficient information on the computer resources (type of compute workers, memory, time of execution) needed to reproduce the experiments?
    \item[] Answer: \answerYes{} 
    \item[] Justification: We provide the type of GPUs and number of GPU hours used for our experiments in Appendix~\ref{app:experimental_details}.
    \item[] Guidelines:
    \begin{itemize}
        \item The answer NA means that the paper does not include experiments.
        \item The paper should indicate the type of compute workers CPU or GPU, internal cluster, or cloud provider, including relevant memory and storage.
        \item The paper should provide the amount of compute required for each of the individual experimental runs as well as estimate the total compute. 
        \item The paper should disclose whether the full research project required more compute than the experiments reported in the paper (e.g., preliminary or failed experiments that didn't make it into the paper). 
    \end{itemize}
    
\item {\bf Code of ethics}
    \item[] Question: Does the research conducted in the paper conform, in every respect, with the NeurIPS Code of Ethics \url{https://neurips.cc/public/EthicsGuidelines}?
    \item[] Answer: \answerYes{} 
    \item[] Justification: Our work adheres to the NeurIPS Code of Ethics. We use public datasets and publicly available models and report on the limitations of our work.
    \item[] Guidelines:
    \begin{itemize}
        \item The answer NA means that the authors have not reviewed the NeurIPS Code of Ethics.
        \item If the authors answer No, they should explain the special circumstances that require a deviation from the Code of Ethics.
        \item The authors should make sure to preserve anonymity (e.g., if there is a special consideration due to laws or regulations in their jurisdiction).
    \end{itemize}

\item {\bf Broader impacts}
    \item[] Question: Does the paper discuss both potential positive societal impacts and negative societal impacts of the work performed?
    \item[] Answer: \answerNA{} 
    \item[] Justification: Our work has no further societal impacts apart from known impacts of LLMs and other Transformer-based models. We are not aware of any applications of our insights that would result in negative societal impacts.
    \item[] Guidelines:
    \begin{itemize}
        \item The answer NA means that there is no societal impact of the work performed.
        \item If the authors answer NA or No, they should explain why their work has no societal impact or why the paper does not address societal impact.
        \item Examples of negative societal impacts include potential malicious or unintended uses (e.g., disinformation, generating fake profiles, surveillance), fairness considerations (e.g., deployment of technologies that could make decisions that unfairly impact specific groups), privacy considerations, and security considerations.
        \item The conference expects that many papers will be foundational research and not tied to particular applications, let alone deployments. However, if there is a direct path to any negative applications, the authors should point it out. For example, it is legitimate to point out that an improvement in the quality of generative models could be used to generate deepfakes for disinformation. On the other hand, it is not needed to point out that a generic algorithm for optimizing neural networks could enable people to train models that generate Deepfakes faster.
        \item The authors should consider possible harms that could arise when the technology is being used as intended and functioning correctly, harms that could arise when the technology is being used as intended but gives incorrect results, and harms following from (intentional or unintentional) misuse of the technology.
        \item If there are negative societal impacts, the authors could also discuss possible mitigation strategies (e.g., gated release of models, providing defenses in addition to attacks, mechanisms for monitoring misuse, mechanisms to monitor how a system learns from feedback over time, improving the efficiency and accessibility of ML).
    \end{itemize}
    
\item {\bf Safeguards}
    \item[] Question: Does the paper describe safeguards that have been put in place for responsible release of data or models that have a high risk for misuse (e.g., pretrained language models, image generators, or scraped datasets)?
    \item[] Answer: \answerNA{} 
    \item[] Justification: Our paper does not release any data or models that have a high risk for misuse.
    \item[] Guidelines:
    \begin{itemize}
        \item The answer NA means that the paper poses no such risks.
        \item Released models that have a high risk for misuse or dual-use should be released with necessary safeguards to allow for controlled use of the model, for example by requiring that users adhere to usage guidelines or restrictions to access the model or implementing safety filters. 
        \item Datasets that have been scraped from the Internet could pose safety risks. The authors should describe how they avoided releasing unsafe images.
        \item We recognize that providing effective safeguards is challenging, and many papers do not require this, but we encourage authors to take this into account and make a best faith effort.
    \end{itemize}

\item {\bf Licenses for existing assets}
    \item[] Question: Are the creators or original owners of assets (e.g., code, data, models), used in the paper, properly credited and are the license and terms of use explicitly mentioned and properly respected?
    \item[] Answer: \answerYes{} 
    \item[] Justification: All datasets, models and repositories were cited appropriately.
    \item[] Guidelines:
    \begin{itemize}
        \item The answer NA means that the paper does not use existing assets.
        \item The authors should cite the original paper that produced the code package or dataset.
        \item The authors should state which version of the asset is used and, if possible, include a URL.
        \item The name of the license (e.g., CC-BY 4.0) should be included for each asset.
        \item For scraped data from a particular source (e.g., website), the copyright and terms of service of that source should be provided.
        \item If assets are released, the license, copyright information, and terms of use in the package should be provided. For popular datasets, \url{paperswithcode.com/datasets} has curated licenses for some datasets. Their licensing guide can help determine the license of a dataset.
        \item For existing datasets that are re-packaged, both the original license and the license of the derived asset (if it has changed) should be provided.
        \item If this information is not available online, the authors are encouraged to reach out to the asset's creators.
    \end{itemize}

\item {\bf New assets}
    \item[] Question: Are new assets introduced in the paper well documented and is the documentation provided alongside the assets?
    \item[] Answer: \answerYes{} 
    \item[] Justification: The source code released along with our paper is properly documented and contains the license terms.
    \item[] Guidelines:
    \begin{itemize}
        \item The answer NA means that the paper does not release new assets.
        \item Researchers should communicate the details of the dataset/code/model as part of their submissions via structured templates. This includes details about training, license, limitations, etc. 
        \item The paper should discuss whether and how consent was obtained from people whose asset is used.
        \item At submission time, remember to anonymize your assets (if applicable). You can either create an anonymized URL or include an anonymized zip file.
    \end{itemize}

\item {\bf Crowdsourcing and research with human subjects}
    \item[] Question: For crowdsourcing experiments and research with human subjects, does the paper include the full text of instructions given to participants and screenshots, if applicable, as well as details about compensation (if any)? 
    \item[] Answer: \answerNA{} 
    \item[] Justification: Our work did not involve crowdsourcing nor research with human subjects. All experiments are performed on publicly available datasets.
    \item[] Guidelines:
    \begin{itemize}
        \item The answer NA means that the paper does not involve crowdsourcing nor research with human subjects.
        \item Including this information in the supplemental material is fine, but if the main contribution of the paper involves human subjects, then as much detail as possible should be included in the main paper. 
        \item According to the NeurIPS Code of Ethics, workers involved in data collection, curation, or other labor should be paid at least the minimum wage in the country of the data collector. 
    \end{itemize}

\item {\bf Institutional review board (IRB) approvals or equivalent for research with human subjects}
    \item[] Question: Does the paper describe potential risks incurred by study participants, whether such risks were disclosed to the subjects, and whether Institutional Review Board (IRB) approvals (or an equivalent approval/review based on the requirements of your country or institution) were obtained?
    \item[] Answer: \answerNA{} 
    \item[] Justification: Our work did not involve crowdsourcing nor research with human subjects. All experiments are performed on publicly available datasets.
    \item[] Guidelines:
    \begin{itemize}
        \item The answer NA means that the paper does not involve crowdsourcing nor research with human subjects.
        \item Depending on the country in which research is conducted, IRB approval (or equivalent) may be required for any human subjects research. If you obtained IRB approval, you should clearly state this in the paper. 
        \item We recognize that the procedures for this may vary significantly between institutions and locations, and we expect authors to adhere to the NeurIPS Code of Ethics and the guidelines for their institution. 
        \item For initial submissions, do not include any information that would break anonymity (if applicable), such as the institution conducting the review.
    \end{itemize}

\item {\bf Declaration of LLM usage}
    \item[] Question: Does the paper describe the usage of LLMs if it is an important, original, or non-standard component of the core methods in this research? Note that if the LLM is used only for writing, editing, or formatting purposes and does not impact the core methodology, scientific rigorousness, or originality of the research, declaration is not required.
    \item[] Answer: \answerNA{} 
    \item[] Justification: No LLM was used in this work for the core methods.
    \item[] Guidelines:
    \begin{itemize}
        \item The answer NA means that the core method development in this research does not involve LLMs as any important, original, or non-standard components.
        \item Please refer to our LLM policy (\url{https://neurips.cc/Conferences/2025/LLM}) for what should or should not be described.
    \end{itemize}

\end{enumerate}

\end{document}